\pdfoutput=1
\documentclass[11pt,letter,reqno,oneside]{amsart}

\usepackage[centering,text={15.5cm,24cm}]{geometry}
\usepackage[dvips]{graphicx}
\usepackage[all]{xy}
\usepackage{mathrsfs}
\usepackage{marvosym}
\usepackage{stmaryrd}
\usepackage{srcltx}

\usepackage{pstricks}

\usepackage{amsfonts,amsmath,amssymb,amscd,amsthm}
\usepackage{amsaddr}
\usepackage{hyperref}
\usepackage{enumerate}
\usepackage[all]{xy}
\usepackage{mathtools}
\usepackage{mathrsfs}
\usepackage{xparse}
\usepackage{float}
\usepackage{mathrsfs}
\usepackage{physics}
\usepackage{dsfont}
\usepackage{bbm}
\usepackage{bm}
\usepackage{tabu}
\usepackage{tensor}
\usepackage{textcomp}
\usepackage{subcaption}

\usepackage[T1]{fontenc}

\usepackage{ifthen}
\usepackage{tikz,tikz-cd}
\usetikzlibrary{arrows,backgrounds,calc,decorations.pathreplacing}

\newcommand{\PaperTitle}{Geometric Deep Learning and Equivariant Neural Networks}
\newcommand{\pageheader}{\noindent \hfill\PaperTitle\hfill\thepage}

\hypersetup{
	pdftitle={\PaperTitle},         %
	pdfauthor={Jan Gerken, Jimmy Aronsson, Oscar Carlsson, Hampus Linander, Fredrik Ohlsson, , Christoffer Petersson, Daniel Persson},      %
	linktocpage=true,           
	colorlinks=true,            
	linkcolor=blue!75!black,    
	citecolor=green!50!black,   
	filecolor=magenta,          
	urlcolor=blue!75!black      
}

\def\mydate{\ifcase\month \or January\or February\or March\or
April\or May\or June\or July\or August\or September\or October\or
November\or December\fi \space\number\day,\space\number\year}

\makeatletter
\renewcommand{\tocsection}[3]{%
  \par\vspace{-0.5em}\hspace{-2em}\indentlabel{\@ifnotempty{#2}{\ignorespaces#1 #2.\quad}}\bfseries#3}
\renewcommand{\tocsubsection}[3]{%
  \indentlabel{\@ifnotempty{#2}{\ignorespaces#1 #2\quad}}#3}

\newcommand\@dotsep{4.5}
\def\@tocline#1#2#3#4#5#6#7{\relax
  \ifnum #1>\c@tocdepth 
  \else
    \par \addpenalty\@secpenalty\addvspace{#2}%
    \begingroup \hyphenpenalty\@M
    \@ifempty{#4}{%
      \@tempdima\csname r@tocindent\number#1\endcsname\relax
    }{%
      \@tempdima#4\relax
    }%
    \parindent\z@ \leftskip#3\relax \advance\leftskip\@tempdima\relax
    \rightskip\@pnumwidth plus1em \parfillskip-\@pnumwidth
    #5\leavevmode\hskip-\@tempdima{#6}\nobreak
    \leaders\hbox{$\m@th\mkern \@dotsep mu\hbox{.}\mkern \@dotsep mu$}\hfill
    \nobreak
    \hspace{-0.5em}\hbox to\@pnumwidth{\@tocpagenum{\ifnum#1=1\bfseries\fi#7}}\par
    \nobreak
    \endgroup
  \fi}
\AtBeginDocument{%
\expandafter\renewcommand\csname r@tocindent0\endcsname{0pt}
}
\def\l@subsection{\@tocline{2}{0pt}{2.5pc}{5pc}{}}
\makeatother

\newenvironment{nouppercase}{%
  \renewcommand{\uppercasenonmath}[1]{}}{}

\newcommand{\nontocsec}[1]{\vspace*{1.5em}\begin{center}\scshape #1\end{center}\vspace{0.7em}}

\makeatletter
\renewcommand{\@bibtitlestyle}{\nontocsec{\normalsize\refname}}
\makeatother

\newtheorem*{namedtheorem}{\theoremname}
\newcommand{\theoremname}{testing}

\newtheorem{theorem}{Theorem}[section]
\newtheorem{lemma}[theorem]{Lemma}
\newtheorem{proposition}[theorem]{Proposition}

\theoremstyle{definition}
\newtheorem{definition}[theorem]{Definition}

\newtheorem{example}[theorem]{Example}

\theoremstyle{remark}
\newtheorem{remark}[theorem]{Remark}

\numberwithin{equation}{section}
\numberwithin{figure}{section}

\DeclareMathOperator{\vrize}{vec}

\DeclareMathOperator{\id}{id}

\DeclareMathOperator*{\argmax}{argmax}

\newcommand{\GL}{\mathrm{GL}}
\newcommand{\SO}{\mathrm{SO}}
\newcommand{\mO}{\mathrm{O}}
\newcommand{\SE}{\mathrm{SE}}
\newcommand{\E}{\mathrm{E}}
\newcommand{\RR}{\mathbb{R}}
\newcommand{\ZZ}{\mathbb{Z}}
\newcommand{\Nc}[1][]{N_{#1}}
\newcommand{\um}{\id}  
\newcommand{\vol}{\mathrm{vol}}

\DeclareMathOperator{\Hom}{Hom}

\makeatletter
\def\grd@save@target#1{%
    \def\grd@target{#1}}
\def\grd@save@start#1{%
    \def\grd@start{#1}}
\def\GridCore{\edef\grd@@target{(\tikzinputsegmentlast)}%
    \tikz@scan@one@point\grd@save@target\grd@@target\relax
    \edef\grd@@start{(\tikzinputsegmentfirst)}%
    \tikz@scan@one@point\grd@save@start\grd@@start\relax
    \draw[minor help lines] (\tikzinputsegmentfirst) grid (\tikzinputsegmentlast);
    \draw[major help lines] (\tikzinputsegmentfirst) grid (\tikzinputsegmentlast);
    \grd@start
    \pgfmathsetmacro{\grd@xa}{\the\pgf@x/1cm}
    \pgfmathsetmacro{\grd@ya}{\the\pgf@y/1cm}
    \grd@target
    \pgfmathsetmacro{\grd@xb}{\the\pgf@x/1cm}
    \pgfmathsetmacro{\grd@yb}{\the\pgf@y/1cm}
    \pgfmathsetmacro{\grd@xc}{\grd@xa + \pgfkeysvalueof{/tikz/grid with coordinates/major step}}
    \pgfmathsetmacro{\grd@yc}{\grd@ya + \pgfkeysvalueof{/tikz/grid with coordinates/major step}}
    \foreach \x in {\grd@xa,\grd@xc,...,\grd@xb}
    {\ifticksB
        \node[anchor=north] at (\x,\grd@ya) {\pgfmathprintnumber{\x}};
        \fi
        \ifticksT
        \node[anchor=south] at (\x,\grd@yb) {\pgfmathprintnumber{\x}};
        \fi
    }
    \foreach \y in {\grd@ya,\grd@yc,...,\grd@yb}
    {\ifticksL
        \node[anchor=east] at (\grd@xa,\y) {\pgfmathprintnumber{\y}};
        \fi
        \ifticksR
        \node[anchor=west] at (\grd@xb,\y) {\pgfmathprintnumber{\y}};
        \fi}
}
\newif\ifticksL
\newif\ifticksR
\newif\ifticksT
\newif\ifticksB
\tikzset{ticks left/.is if=ticksL,
    ticks right/.is if=ticksR,
    ticks on top/.is if=ticksT,
    ticks at bottom/.is if=ticksB,
    ticks left=true,
    ticks at bottom=true,
    ticks right=false,
    ticks on top=false,
    grid with coordinates/.style={
        decorate,decoration={show path construction,
            lineto code={\GridCore
        }}
    },
    minor help lines/.style={
        help lines,
        step=\pgfkeysvalueof{/tikz/grid with coordinates/minor step}
    },
    major help lines/.style={
        help lines,
        line width=\pgfkeysvalueof{/tikz/grid with coordinates/major line width},
        step=\pgfkeysvalueof{/tikz/grid with coordinates/major step}
    },
    grid with coordinates/.cd,
    minor step/.initial=.2,
    major step/.initial=1,
    major line width/.initial=2pt,
}
\makeatother

\newcommand{\rulesep}{\unskip\ \vrule height -1ex\ }

\begin{document}

\author[Gerken]{\vspace*{-1em}\scshape Jan E. Gerken$^{1}$, Jimmy Aronsson$^{1\star}$, Oscar Carlsson$^{1\star}$, Hampus Linander$^{2}$, Fredrik Ohlsson$^{3}$, Christoffer Petersson$^{1,4}$ and Daniel Persson$^{1}$}
\address{\vspace*{-7em}$^{1}$ Chalmers University of Technology, Department of Mathematical Sciences\\
  SE-412\,96 Gothenburg, Sweden\\[1em]
  $^{2}$ Gothenburg University, Department of Physics\\
  SE-412\,96, Gothenburg, Sweden\\[1em]
  $^{3}$ Ume\aa\ University, Department of Mathematics and Mathematical Statistics\\
  SE-901\,87, Ume\aa, Sweden\\[1em]
  $^{4}$ Zenseact\\
  SE-417\,56, Gothenburg, Sweden\\[1em]
  $^{\star}$ equal contribution\vspace{2em}}
\email{gerken@chalmers.se}

\email{jimmyar@chalmers.se}

\email{osccarls@chalmers.se}

\email{hampus.linander@gu.se}

\email{fredrik.ohlsson@umu.se}

\email{christoffer.petersson@chalmers.se}

\email{daniel.persson@chalmers.se}

\title{{\Large \PaperTitle}}
\date{\today}
\begin{nouppercase}
  \maketitle
\end{nouppercase}
\begin{abstract}
We survey the mathematical foundations of geometric deep learning, focusing on group equivariant and gauge equivariant neural networks. We develop gauge equivariant convolutional neural networks on arbitrary manifolds $\mathcal{M}$ using principal bundles with structure group $K$ and equivariant maps between sections of associated vector bundles. We also discuss group equivariant neural networks for homogeneous spaces $\mathcal{M}=G/K$, which are instead equivariant with respect to the global symmetry $G$ on $\mathcal{M}$. Group equivariant layers can be interpreted as intertwiners between induced representations of $G$, and we show their relation to gauge equivariant convolutional layers. We analyze several applications of this formalism, including semantic segmentation  and object detection networks. We also discuss the case of spherical networks in great detail, corresponding to the case $\mathcal{M}=S^2=\SO(3)/\SO(2)$. Here we emphasize the use of Fourier analysis involving Wigner matrices, spherical harmonics and Clebsch--Gordan coefficients for $G=\SO(3)$, illustrating the power of representation theory for deep learning.  
\end{abstract}

\pagebreak
\vspace*{-4em}
\thispagestyle{plain}
\tableofcontents
\pagebreak

\makeatletter

\let\@mkboth\@gobbletwo
\let\@oddhead\pageheader
\let\@evenhead\pageheader
\makeatother

\section{Introduction}
\label{sec:intro}

\noindent Deep learning is an approach to machine learning that uses multiple transformation layers to extract hierarchical features and learn descriptive representations of the input data. These learned features can be applied to a wide variety of classification and regression tasks. Deep learning has for example been enormously successful in tasks such as computer vision, speech recognition and language processing. However,
despite the overwhelming success of deep neural networks we are still at a loss for explaining exactly  why deep learning works so well. One way to address this is to explore the underlying mathematical framework. A promising direction is to consider symmetries as a fundamental design principle for network architectures. This can be implemented by constructing deep neural networks that are compatible with a symmetry group $G$ that acts transitively on the input data. This is directly relevant for instance in the case of spherical signals where $G$ is a rotation group. In practical applications, it was found that equivariance improves per-sample efficiency, reducing the need for data augmentation \cite{muller2021}. For linear models, this has been proven mathematically \cite{2021arXiv210210333E}.

Even more generally, it is natural to consider the question of how to train neural networks in the case of ``non-flat'' data. Relevant applications include fisheye cameras~\cite{coors2018}, biomedicine ~\cite{boomsma2017,elaldi2021}, and cosmological data~\cite{perraudin2019}, just to mention a few situations where the data is naturally curved.  Mathematically, this calls for developing a theory of deep learning on manifolds, or even more exotic structures, like graphs or algebraic varieties. This rapidly growing research field is referred to as \emph{geometric deep learning}  \cite{LeCunGeometric}. The recent book~\cite{bronstein2021} gives an in-depth overview of geometric deep learning.

In this introduction we shall provide a birds-eye view on the subject of geometric deep learning, with emphasis on the mathematical foundations. We will gradually build up the  formalism, starting from a simple semantic segmentation model which already illustrates the role of symmetries in neural networks. We discuss group and gauge equivariant convolutional neural networks, which play a leading role in the paper. The introduction concludes with a summary of our main results, a survey of related literature, as well as an outline of the paper.

\subsection{Warm up: A semantic segmentation model}\label{sec:warmup}

The basic idea of deep learning is that the learning process takes place in multi-layer networks known as deep neural networks of
``artificial neurons'', where each layer receives data from the preceding layer and processes it before sending it to the subsequent layer.
Suppose one wishes to categorize some data sample $x$ according to which class $y$ it belongs to. As a simple example, the input sample $x$ could be an image and the output $y$ could be a binary classification of whether a dog or a cat is present in the image. The first layers of a deep neural network would learn some basic low-level features, such as edges and contours, which are then transferred as input to the subsequent layers. These layers then learn more sophisticated high-level features, such as combinations of edges representing legs and ears. The learning process takes place in the sequence of hidden layers, until finally producing an output $\hat{y}$, to be compared with the correct image class $y$. The better the learning algorithm, the closer the neural network predictions $\hat{y}$ will be to $y$ on new data samples it has not trained on. In short, one wishes to minimize a \emph{loss function} that measures the difference between the output $\hat{y}$ and the class $y$.

More abstractly, let us view a neural network as a nonlinear map $\mathcal{N}$ between a set of input variables  $X$ and output variables  $Y \supseteq \mathcal{N}(X)$. Suppose one performs a transformation $T$  of the input data. This could for instance correspond to a translation or rotation of the elements in $X$. The neural network is said to be \emph{equivariant} to the transformation $T$ if it satisfies
\begin{equation}
    \mathcal{N}(Tx)=T^{\prime}\mathcal{N}(x)\,,
    \end{equation}
    for any input element $x$ and some transformation $T^{\prime}$ acting on $Y$.
A special case of this is that the transformation $T^{\prime}$ is the identity in which case the network is simply \emph{invariant} under the transformation, i.e. $\mathcal{N}(Tx)=\mathcal{N}(x)$. This is for instance the case of convolutional neural networks used for image classification problems, for which we have invariance with respect to translations of the image. A prototypical example of a problem which requires true \emph{equivariance} is the the commonly encountered problem of semantic segmentation in computer vision. Intuitively,  this follows since the output is a pixel-wise segmentation mask which must transform in the same way as the input image. In the remainder of this section we will therefore discuss such a model and highlight its equivariance properties in order to set the stage for later developments.

An image can be viewed as a compactly supported function $f : \mathbb{Z}^2\to \mathbb{R}^{\Nc[\mathrm{in}]}$, where $\mathbb{Z}^2$ represents the pixel grid and $\mathbb{R}^{\Nc[\mathrm{in}]}$ the color space. For example, the case of $\Nc[\mathrm{in}]=1$ corresponds to a gray scale image while $\Nc[\mathrm{in}]=3$ can represent a color RGB image. Analogously, a feature map $f_{i}$ associated with layer $i$ in a CNN can be viewed as a map $\mathbb{Z}^2\to \mathbb{R}^{\Nc[i]}$, where $\mathbb{R}^{\Nc[i]}$ is the space of feature representations.

Consider a neural network $\mathcal{N}$ classifying each individual pixel of RGB images $f_\mathrm{in} : \mathbb{Z}^2 \to \mathbb{R}^3$, supported on $[0, W] \times [0, H] \subset \ZZ^{2}$, into $\Nc[\mathrm{out}]$ classes using a convolutional neural network. Let the sample space of $\Nc[\mathrm{out}]$ classes be denoted $\Omega$ and let $P(\Omega)$ denote the space of probability distributions over $\Omega$.

The network as a whole can be viewed as a map
\begin{equation}
\mathcal{N}: L^{2}(\mathbb{Z}^{2},\, \mathbb{R}^3) \rightarrow L^{2}\left(\mathbb{Z}^{2},\, P(\Omega)\right),\label{eq:35}
\end{equation}
where $L^2(X, Y)$ denote the space of square integrable functions with domain $X$ and co-domain $Y$. This class of functions ensures well-defined convolutions and the possibility to construct standard loss functions.

The co-domain $L^2\left(\mathbb{Z}^2, P(\Omega)\right)$ of $\mathcal{N}$ is usually referred to as semantic segmentations since an element assigns a semantic class probability distribution to every pixel in an input image.

For simplicity, let the model consist of two convolutional layers where the output of the last layer, $f_{\mathrm{out}}$, maps into an $\Nc[\mathrm{out}]$-dimensional vector space, followed by a softmax operator to produce a probability distribution over the classes for every element in $[0, W] \times [0, H]$.
 See Figure \ref{fig:image_classification} for an overview of the spaces involved for the semantic segmentation model.

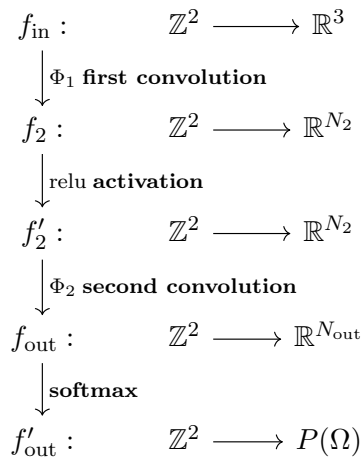
\begin{figure}
\begin{tikzcd}
  f_\mathrm{in}:\arrow{d}{\Phi_1 \;\textbf{first convolution}}& \ZZ^{2} \arrow{r} & \mathbb{R}^3 \\
  f_2:\arrow{d}{\textrm{relu} \;\textbf{activation}}& \ZZ^{2} \arrow{r} & \mathbb{R}^{\Nc[2]} \\
  f_2':\arrow{d}{\Phi_2 \;\textbf{second convolution}}& \ZZ^{2} \arrow{r} & \mathbb{R}^{\Nc[2]} \\
  f_{\mathrm{out}}:\arrow{d}{\textbf{softmax}} & \ZZ^{2} \arrow{r} & \mathbb{R}^{\Nc[\mathrm{out}]} \\
  f_\mathrm{out}^\prime: & \ZZ^{2} \arrow{r} & P(\Omega)
\end{tikzcd}
\caption{Semantic segmentation model. In the first row the name of the first feature map (i.e. the input image) is given by $f_\mathrm{in}$, it maps the domain $\ZZ^{2}$ to RGB values and has support on $[0, W] \times [0, H]$. It is followed by the first convolution, resulting in a feature map $f_2$ that now maps into a $\Nc[2]$ dimensional space corresponding to the $\Nc[2]$ filters of the convolution. After a point-wise activation the second convolution results in a feature map $f_{\mathrm{out}}$ that associates an $\Nc[\mathrm{out}]$-dimensional vector to each point in the domain. This vector is then mapped to a probability distribution over the $\Nc[\mathrm{out}]$ classes using the softmax operator.}
\label{fig:image_classification}
\end{figure}

The standard 2d convolution, for example $\Phi_1$ in Figure~\ref{fig:image_classification}, is given by
\begin{equation}
    [\kappa_1 \star f_\mathrm{in}](x,y) = \sum_{(x', y') \in \mathbb{Z}^2} \kappa_1(x', y')f_\mathrm{in}(x' - x, y' - y)\,.\label{eq:2dCNN}
\end{equation}
This is formally a cross-correlation, but it can be transformed into a convolution by redefining the kernel.

Noting that the two-dimensional translation group $\mathbb{Z}^{2}$ acts on the space of images by left translation
\begin{equation}
  L_{(x, y)} f_\mathrm{in}(x', y') = f_\mathrm{in}(x' - x, y' - y)\,,
\end{equation}
we can rewrite this convolution group theoretically as
\begin{equation}
    [\kappa_1 \star f_\mathrm{in}](x,y) = \sum_{(x',y') \in \mathbb{Z}^2} \kappa_1(x', y') \left[L_{(x,y)}f_\mathrm{in}\right](x',y')\,,\label{eq:2dconvWithTrans}
\end{equation}
where $\kappa_1 \in L^2(\mathbb{Z}^2, \mathbb{R}^{\Nc[2] \times 3})$ is the kernel for the convolution $\Phi_1$. In the context of convolutions the terms kernel and filter appear with slight variations in their precise meaning and mutual relationship in the literature. We will generally use them interchangeably throughout this paper.

This convolution is equivariant with respect to translations $(x, y) \in \mathbb{Z}^2$, i.e.
\begin{equation}
    \left[\kappa\star \left(L_{(x, y)} f_\mathrm{in}\right)\right](x', y')=\left(L_{(x, y)}[\kappa\star f_\mathrm{in}]\right)(x', y')\,.\label{eq:2dconvequiv}
\end{equation}
The point-wise activation function and softmax operator also satisfy this property,
\begin{equation}
    \left[\mathrm{relu}\left(L_{(x, y)} f_\mathrm{in}\right)\right](x', y') = \mathrm{relu}\left(f_\mathrm{in}(x'-x, y' - y)\right) = L_{(x, y)}\left[ \mathrm{relu}(f_\mathrm{in})\right](x',y')\,,\label{eq:reluequiv}
\end{equation}
where $\mathrm{relu}(x) = \max\{0,x\}$, so that the model as a whole is equivariant under translations in $\mathbb{Z}^2$. Note that this equivariance of the model ensures that a translated image produces the corresponding translated segmentation,
\begin{equation}
  \mathcal{N}(L_{(x, y)}f_\mathrm{in}) = L_{(x, y)} \mathcal{N}(f_\mathrm{in})\,,\label{eq:transl_equiv_CNN}
\end{equation}
as illustrated in Figure~\ref{fig:equivariant_segmentation}. The layers in this particular model turns out to be equivariant with respect to translations but there are many examples of non-equivariant layers such as max pooling.
Exactly what restrictions equivariance implies for a layer in an artificial neural network is the topic of Sections~\ref{sec:gaugeeq} and \ref{sec:equiv-conv-layers}.

\begin{figure}
\begin{tikzcd}
  \raisebox{-0.5\height}{\includegraphics[width=0.2\textwidth]{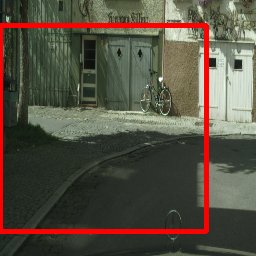}} \arrow{r}{\mathcal{N}}\arrow{d}{L_{(x, y)}} & \raisebox{-0.5\height}{\includegraphics[width=0.2\textwidth]{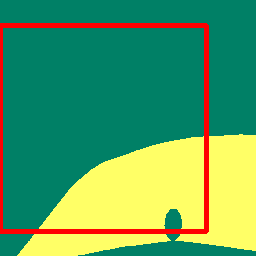}}\arrow{d}{L_{(x, y)}} \\
  \raisebox{-0.5\height}{\includegraphics[width=0.2\textwidth]{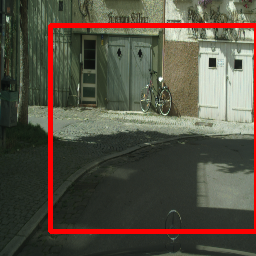}} \arrow{r}{\mathcal{N}} & \raisebox{-0.5\height}{\includegraphics[width=0.2\textwidth]{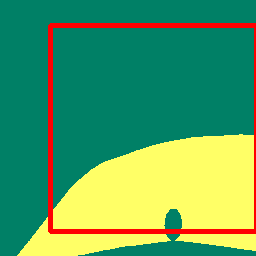}}
\end{tikzcd}
\caption{$\mathbb{Z}^{2}$ equivariance of a semantic segmentation model classifying pixels into classes $\Omega = \{\text{road}, \text{non-road}\}$. The network $\mathcal{N}$ maps input images, indicated by the content of the red rectangles in the left column, to semantic masks indicated by the corresponding content of the red rectangles in the right column. Image and semantic mask from \cite{Cordts2016Cityscapes}.}
\label{fig:equivariant_segmentation}
\end{figure}

\subsection{Group equivariant convolutional neural networks}
\label{sec:GCNN_background}
Convolutional neural networks are ordinary feed-forward networks that make use of convolutional operations of the form \eqref{eq:2dCNN}. One of the main reasons for their power is their aforementioned {\it translation equivariance}~\eqref{eq:transl_equiv_CNN}, which implies that a translation of the pixels in an image produces an overall translation of the convolution. Since each layer is translation equivariant all representations will be translated when the input data is translated. Furthermore, the local support of the convolution allows for efficient weight sharing across the input data.

 Notice that $\mathbb{Z}^2$ in the semantic segmentation model is a group with respect to addition, and the space of feature representations $\mathbb{R}^N$ is a vector space. It is therefore natural to generalize this construction by replacing $\mathbb{Z}^2$ with an arbitrary group $G$ and $\mathbb{R}^N$ with a vector space $V$. The feature map then generalizes to
\begin{equation}
 f : G \to V,
 \label{featuremapgeneral}
\end{equation}
 and the convolution operation~\eqref{eq:2dCNN} to
\begin{equation}
 [\kappa \star f](g)=\int_G \kappa(g^{-1}h) f(h)\dd h\,,
 \label{convgeneral}
\end{equation}
where $\dd h$ is a left-invariant Haar measure on $G$. If $G$ is a discrete group such as $\mathbb{Z}^2$, then $\dd h$ becomes the counting measure and the integral reduces to a sum.

The generalized kernel $\kappa : G \to \mathrm{Hom}(V,W)$ appearing in~\eqref{convgeneral} is a function from the group to homomorphisms between $V$ and some feature vector space $W$, which can in general be different from $V$. Consequently, the result of the convolution is another feature map
\begin{equation}
    [\kappa \star f] : G \to W \,,
\end{equation}
and in analogy with the terminology for ordinary CNNs we refer to the convolution~(\ref{convgeneral}) itself as a layer in the network. The general form of the convolution~(\ref{convgeneral}) is equivariant with respect to the left-translation by $G$:
\begin{equation}
    [\kappa \star L_hf](g)=L_h[\kappa \star f](g)\,,
\end{equation}
motivating the term equivariant layer.

In the convolution~(\ref{convgeneral}), the kernel $\kappa$ in the integral over $G$ is transported in the group using the right action of $G$ on itself. This transport corresponds to the translation of the convolutional kernel in~\eqref{eq:2dCNN} and generalizes the weight sharing in the case of a locally supported kernel $\kappa$.

In this paper we will explore the structure of group equivariant convolutional neural networks and their further generalizations to manifolds. A key step in this direction is to expose the connection with the theory of fiber bundles as well as the representation theory of $G$. To this end we shall now proceed to discuss this point of view.

It is natural to generalize the above construction even more by introducing a choice of subgroup $K\subset G$ for each feature map, and a choice of representation $\rho$ of $K$, i.e.,
\begin{equation}
\rho : K \to \GL(V)\,,
\end{equation}
where $V$ is a vector space. Consider then the coset space $G/K$ and a vector bundle $E\xrightarrow{\pi} G/K$ with characteristic fiber $V$. Sections of $E$ are maps $s : G/K\to E$ which locally can be represented by vector-valued functions
\begin{equation}
f: G/K \to V\,.\label{eq:hom_feature_map}
\end{equation}
These maps can be identified with the feature maps of a group equivariant convolutional neural network. Indeed, in the special case when $G=\mathbb{Z}^2$, $K$ is trivial and $V=\mathbb{R}^{\Nc}$, we recover the ordinary feature maps $f : \mathbb{Z}^2\to \mathbb{R}^{\Nc}$ of a CNN. When the representation $\rho$ is non-trivial the network is called \emph{steerable} (see~\cite{weiler2018,weiler2019}).

As an example, consider spherical signals, i.e.\ the case in which the input feature map is defined on the two-sphere $S^2$ and can therefore be written in the form \eqref{eq:hom_feature_map}, since $\SO(3)/\SO(2)=S^2$ and hence $G=\SO(3)$ and $K=\SO(2)$. Consequently, feature maps correspond to sections $f :  \SO(3)/\SO(2) \to V$.

This construction allows us to describe $G$-equivariant CNNs using a very general mathematical framework. The space of feature maps is identified with the space of sections $\Gamma(E)$, while maps between feature maps, which we refer to as layers, belong to the space of so called $G$-equivariant intertwiners if they are equivariant with respect to the right action of $G$ on the bundle $E$. This implies that many properties of group equivariant CNNs can be understood using the representation theory of $G$.

\subsection{Group theory properties of machine learning tasks}
After having laid out the basic idea of group equivariant neural networks, in this section we will make this more concrete by discussing the group theoretical properties of the common computer vision tasks of image classification, semantic segmentation and object detection.

Let us first focus on the case of classification. In this situation the input data consists of color values of pixels and therefore the input feature map $f_\mathrm{in}$  transforms under the regular representation $\pi_{\mathrm{reg}}$ of $G$, i.e. as a collection of scalar fields:
\begin{equation}
  f_\mathrm{in}(x) \to [\pi_{\mathrm{reg}}(g)f_\mathrm{in}](x) = f_\mathrm{in}(\sigma^{-1}(g)x)\,, \qquad g\in G\,,
\end{equation}
where $\sigma$ is a representation of $G$ that dictates the transformation of the image. In other words, the color channels are not mixed by the spatial transformations. However, for image classification the identification of images should be completely independent of how they are transformed. For this reason we expect the network $\mathcal{N}(f_\mathrm{in})$ to be \emph{invariant} under $G$,
\begin{align}
  \mathcal{N}(\pi_{\mathrm{reg}}(g) f_{\mathrm{in}}) = \mathcal{N}(f_{\mathrm{in}})\,,\qquad g\in G\,.
\end{align}

On the other hand, when doing semantic segmentation we are effectively classifying each individual pixel, giving a \emph{segmentation mask} for the objects we wish to identify, as described in Section~\ref{sec:warmup}. This implies that the output features must transform in the same way as the input image. In this case one should not demand invariance of $\mathcal{N}$, but rather non-trivial \emph{equivariance}
\begin{equation}
  \mathcal{N}(\pi_{\mathrm{reg}}(g) f_{\mathrm{in}}) = \pi_{\mathrm{reg}}(g) [\mathcal{N}(f_{\mathrm{in}})]\,,\qquad g\in G\,,
\end{equation}
where the regular representation $\pi_{\mathrm{reg}}$ on the right-hand side transforms the output feature map of the network. The group-theoretic aspects of semantic segmentation are further explored in Section~\ref{sec:Eqdeeparch}.

Object detection is a slightly more complicated task. The output in this case consists of bounding boxes around the objects present in the image together with class labels. We may view this as a generalization of the semantic segmentation, such that, for each pixel, we get a class probability vector $p\in\mathbb{R}^{\Nc}$ (one of the classes labels the background) and three vectors $a,v_1,v_2\in\mathbb{R}^{2}$ that indicate the pixel-position of the upper-left corner and two vectors that span the parallelogram of the associated bounding box. Hence, the output is an object $(p,a,v_1, v_2)\in\mathbb{R}^{\Nc+6}$ for each pixel. The first $\Nc$ components of the output feature map $f:\mathbb{R}^{2} \rightarrow \mathbb{R}^{\Nc+4}$ transform as scalars as before. The three vectors $a, v_1, v_2$ on the other hand transform in a non-trivial two-dimensional representation $\rho$ of $G$. The output feature map $f_\mathrm{out}=\mathcal{N}(f_\mathrm{in})$ hence transforms in the representation $\pi_{\mathrm{out}}$ according to
\begin{align}
  (\pi_{\mathrm{out}}f_{\mathrm{out}})(x)=\rho_{\mathrm{out}}(g)f_\mathrm{out}(\sigma^{-1}(g)x)\,,
\end{align}
where $\rho_{\mathrm{out}}=\id_{\Nc}\oplus\,\rho\oplus\rho\oplus\rho$. The network is then equivariant with respect to $\pi_{\mathrm{reg}}$ in the input and $\pi_{\mathrm{out}}$ in the output if
\begin{align}
  \mathcal{N}(\pi_{\mathrm{reg}}(g)f_{\mathrm{in}})=\pi_{\mathrm{out}}(g)[\mathcal{N}(f_{\mathrm{in}})]\,.
\end{align}
For more details on equivariant object detection see Sections~\ref{sec:obj_detection} and \ref{ex:SE3Outputs}.

\subsection{Gauge equivariant networks}

 In the group equivariant networks discussed above, we exploited that the domain of the input feature map had global symmetries. Using inspiration from the physics of gauge theories and general relativity, the framework of equivariant layers can be extended to feature maps which are defined on a general manifold $\mathcal{M}$. A manifold can be thought of as consisting of a union of \emph{charts}, giving rise to coordinates on $\mathcal{M}$, subject to suitable gluing conditions where the charts overlap. However, the choice of charts is arbitrary and tensors transform in a well-defined way under changes of charts. Such transformations are called \emph{gauge transformations} in physics and correspond to the freedom of making local coordinate transformations across the manifold. Feature maps can in this context be viewed as sections of vector bundles associated to a principal bundle, called \emph{fields} in physics parlance. A \emph{gauge equivariant network} for such fields consists of layers which are equivariant with respect to change of coordinates in $\mathcal{M}$, such that the output of the network transforms like a tensor.

In order to realize gauge equivariant layers using convolutions, we need to shift the kernel across $\mathcal{M}$. In general, a manifold does not have any global symmetries to utilize for this. Instead, one may use \emph{parallel transport} to move the filter on the manifold. This transport of features will generically depend on the chosen path. A gauge equivariant CNN is constructed precisely such that it is independent of the choice of path. In other words, making a coordinate transformation at $x\in \mathcal{M}$ and transporting the filter to $y\in \mathcal{M}$ should give the same result as first transporting the filter from $x$ to $y$, and then performing the coordinate transformation. Therefore, the resulting layers are gauge equivariant. The first steps toward a general theory of gauge equivariant convolutional neural networks on manifolds were  taken in \cite{CohenGauge2019,CohenChengGauge2019}.

\subsection{Summary of results}
This paper aims to put geometric aspects of deep learning into a mathematical context. Our intended audience includes mathematicians, theoretical physicists as well as mathematically minded machine learning experts.

The main contribution of this paper is to give a mathematical overview of the recent developments in group equivariant and gauge equivariant neural networks. We strive to develop the theory in a mathematical fashion, emphasizing the bundle perspective throughout. In contrast to most of the literature on the subject we start from the point of view of neural networks on arbitrary manifolds $\mathcal{M}$ (sometimes called ``geometric deep learning''). This requires gauge equivariant networks which we develop using the gauge theoretic notions of principal bundles and associated vector bundles. Feature maps will be sections of associated vector bundles. These notions have been used previously in different equivariant architectures \cite{bronstein2021,CohenChengGauge2019,cohen2018} and we present a unified picture. We analyze when maps between feature spaces are equivariant with respect to gauge transformations and define \emph{gauge equivariant layers} accordingly in Section \ref{sec:gaugeeq}. Furthermore, we develop gauge equivariant convolutional layers for arbitrary principal bundles in Section \ref{subsec:gauge_equivariant_convoultion} and thereby define gauge equivariant CNNs. In Section~\ref{subsec:gauge_equivariant_convoultion} we generalize the gauge equivariant convolution presented in~\cite{CohenChengGauge2019} to the principal bundle setting.

Different principal bundles $P$ describe different local (gauge) symmetries. One example of a local symmetry is the freedom to choose a basis in each tangent space or, in other words, the freedom to choose local frames of the frame bundle $P = \mathcal{LM}$. In this case, local gauge transformations transform between different bases in tangent spaces. When the manifold is a homogeneous space $\mathcal{M} = G/K$, the global symmetry group forms a principal bundle $P=G$. Here, the local symmetry is the freedom to perform translations that do not move a given point, e.g. the north pole on $S^2$ being invariant to rotations about the $z$-axis, but we are more interested in the global translation symmetry for this bundle. Building on \cite{aronsson2021}, we motivate group equivariant networks from the viewpoint of homogeneous vector bundles in Section \ref{sec:equiv-conv-layers} and connect these to the gauge equivariant networks. In Section \ref{subsec:equiv-intensity}, we discuss equivariance with respect to intensity; point-wise scaling of feature maps.

Furthermore, starting from a very general setup of a symmetry group acting on functions defined on topological spaces, we connect and unify various equivariant convolutions that are available in the literature.

Having developed the mathematical framework underlying equivariant neural networks, we give an overview of equivariant nonlinearities and in particular extend vector field nonlinearities to arbitrary semi-direct product groups, cf.\ Proposition~\ref{prop:vfieldnonlinearity}. We review the entire equivariant network architecture for semantic segmentation and object detection and  the associated representations.

Finally, we consider spherical networks corresponding to data defined on the two-sphere $\mathcal{M}=S^2=\SO(3)/\SO(2)$. For this case, we explain how convolutions can be computed in Fourier space and we give a detailed description of the convolution in terms of Wigner matrices and Clebsch-Gordan coefficients, involving in particular the decomposition of tensor products into irreducible representations of $\SO(3)$. This is well-known in the mathematical physics community but we collect and present this material in a coherent way which we found was lacking in the literature. We illustrate the formalism in terms of object detection for $\SE(3)$ (see Section~\ref{ex:SE3Outputs}).

\subsection{Related literature}
\label{sec:related-literature}
The notion of geometric deep learning was first discussed in the seminal paper of LeCun et.\ al.~ \cite{LeCunGeometric}. They emphasized the need for neural networks defined on arbitrary data manifolds and graphs. In a different development, group equivariant convolutional neural networks, which incorporate global symmetries of the input data beyond the translational equivariance of ordinary CNNs, were proposed by Cohen and Welling \cite{Cohen2016}.
Kondor and Trivedi \cite{Kondor2018} proved that for compact groups $G$, a neural network architecture can be $G$-equivariant if and only if it is built out of convolutions of the form~\eqref{convgeneral}. The theory of equivariant neural networks on homogeneous spaces $G/K$ was further formalized in \cite{cohen2018} using the theory of vector bundles in conjunction with the representation theory of $G$. A proposal for including \emph{attention}~\cite{2017arXiv170603762V} into group equivariant CNNs was also put forward in~\cite{2020arXiv200203830R}. Equivariant normalizing flows were recently constructed in \cite{satorras2021a}.

The case of neural networks on spheres has attracted considerable attention due to its extensive applicability. Group equivariant CNNs on $S^2$ were studied in \cite{cohen2018b} by  implementing efficient Fourier analysis on $\mathcal{M} = S^2$ and $G = \SO(3)$. Some applications instead benefit from equivariance with respect to azimuthal rotations, rather than arbitrary rotations in $\SO(3)$ \cite{toft2021azimuthal}. One such example is the use of neural networks in self-driving cars to identify vehicles and other objects.

The approach presented here follows \cite{cohen2018b} and extends the results in the reference at some points. The Clebsch--Gordan nets introduced in \cite{kondor2018a} have a similar structure but use as nonlinearities tensor products in the Fourier domain, instead of point-wise nonlinearities in the spatial domain. Several modifications of this approach led to a more efficient implementation in \cite{cobb2020}. The constructions mentioned so far involve convolutions which map spherical features to features defined on $\SO(3)$. The construction in \cite{esteves2018} on the other hand uses convolutions which map spherical features to spherical features, at the cost of restricting to isotropic filters. Isotropic filters on the sphere have also been realized by using graph convolutions in \cite{defferrard2020deepsphere}. In \cite{esteves2020b}, spin-weighted spherical harmonics are used to obtain anisotropic filters while still keeping the feature maps on the sphere.

A somewhat different approach to spherical signals is taken in \cite{jiang2018}, where a linear combination of differential operators acting on the input signal is evaluated. Although this construction is not equivariant, an equivariant version has been developed in \cite{shen2021}.

An important downside to many of the approaches outlined above is their poor scaling behavior in the resolution of the input image. To improve on this problem, \cite{mcewen2021} introduces scattering networks as an equivariant preprocessing step.

Aside from the equivariant approaches to spherical convolutions, much work has also been done on modifying flat convolutions in $\RR^{2}$ to deal with the distortions in spherical data without imposing equivariance \cite{coors2018, boomsma2017, su2017, monroy2018}.

A different approach to spherical CNNs was proposed in \cite{CohenGauge2019}, were the basic premise is to treat $S^2$ as a manifold and use a gauge equivariant CNN, realized using the icosahedron as a discretization of the sphere. A general theory of gauge equivariant CNNs is discussed in \cite{CohenChengGauge2019}. Further developments include gauge CNNs on meshes and grids~\cite{CohenMesh,2020arXiv200601570W} and applications to lattice gauge theories~\cite{Favoni:2020reg,Luo:2020stn}.
In this paper we continue to explore the mathematical structures underlying gauge equivariant CNNs and clarify their relation to GCNNs.

 A further important case studied extensively in the literature are networks equivariant with respect to $\E(n)$ or $\SE(n)$. These networks have been applied with great success to 3d shape classification~\cite{weiler2018, thomas2018}, protein structure classification~\cite{weiler2018}, atomic potential prediction~\cite{kondor2018b} and medical imaging~\cite{muller2021, worrall2018, winkels2018}.

  The earliest papers in the direction of $\SE(n)$ equivariant network architectures extended classical GCNNs to 3d convolutions and discrete subgroups of $\SO(3)$~\cite{worrall2018, winkels2018, marcos2017, weiler2018a}.

  Our discussion of $\SE(n)$ equivariant networks is most similar to the $\SE(3)$ equivariant networks in \cite{weiler2018}, where an equivariance constraint on the convolution kernel of a standard 3d convolution is solved  by expanding it in spherical harmonics. A similar approach was used earlier in \cite{worrall2017} to construct $\SE(2)$ equivariant networks using circular harmonics. A comprehensive comparison of different architectures which are equivariant with respect to $\E(2)$ and various subgroups was given in~\cite{weiler2019}.

  Whereas the aforementioned papers specialize standard convolutions by imposing constraints on the filters and are therefore restricted to data on regular grids, \cite{thomas2018, kondor2018b} operate on irregular point clouds and make the positions of the points part of the features. These networks operate on non-trivially transforming input features and also expand the convolution kernels into spherical harmonics but use Clebsch--Gordan coefficients to combine representations.

  So far, the applied part of the literature is mainly focused on specific groups (mostly rotations and translations in two and three dimensions and their subgroups). However, in the recent paper \cite{lang2020}, a general approach to solve the kernel constraint for arbitrary compact groups is constructed by deriving a Wigner--Eckart theorem for equivariant convolution kernels. The implementation in \cite{finzi2021} uses a different algorithm to solve the kernel constraint for matrix groups and allows to automatically construct equivariant CNN layers.

  The review \cite{esteves2020} discusses various aspects of equivariant networks. The book~\cite{bronstein2021} gives an exhaustive survey of many of the developments related to geometric deep learning and equivariant CNNs.

\subsection{Outline of the paper}
Our paper is structured as follows. In Section~\ref{sec:gaugeeq} we introduce gauge equivariant convolutional neural networks on manifolds. We discuss global versus local symmetries in neural networks. Associated vector bundles are introduced and maps between feature spaces are defined. Gauge equivariant CNNs are constructed using principal bundles over the manifold. We conclude Section~\ref{sec:gaugeeq} with a discussion of some concrete examples of how gauge equivariant CNNs can be implemented for neural networks on graphs. In Section~\ref{sec:equiv-conv-layers} we restrict to homogeneous spaces $\mathcal{M}=G/K$ and introduce homogeneous vector bundles. We show that when restricting to homogeneous spaces the general framework of Section~\ref{sec:gaugeeq} gives rise to group equivariant convolutional neural networks with respect to the global symmetry $G$. We also introduce intensity equivariance and investigate its compatibility with group equivariance. Still  restricting to global symmetries, Section~\ref{sec:GC} explores the form of the convolutional integral and the kernel constraints in various cases.  Here, the starting point are vector valued maps between arbitrary topological spaces. This allows us to investigate convolutions between non-scalar features as well as non-transitive group actions. As an application of this we consider semi-direct product groups which are relevant for steerable neural networks. In Section~\ref{sec:Eqdeeparch} we assemble the pieces and discuss how one can construct equivariant deep architectures using our framework. To this end we begin by discussing how nonlinearities can be included in an equivariant setting. We further illustrate the group equivariant formalism by analyzing deep neural networks for semantic segmentation and object detection tasks.  In Section~\ref{sec:spherical} we provide a detailed analysis of spherical convolutions. Convolutions on $S^2$ can be computed using Fourier analysis with the aid of spherical harmonics and Wigner matrices. The output of the network is characterized through certain tensor products which decompose into irreducible representations of $\SO(3)$. In the final Section~\ref{sec:conclusions} we offer some conclusions and suggestions for future work.

\section{Gauge equivariant convolutional layers}
\label{sec:gaugeeq}
\noindent In this section we present the structure needed to discuss local transformations and symmetries on general manifolds. We also discuss the gauge equivariant convolution in~\cite{CohenChengGauge2019} for features defined on a smooth manifold $\mathcal{M}$ along with lifting this into the principal bundle formalism. We end this section by expanding on two applications of convolutions on manifolds via a discretization to a mesh and compare these to the convolution on the smooth manifold.

\subsection{Global and local symmetries}\label{sec:global_local_symmetries}
In physics the concept of symmetry is a central aspect when constructing new theories. 
An object is symmetric to a transformation if applying that transformation leaves the object as it started.

In any theory, the relevant symmetry transformations on a space form a \emph{group} $K$. When the symmetry transformations act on vectors via linear transformations, we have a \emph{representation} of the group. This is needed since an abstract group has no canonical action on a vector space; to allow the action of a group $K$ on a vector space $V$, one must specify a representation. Formally a representation is a map $\rho:K\to \GL(\dim(V),F)$ into the space of all invertible $\dim(V)\times \dim(V)$ matrices over a field $F$. Unless otherwise stated, we use complex representations ($F=\mathbb{C}$). In contrast, real representations use $F=\mathbb{R}$. The representation needs to preserve the group structure, i.e. 
\begin{equation}
    \rho(kk')=\rho(k)\rho(k')\,,
\end{equation}
for all $k,k'\in K$. In particular, $\rho(k^{-1})=\rho(k)^{-1}$ and $\rho(e)=\mathrm{id}_V$ where $e \in K$ is the identity element. 
\begin{remark}
    There are several ways to denote a representation and in this paper we use $V_\rho$ to represent the representation $\rho$ acting on some vector space, hence $V_\rho$ and $V_\eta$ will be viewed as two (possibly) different vector spaces acted on by $\rho$ and $\eta$ respectively.
\end{remark}

Returning to symmetries there exists, in brief, two types of symmetries: \emph{global} and \emph{local}. An explicit example of a global transformation of a field $ \phi:\mathbb{R}^{2}\to \mathbb{R}^{3} $ is a rotation $ R\in\SO(2) $ of the domain as
\begin{equation}
    \phi(x)\xrightarrow{R}\phi'(x)=\rho(R)\phi(\eta(R^{-1})x)\,,
\end{equation}
where $ \rho $ is a representation for how $\SO(2)$ acts on $ \mathbb{R}^{3} $ and $\eta$ is the standard representation for how $\SO(2)$ acts on $\mathbb{R}^2$.
\begin{remark}
    Note that this transformation not only transforms the vector $\phi(x)$ at each point $x \in \mathbb{R}^2$, but also moves the point $x$ itself.
\end{remark}
\begin{example}
    If we let $R\in \SO(2)$ be the action of rotating with an angle $R$, then the standard representation of $R$ would be
    \begin{equation}
        \eta(R)=\begin{pmatrix} \cos(R) & \sin(R) \\ -\sin(R) & \cos(R) \end{pmatrix}\,.
    \end{equation}
\end{example}

\begin{example}
    An example of a rotationally symmetric object is when $\phi$ is three scalar fields, i.e. $\rho_3(R)=1\oplus1\oplus1$, where each scalar field only depends on the distance from the origin. This yields
\begin{equation}
     \phi'(x)=\rho_3(R)\phi(R^{-1}x)=\phi(R^{-1}x)=\phi(x)\,,
\end{equation}
since rotation of the domain around the origin leaves distances to the origin unchanged.
\end{example}

With the global transformation above we act with the same transformation on every point; with local transformations we are allowed to transform the object at each point differently. We can construct a similar explicit example of a local symmetry: Given a field $ \phi:\mathbb{R}^{2}\to \mathbb{C} $ we can define a local transformation as
\begin{equation}\label{eq:field_example_gauge}
    \phi(x)\to \phi'(x)=\exp(if(x))\phi(x)\,,
\end{equation}
where $ f:\mathbb{R}^{2}\to [0,2\pi) $. This is local in the sense that at each point $x$ the field $\phi$ is transformed by $\exp(if(x))$ where $f$ is allowed to vary over $\mathbb{R}^2$. We will refer to a local transformation as a \emph{gauge transformation}. If an object is invariant under local (gauge) transformations, it is called \emph{gauge invariant} or that it has a \emph{gauge symmetry}. The group consisting of all local transformations is called the \emph{gauge group}.
\begin{remark}
    The example of a local transformation presented in~\eqref{eq:field_example_gauge} does not move the base point but in general there are local transformations that move the base point; the main example of which is locally defined diffeomorphisms which are heavily used in general relativity. In this section we will only consider transformations that do not move the base point and gauge transformations falls in this category. For a more detailed discussion on this see Remark~5.3.10 in \cite{hamiltonMathematicalGaugeTheory2017}.
\end{remark}
\begin{example}
    A simple example of a gauge invariant object using~\eqref{eq:field_example_gauge} is the field
    \begin{equation}\label{eq:gauge_invariant_field}
        \overline{\phi(x)}\phi(x)\,,
    \end{equation}
    where the bar denotes complex conjugate. This works since multiplication of complex numbers is commutative. This is an example of a commutative gauge symmetry. 
\end{example}
Note that in the above example the phase of $\phi$ at each point can be transformed arbitrarily without affecting the field in~\eqref{eq:gauge_invariant_field}, hence we have a redundancy in the phase of $\phi$. For any object with a gauge symmetry one can get rid of the redundancy by choosing a specific gauge.
\begin{example}
    The phase redundancy in the above example can be remedied by choosing a phase for each point. For example this can be done by 
    \begin{equation}
        \phi(x)\to\phi'(x)= |\phi(x)| = \exp\Big(-i\arg\big(\phi(x)\big)\Big)\phi(x)\,.
    \end{equation}
    Thus $\phi'$ only takes real values at each point and since~\eqref{eq:gauge_invariant_field} is invariant to this transformation we have an equivalent object with real fields.
\end{example}

To introduce equivariance, let $ K $ be a group acting on two vector spaces $V_\rho$, $V_\eta $ through representations $ \rho$, $\eta $ and let $ \Phi:V_\rho\to V_\eta $ be a map. We say that $ \Phi $ is \emph{equivariant} with respect to $ K $ if for all $ k\in K $ and $ v\in V_\rho $,
\begin{equation}\label{eq:equivariance}
    \Phi(\rho(k)v)=\eta(k)\Phi(v)\,,
\end{equation}
or equivalently, with the representations left implicit, expressed $ \Phi\circ k=k\circ \Phi $.

\subsection{Motivation and intuition for the general formulation}
When a neural network receives some numerical input, unless specified, it does not know what basis the data is expressed in, be it local or global.  
The goal of the general formulation presented in Section~\ref{sec:general_gauge} is thus that if two numerical inputs are related by a transformation between equivalent states, e.g.\ change of basis, then the output from a layer should be related by the same transformation: if $ u=k\triangleright v $ then $ \phi(u)=k\triangleright \phi(v) $, where $k\triangleright$ is a general action of the group element on the input data $v$; in words $\phi$ is an equivariant map. The intuition for this is that it ensures that there is no basis dependence in the way $\phi$ acts on data. 
        
To construct such a map $ \phi $ we will construct objects which are gauge invariant but contain components that transform under gauge transformations. We then define a map $ \Phi $ on those using $ \phi $ to act on one of the components.

Our objects will be equivalence classes consisting of two elements: an element that specifies which gauge the numerical values are in and one element which are the numerical values. By construction these will transform in “opposite” ways and hence each equivalence class is gauge invariant. The intuition for the first element is that it will serve as book-keeping for the theoretical formulation.

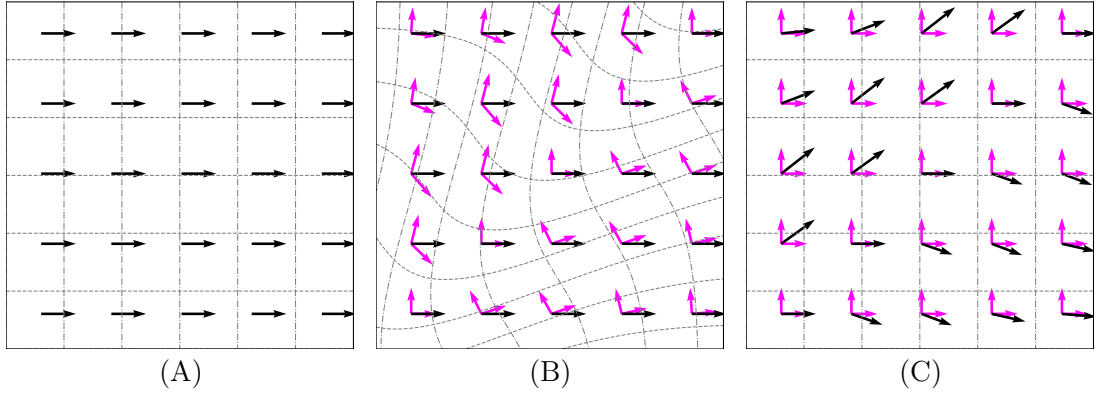
\begin{figure}[t]
    \centering
    \begin{subfigure}[t]{0.3\linewidth}
        \begin{tikzpicture}
            \node[anchor=south west] at (0,0) {\includegraphics[width=\linewidth]{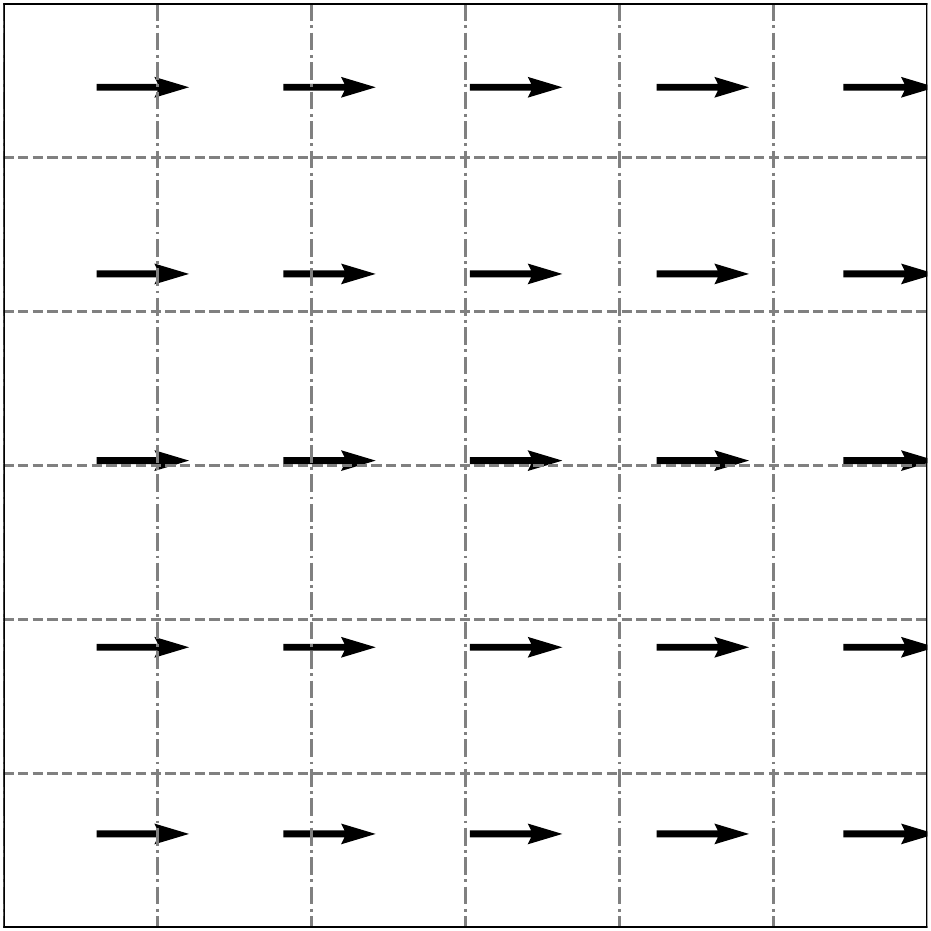}};
            \node[anchor=south west] at (2.05,-0.5) {(A)};
        \end{tikzpicture}
    \end{subfigure}
    \rulesep
    \begin{subfigure}[t]{0.3\linewidth}
        \begin{tikzpicture}
            \node[anchor=south west] at (0,0) {\includegraphics[width=\linewidth]{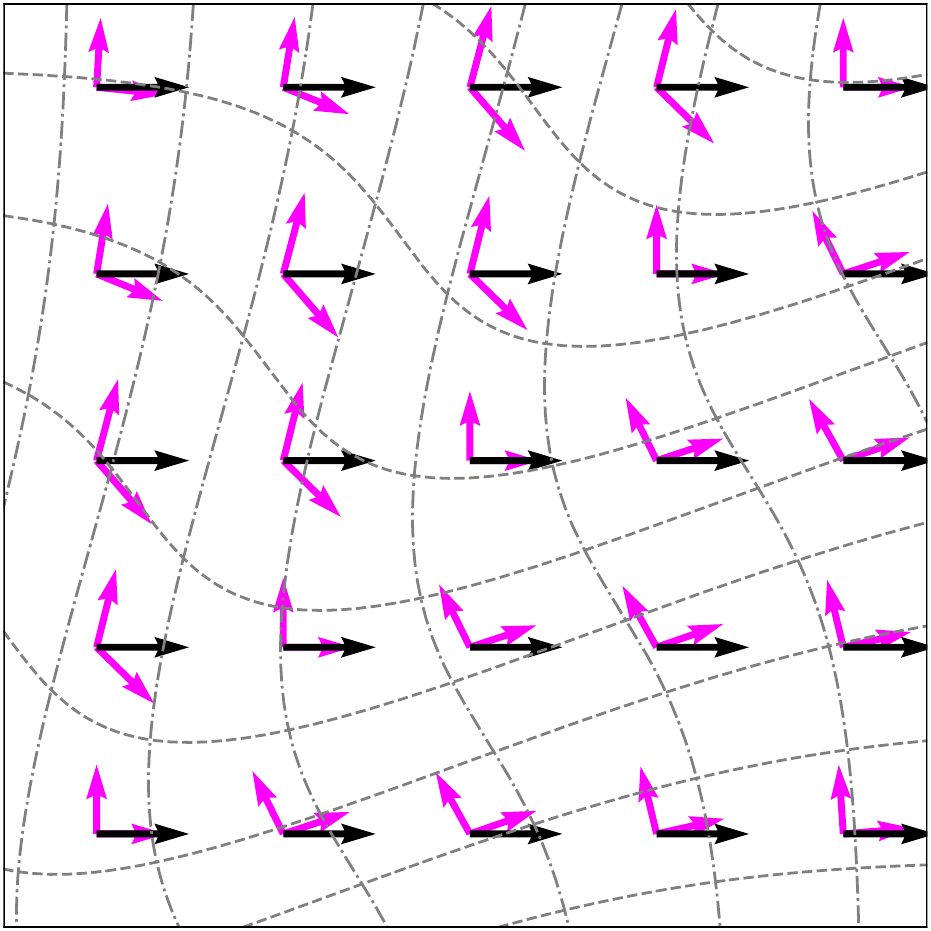}};
            \node[anchor=south west] at (2.05,-0.5) {(B)};
        \end{tikzpicture}
    \end{subfigure}
    \rulesep
    \begin{subfigure}[t]{0.3\linewidth}
        \begin{tikzpicture}
            \node[anchor=south west] at (0,0) {\includegraphics[width=\linewidth]{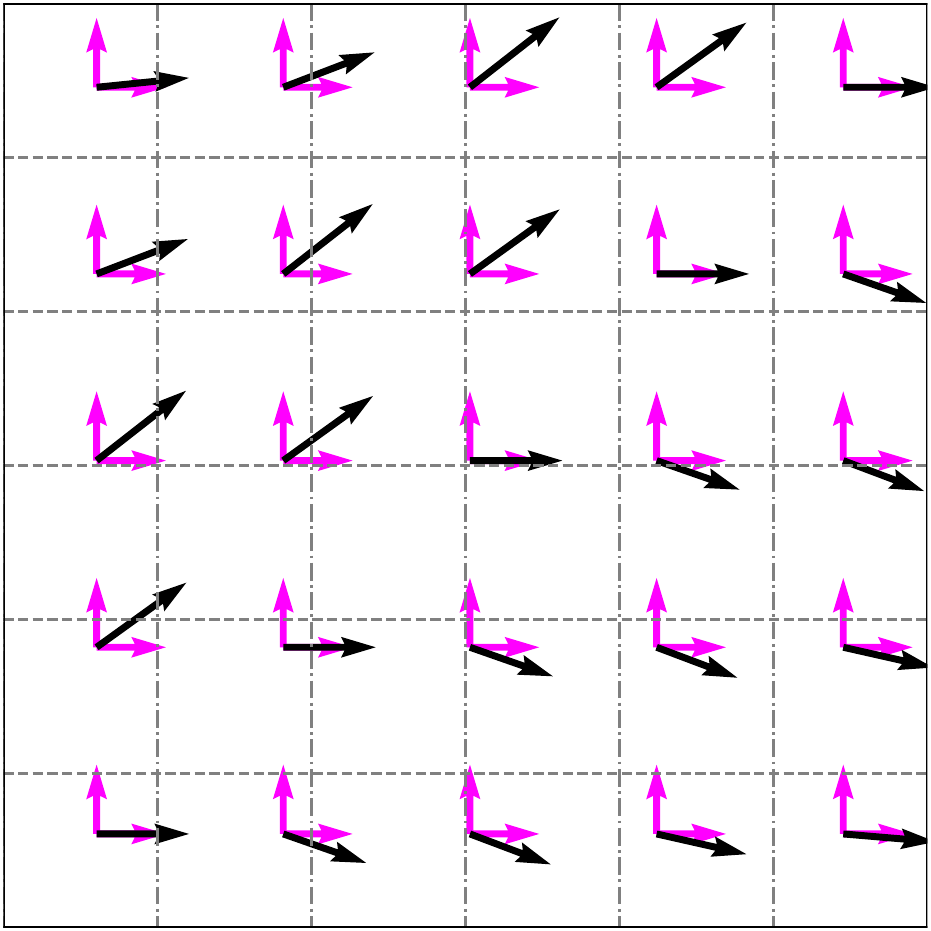}};
            \node[anchor=south west] at (2.05,-0.5) {(C)};
        \end{tikzpicture}
    \end{subfigure}
    \vspace{-0.4cm}
    \caption{A vector field on a base space $\mathcal{M}$ in different local coordinates. The local coordinates can be viewed either as induced by the coordinate system $u:\mathcal{M}\to \mathbb{R}^2$ or as a section $\omega:\mathcal{M}\to P$ of the frame bundle $P = \mathcal{LM}$ that specifies a local basis at each point -- see Section \ref{sec:general_gauge} for details. (A)  The local basis for the vector field is at every point aligned with the standard basis in $\mathbb{R}^2$. (B) A new local basis induced by the coordinate system $u':\mathcal{M}\to \mathbb{R}^2$, or equivalently as a section $\omega':\mathcal{M}\to {P}$. The transition map is either $u'\circ u^{-1}$ or $\omega'(x)=\omega(x)\triangleleft\sigma(x)$ where $\sigma:\mathcal{M}\to K$ is a map from the base space to the gauge group $K$. (C) Here the vector field is at each point expressed in the new local coordinates. This illustrates that two vector fields can look very different when expressed in components, but when taking the local basis into account, they are the same.} \label{fig:local_gauge_transformation_of_vector_field}
\end{figure}

The input to a neural network can then be interpreted as the numerical part of an equivalence class, and if two inputs are related by a gauge transformation the outputs from $ \phi $ are also related by the same transformation. (Through whatever representation is chosen to act on the output space.)

Images can be thought of as compactly supported vector valued functions on $ \mathbb{R}^{2} $ (or $ \mathbb{Z}^{2} $ after discretization) but when going to a more general surface, a smooth $ d $-dimensional manifold, the image needs instead to be interpreted as a section of a fiber bundle. In this case we cannot, in general, have a global transformation and we really need the local gauge viewpoint.

\begin{example}
    If our image is an RGB image every pixel has a “color vector” and one could apply a random permutation of the color channels for each pixel. Applying a random $\SO(3) $ element would not quite work since each element of the “color vector” needs to lie in $ [0,255]\cap\mathbb{Z} $ and under a random $\SO(3) $ element a “color vector” could be rotated outside this allowable space.
\end{example}

\begin{remark}
    As mentioned in the previous section there are local transformations that moves base points locally, e.g. a diffeomorphism $\psi:\mathcal{M}\to \mathcal{M}$  such that $\psi(x)=x$ for some $x\in\mathcal{M}$. Note that this is not the same as a local change of coordinates at each point. In the neighborhood, and the tangent space, of $x$ though one can view this as a local change of coordinates. Hence if one transforms a feature map at each point, the transformation of each feature vector can be viewed as a diffeomorphism centered at that point. See Figure~\ref{fig:changeOfGauge}.
\end{remark}

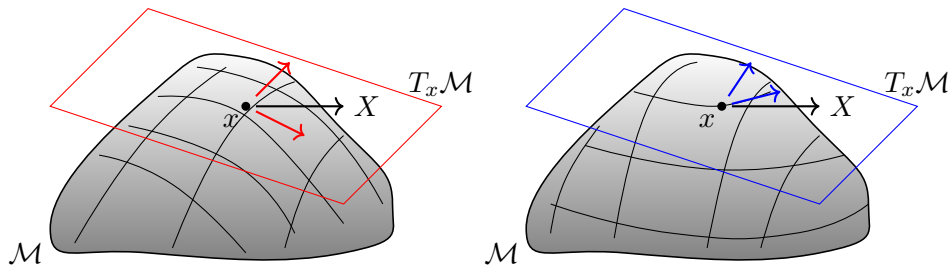
\begin{figure}[t]
    \centering
    \begin{subfigure}{0.4\linewidth}
        \begin{tikzpicture}[scale=0.65]

            \draw[top color=white,bottom color=gray,semithick]  plot[smooth cycle, tension=0.9] coordinates
            {(-4,-2.5) (-2,0.2) (0,1) (2,-0.5) (3,-2) (2,-3) (-2,-3)};  

            \node (or) at (0,0) {};

            \draw plot[smooth, tension=0.7,gray] coordinates {(-1.8,0.2) (0,-0.2) (2,-2.4)};
            \draw plot[smooth, tension=0.9,gray] coordinates {(-3,-1) (-1.5,-1.6) (0,-3)};
            \draw plot[smooth, tension=0.9,gray] coordinates {(-3.5,-2.8) (-2,-0.6) (-0.4,0.8)};

            \draw plot[smooth, tension=0.9,gray] coordinates {(-2.4,-0.4) (-0.2,-1.2) (1,-2.6)};
            \draw plot[smooth, tension=0.9,gray] coordinates {(-1,0.8) (1.2,0) (2.8,-2)};
            \draw plot[smooth, tension=0.9,gray] coordinates {(-1.5,-2.9) (-0.3,-0.5) (1,0.5)};
            \draw plot[smooth, tension=0.9,gray] coordinates {(0.8,-2.9) (1.3,-1.5) (2,-0.6)};

            \draw[fill=black]  (or) circle (0.08);
            \node[anchor=west] (vec) at ($(or)+(2,0)$) {$X$};
            \draw[->,thick] (or) -- (vec);
            \node at ($ (or)-(0.3,0.3) $) {$x$};
            
            \draw[->, thick, red] (or) -- ($(or)+(1.5*0.8,-1.5*0.4)$);
            \draw[->, thick, red] (or) -- ($(or)+(1.5*0.6,1.5*0.6)$);
            
            \node at (-4.5,-3) {$\mathcal{M}$};

            \coordinate (v1) at ($(or)+(4,0)$);
            \coordinate (v2) at ($(or)+(2,-2)$);

            \draw[red] ($(or)+(v2)$) -- ($(or)+(v1)$) -- ($(or)-(v2)$) -- ($(or)-(v1)$) -- cycle;
            \node[anchor=south] at ($(or)+(v1)$) {$ T_x\mathcal{M} $};

        \end{tikzpicture}
    \end{subfigure}
    ~
    \begin{subfigure}{0.4\linewidth}
        \begin{tikzpicture}[scale=0.65]
            \draw[top color=white,bottom color=gray,semithick]  plot[smooth cycle, tension=0.9] coordinates
            {(-4,-2.5) (-2,0.2) (0,1) (2,-0.5) (3,-2) (2,-3) (-2,-3)};

            \node (or) at (0,0) {};

            \draw plot[smooth, tension=0.7,gray] coordinates {(-1.8,0.3) (-0.2,0) (1,0.3)};
            \draw plot[smooth, tension=0.9,gray] coordinates {(-2.8,-0.8) (0,-1.3) (2.5,-1)};
            \draw plot[smooth, tension=0.9,gray] coordinates {(-3.5,-2) (-1,-2.6) (1.5,-2.6) (3,-2)};

            \draw plot[smooth, tension=0.9,gray] coordinates {(-3.5,-2.6) (-3.2,-1.5) (-2.5,-0.4)};
            \draw plot[smooth, tension=0.9,gray] coordinates {(-2.2,-2.7) (-1.5,0) (-0.5,0.9)};
            \draw plot[smooth, tension=0.9,gray] coordinates {(-0.3,-2.9) (0.3,-0.5) (1,0.5)};
            \draw plot[smooth, tension=0.9,gray] coordinates {(0.8,-2.9) (1.3,-1.5) (2,-0.6)};

            \draw[fill=black]  (or) circle (0.08);
            \node[anchor=west] (vec) at ($(or)+(2,0)$) {$X$};
            \draw[->,thick] (or) -- (vec);
            \node at ($ (or)-(0.3,0.3) $) {$x$};
            
            \draw[->, thick, blue] (or) -- ($(or)+(1.5*0.8,1.5*0.2)$);
            \draw[->, thick, blue] (or) -- ($(or)+(1.5*0.4,1.5*0.6)$);

            \node at (-4.5,-3) {$\mathcal{M}$};

            \coordinate (v1) at ($(or)+(4,0)$);
            \coordinate (v2) at ($(or)+(2,-2)$);

            \draw[blue] ($(or)+(v2)$) -- ($(or)+(v1)$) -- ($(or)-(v2)$) -- ($(or)-(v1)$) -- cycle;

            \node[anchor=south] at ($(or)+(v1)$) {$ T_x\mathcal{M} $};
            
        \end{tikzpicture}
    \end{subfigure}
    \caption[Change of gauge]{The manifold $ \mathcal{M} $ has a different choice of gauge --- local coordinates on $ \mathcal{M} $ --- at $ x $ in the left and the right figure leading to a different choice of basis in the corresponding tangent space $ T_x\mathcal{M} $. This is the same process as in Figure~\ref{fig:local_gauge_transformation_of_vector_field} where the difference is that in this case the base space is curved.}\label{fig:changeOfGauge}
\end{figure}

\subsection{General formulation}\label{sec:general_gauge}

The construction presented here is based on the one used in~\cite{aronsson2021}. Since in the general case all transformations will be local, we need to formulate the theory in the language of fiber bundles. In short, a \emph{bundle} $ E\xrightarrow{\pi}\mathcal{M} $ is a triple $ (E,\pi,\mathcal{M}) $ consisting of a \emph{total space} $ E $ and a surjective continuous projection map $ \pi $ onto a \emph{base manifold} $ \mathcal{M} $. Given a point $ x\in\mathcal{M} $, the \emph{fiber} over $ x $ is $ \pi^{-1}(x)=\{v\in E:\pi(v)=x\} $ and is denoted $ E_x $. If for every $ x\in\mathcal{M} $ the fibers $ E_x $ are isomorphic the bundle is called a \emph{fiber bundle}. Furthermore, a \emph{section} of a bundle $ E $ is a map $ \sigma:\mathcal{M}\to E $ such that $ \pi\circ\sigma=\id_\mathcal{M} $. We will use associated, vector, and principal bundles and will assume a passing familiarity with these, but will for completeness give a short overview. For more details see \cite{nakaharaGeometryTopologyPhysics2007, kolarNaturalOperationsDifferential1993, marshGaugeTheoriesFiber2019a}.

In this section we follow the construction of~\cite{aronsson2021} and begin by defining a principal bundle encoding some symmetry group $K$ we want to incorporate into our network:
\begin{definition}
    Let $ K $ be a Lie-group. A \emph{principal $K$-bundle} over $ \mathcal{M} $ is a fiber bundle $ {P}\xrightarrow{\pi_{P}}\mathcal{M} $ with a fiber preserving, free and transitive right action of $ K $ on $ {P} $,
    \begin{equation}
        \triangleleft: {P}\times K\to {P}\,,\quad \text{satisfying} \quad \pi_{P}(p\triangleleft k)=\pi_{P}(p)\,,
    \end{equation}
    and such that $ p\triangleleft e=p $ for all $ p\in{P} $ where $ e $ is the identity in $ K $. As a consequence of the action being free and transitive we have that $ {P}_x $ and $ K $ are isomorphic as sets.
\end{definition}
\begin{remark}
    Note that even though $ {P}_x $ and $ K $ are isomorphic as sets we cannot view $ {P}_x $ as a group since there is no identity in $ {P}_x $.
    To view the fiber $ {P}_x $ as isomorphic to the group we need to specify  $p\in {P}_x $ as a reference-point. With this we can make the following map from $K$ to ${P}_x$
    \begin{equation}
        k\mapsto p\triangleleft k\,,
    \end{equation}
    which is an isomorphism for each fixed $p\in {P}_x$ since the group action on the fibers is free and transitive. We will refer to this choice of reference point as \emph{choosing a gauge for the fiber} ${P}_x$.
\end{remark}

Given a local patch $U\subseteq \mathcal{M}$, a local section $\omega:U\to {P} $ of the principal bundle $ {P} $ provides an element $ \omega(x) $ which can be used as reference point in the fiber $ {P}_x $, yielding a local trivialization
\begin{equation}\label{eq:P_local_trivialization}
    U\times K\to P\,,\quad (x,k)\mapsto\omega(x)\triangleleft k\,.
\end{equation}
This is called choosing a \emph{(local) gauge} on the patch $ U \subseteq \mathcal{M} $.
\begin{remark}
    For the case when $\mathcal{M}$ is a surface, i.e. is of (real) dimension 2, \cite{melziGFramesGradientBasedLocal2019} presents a method for assigning a gauge (basis for the tangent space) to each point. This is done by letting one basis vector be the gradient of some suitable scalar function on $\mathcal{M}$ and the other being the cross product of the normal at that point with the gradient. This requires $\mathcal{M}$ to be embedded in $\mathbb{R}^3$.
\end{remark}
If it is possible to choose a continuous global gauge ($ U=\mathcal{M} $) then the principal bundle is trivial since \eqref{eq:P_local_trivialization} is then a global trivialization. On the other hand, if $ {P} $ is non-trivial and we allow $ \omega:U\to {P} $ to be discontinuous we can always choose a reference-point independently for each fiber. Alternatively, we may choose a set of continuous sections $ \{\omega_i\} $ whose domains $ U_i $ are open subsets covering $ \mathcal{M} $. In the latter setting, if $ x\in U_i\cap U_j $ then there is a unique element $ k_x\in K $ relating the gauges $\omega_i$ and $\omega_j$ at $ x $, such that $ \omega_i(x)=\omega_j(x)\triangleleft k_x $.
\begin{remark}
    As stated, a principal bundle is trivial if and only if there exists a global continuous gauge. Consequently, even if one has a covering $\{U_i\}$ of the manifold $\mathcal{M}$ and a local gauge $\omega_i$ for each $U_i$ there is no way to combine these into a global continuous gauge unless the principal bundle is trivial. It is, however, possible to define a global section if one drops the continuity condition. In this paper will not implicitly assume that a chosen gauge is continuous unless specified.
\end{remark}
Continuing in the same fashion we can view a local map $ \sigma: U\subseteq\mathcal{M}\to K $ as changing the gauge at each point, or in other words a gauge transformation. 
With this we now define associated bundles: 

\begin{definition}
Let $ {P}\xrightarrow{\pi_{P}}\mathcal{M} $ be a principal $ K $-bundle and let $ V $ be a vector space on which $ K $ acts from the left through some representation $ \rho $
\begin{equation}
    K\times V\to V\,,\quad k\triangleright v = \rho(k)v\,.
\end{equation}
Now, consider the space $ {P}\times V $ consisting of pairs $ (p,v) $ and define an equivalence relation on this space by
\begin{equation}\label{eq:gauge_equivalence_relation}
    (p,v)\sim_{\rho}(p\triangleleft k,k^{-1} \triangleright v) \,, \quad k\in K\,.
\end{equation}
Denoting the resulting quotient space $P \times V/\sim_{\rho}$ by $ {E}_{\rho} = {P}\times_\rho V_\rho $ and equipping it with a projection $ \pi_{\rho}:{E}_{\rho}\to \mathcal{M} $ acting as $ \pi_\rho([p,v])=\pi_{{P}}(p) $ makes $ {E}_{\rho} $ a fiber bundle over $ \mathcal{M} $ associated to $P$. Moreover, $ \pi_\rho^{-1}(x) $ is isomorphic to $ V $ and thus $ {P}\times_\rho V_\rho $ is a vector bundle over $ \mathcal{M} $. The space of sections $\Gamma(E_{\rho})$ is a vector space with respect to the point-wise addition and scalar multiplication of $V$.    
\end{definition}
Moving on we provide the definition of data points and feature maps in the setting of networks.
\begin{definition}
    We define a \emph{data point} $ s\in\Gamma(E_{\rho}) $ as a section of an associated bundle $ E_\rho $ and a \emph{feature map} $ f\in C(P;{\rho}) $ as a map from the principal bundle to the vector space $ V_\rho $. A feature map should also satisfy the following equivariance condition
    \begin{equation}\label{eq:feature_map_equivariance}
        k\triangleright f(p)=f(p\triangleleft k^{-1})\,.
    \end{equation}
\end{definition}
To connect the two notions we can use the following lemma.
\begin{lemma}[\cite{kolarNaturalOperationsDifferential1993}]
    The linear map $ \varphi_{\rho}:C(P;{\rho}) \to \Gamma(E_{\rho}) \,,\, f \mapsto s_{f}=[\pi^{-1}_{P},f \circ \pi^{-1}_{P}] $, is a vector space isomorphism.
\end{lemma}
\begin{remark}
    With the notation $[\pi^{-1}_{P},f \circ \pi^{-1}_{P}]$ we mean that when applying this to $x\in\mathcal{M}$ one first picks an element $p\in\pi_{{P}}^{-1}$ so that $s_f(x)=[p,f(p)]$. Note that this is well defined since $s_f(x)=[p,f(p)]$ where is independent of the choice of $p$ in the fiber over $x$:
    \begin{equation}
        [p',f(p')]=[p\triangleleft k, f(p\triangleleft k)]=[p\triangleleft k, k^{-1}\triangleright f(p)]=[p,f(p)]\,.
    \end{equation}
\end{remark}
\begin{remark}
    To present a quick argument for this note that with the equivalence class structure on the associated bundle every section $ s_{f}\in\Gamma(E_{\rho}) $ is of the form $ s_{f}(x)=[p,f(p)] $ where $ p\in\pi^{-1}_{P}(x) $ since every element in the associated bundle consists of an element from the principal bundle and a vector from a vector space. Here we use $f$ to specify which vector is used.
\end{remark}

Before moving on to gauge equivariant layers we need to establish how a gauge transformation $ \sigma:\mathcal{M}\to K $ acts on the data points.
\begin{definition}
    Let $ k\in K $ be an element of the gauge group and $[p,v]$ an element of the associated bundle $E_\rho = P \times_\rho V_\rho$. Then the action of $k$ on $[p,v]$ is defined as
    \begin{equation}
        k\cdot [p,v]=[p\triangleleft k, v] = [p,k\triangleright v]\,,
    \end{equation}
    which induces an action of $ \sigma:U\subseteq\mathcal{M}\to K $ on $ s_f=[\cdot,f] $ as
    \begin{equation}\label{eq:gauge_on_data_point}
        (\sigma \cdot s_f)(x)=\sigma(x)\cdot[p,f(p)]=[p,\sigma(x)\triangleright f(p)]=s_{\sigma\triangleright f}(x)\,.
    \end{equation}
    We call $\sigma$ a \emph{gauge transformation}. If $\sigma$ is constant on $U$ it is called a \emph{rigid gauge transformation}.
\end{definition}
\begin{remark}
    Note that~\eqref{eq:gauge_on_data_point} is really abuse of notation. To view $\sigma$ as a map from region of the manifold $U\subseteq \mathcal{M}$ to the group $K$ we first have to choose a gauge $\omega:U\to P$. This is identical to choosing a local trivialization of ${P}$ and allows us to identify the fibers of ${P}$ with $K$. For details on this see~\cite{hamiltonMathematicalGaugeTheory2017}.
\end{remark}
With this we now define layers as map between spaces of sections of associated vector bundles.
\begin{definition}
    Let $ E_{\rho}={P}\times_{\rho}V_{\rho} $ and $ E_{\eta}={P}\times_{\eta}V_{\eta} $ be two associated bundles over a smooth manifold $\mathcal{M}$. Then a \emph{layer} is a map $ \Phi:\Gamma(E_{\rho})\to\Gamma(E_{\eta}) $. Moreover, $ \Phi $ is \emph{gauge equivariant} if $ \sigma\circ\Phi=\Phi\circ\sigma $ in the sense~\eqref{eq:gauge_on_data_point} for all gauge transformations $ \sigma $.
\end{definition}
\begin{remark}
    Since the map $ \varphi_{\pi}:C(P;{\pi}) \to \Gamma(E_{\pi}) $ is a vector space isomorphism for any representation $\pi$, any layer $ {\Phi:\Gamma(E_{\rho})\to\Gamma(E_{\eta})} $ induces a unique map $ \phi:C(P;{\rho})\to C(P;{\eta}) $ by $ \phi=\varphi_{\eta}^{-1}\circ\Phi\circ\varphi_{\rho} $, and vice versa; see the diagram below. We will refer to both $\Phi$ and $\phi$ as layers.
    \[
\begin{tikzcd}
\Gamma(E_\rho) \arrow{rr}{\Phi} && \Gamma(E_\eta) \arrow{dd}{\varphi_\eta^{-1}}\\
&\\
C(P;\rho) \arrow{uu}{\varphi_\rho} \arrow[dashed]{rr}{\phi} && C(P;\eta)
\end{tikzcd}
\]
\end{remark}
The equivariance property expressed explicitly in terms of $\phi$, given a feature map $f$, is
\begin{equation}\label{eq:phi_equivariance}
    k\triangleright\big(\phi f\big)(p)=\phi(k\triangleright f)(p)\,,
\end{equation}
and since $\phi f\in C(P;\eta)$ it needs to satisfy~\eqref{eq:feature_map_equivariance} which results in
\begin{equation}
    (\phi f)(p\triangleleft k^{-1})=\phi(k\triangleright f)(p)\,.
\end{equation}
To solve this constraint one needs to relate of $\phi f$ acts on points in the principal bundle to how $\phi f$ depends on $f$ which is not trivial.

\subsection{Geometric symmetries on manifolds and the frame bundle}
A special case of the general structure presented above is when the gauge symmetry is a freedom to choose a different coordinate patch $U\subset\mathcal{M}$ around each point in the manifold $ \mathcal{M} $. This is also needed as soon as $ \mathcal{M} $ is curved since it is impossible to impose a global consistent coordinate system on a general curved manifold. Because of this one is forced to work locally and use the fiber bundle approach.

In this section we will work within a coordinate patch $ u:U\to\mathbb{R}^{d} $ around $ x\in U $, and the $ i $:th coordinate of $ x $ is denoted $ y^{i}=u^{i}(x) $ being the $ i $:th component of the vector $u \in \mathbb{R}^d$. A coordinate patch is sometimes denoted $(u,U)$. This coordinate chart induces a basis for the tangent space $ T_x\mathcal{M} $ as $ \{\partial_{1},\dots,\partial_{d}\} $ such that any vector $ v\in T_x\mathcal{M} $ can be written
\begin{equation}
    v=\sum_{m=1}^{d}v^{m}\partial_{m}=v^{m}\partial_{m}\,,
\end{equation}
where $ v^{m} $ are called the components of $ v $ and we are using the Einstein summation convention that indices repeated upstairs and downstairs are implicitly summed over. We will use this convention for the rest of this section.

Given a coordinate chart $ u:U\to\mathbb{R}^{d} $ around $ x\in U $ in which we denote coordinates $ y^{i} $ a change to a different coordinate chart $ u':U\to \mathbb{R}^{d} $ where we denote coordinates $ y^{\prime i} $ would be done through a map $ u'\circ u^{-1}:\mathbb{R}^{d}\to\mathbb{R}^{d} $. This map can be expressed as a matrix $ k^{-1}\in \GL(d) $ (we use $ k^{-1} $ for notational reasons) and in coordinates this transformation would be
\begin{equation}
    y^{\prime n}=\tensor{(k^{-1})}{^{n}_{m}}y^{m}\,,
\end{equation}
where $ \tensor{(k^{-1})}{^{n}_{m}} $ should be interpreted as the element in the $ n $:th row and $ m $:th column. For a tangent vector $ v\in T_x\mathcal{M} $ expressed in the first coordinate chart $ v=v^{m}\partial_{m} $ the components transform as the coordinates
\begin{equation}
    v'^{m}=\tensor{k}{^{m}_n}v^{n}\,,
\end{equation}
whereas the basis transforms
\begin{equation}
    \partial'_{m}=\tensor{k}{^{n}_{m}}\partial_{n}\,.
\end{equation}
As a consequence the vector $ v $ is independent of the choice of coordinate chart:
\begin{equation}
    v=v^{m}\partial_{m}=\tensor{(k^{-1})}{^{n}_{m}}v^{\prime m}\tensor{k}{^{j}_{n}}\partial'_{j}=v^{\prime m}\tensor{k}{^{j}_{n}}\tensor{(k^{-1})}{^{n}_{m}}\partial'_{j}=v^{\prime m}\tensor{\delta}{^{j}_{m}}\partial'_{j}=v^{\prime m}\partial'_{m}=v'\,,
\end{equation}
where $ \tensor{\delta}{^n_m}=1 $ if $ n=m $ else $ 0 $. The change of coordinates is hence a gauge symmetry as presented in Section~\ref{sec:global_local_symmetries} and because of this we have the freedom of choosing a local basis at each point in $ \mathcal{M} $ without affecting the tangent vectors, although the components might change, so we need to track both the components and the gauge the vector is expressed in.

To express this in the formalism introduced in Section~\ref{sec:general_gauge}, let again $\mathcal{M}$ be a $d$-dimensional smooth manifold and consider the frame bundle $P = \mathcal{LM}$ over $ \mathcal{M} $. The frame bundle is defined by its fiber $ \mathcal{L}_x\mathcal{M} $ over a point $ x\in \mathcal{M} $ as
\begin{equation}
    \mathcal{L}_x\mathcal{M}=\{(e_1,\dots, e_d): (e_1,\dots, e_d)\text{ is a basis for }T_x\mathcal{M}\}\simeq \GL(d,\mathbb{R})\,.
\end{equation}
This is a principal $ K $-bundle $\mathcal{LM} \xrightarrow{\pi} \mathcal{M}$ over $ \mathcal{M} $ with $K = \GL(d,\mathbb{R})$; the projection is defined by mapping a given frame to the point in $\mathcal{M}$ where it is attached. In the fundamental representation, the group acts from the left on the frame bundle as
\begin{equation}
    \triangleright : \GL(d,\mathbb{R}) \times \mathcal{LM} \to \mathcal{LM}\,, \qquad (e_1,\ldots,e_d)\triangleleft  k = (\tensor{k}{^{m}_{1}}e_m, \ldots, \tensor{k}{^{m}_{d}} e_m)\,.
\end{equation}
Hence the action performs a change of basis in the tangent space as described above. If we also choose a (real) representation $(\rho,V)$ of the general linear group, we obtain an associated bundle $\mathcal{LM} \times_\rho V $ with typical fiber $ V $ as constructed in Section~\ref{sec:general_gauge}. Its elements $[e,v]$ are equivalence classes of frames $e = (e_1,\ldots,e_d) \in \mathcal{LM}$ and vectors $v \in V$, subject to the equivalence relation introduced in~\eqref{eq:gauge_equivalence_relation}:
\begin{equation}
    [e,v]= [ e \triangleleft  k ,  k^{-1}\triangleright v]=[e \triangleleft  k, \rho(k^{-1})v]\,.
\end{equation}
The bundle associated to $\mathcal{LM}$ with fiber $ \mathbb{R}^{d} $ over a smooth $ d $-dimensional manifold $ \mathcal{M} $ is closely related to the tangent bundle $ T\mathcal{M} $ by the following classical result, see~\cite{nakaharaGeometryTopologyPhysics2007, kolarNaturalOperationsDifferential1993} for more details:
\begin{theorem}
    Let $\rho$ be the standard representation of $\GL(d,\mathbb{R})$ on $\mathbb{R}^d$. Then the associated bundle ${\mathcal{LM}} \times_\rho \mathbb{R}^d$ and the tangent bundle $T\mathcal{M}$ are isomorphic as bundles.
\end{theorem}
\begin{remark}
    One can use the same approach to create an associated bundle which is isomorphic to the $(p,q)$-tensor bundle $T_q^p\mathcal{M}$, for any $p,q$, where the tangent bundle is the special case $ (p,q)=(1,0) $.
\end{remark}

\begin{remark} It is not necessary to use the full general linear group as structure group. One can instead construct a frame bundle and associated bundles using a subgroup, such as $\SO(d)$. This represents the restriction to orthogonal frames on a Riemannian manifold.
\end{remark}

\subsection{Explicit gauge equivariant convolution}\label{subsec:gauge_equivariant_convoultion}
Having developed the framework for gauge equivariant neural network layers, we will now describe how this can be used to construct concrete gauge equivariant convolutions. These were first developed in in~\cite{CohenChengGauge2019} and we lift this construction to the principal bundle ${P}$. 

To begin, let ${P}$ be a principal $K$-bundle, $E_\rho$ and $E_\eta$ be two associated bundles as constructed in Section~\ref{sec:general_gauge}. Since the general layer $\Phi:\Gamma(E_{\rho})\to\Gamma(E_{\eta})$ induces a unique map $\phi:C(P;{\rho})\to C(P;{\eta})$ we can focus on presenting a construction of the map $\phi$. When constructing such a map with $C(P;{\eta})$ as its co-domain it needs to respect the “equivariance” of the feature maps~\eqref{eq:feature_map_equivariance}. Hence we get the following condition on $\phi$:
\begin{equation}\label{eq:phi_feature_map_condition}
    k\triangleright(\phi f)(p)=(\phi f)(p\triangleleft k^{-1})\,.
\end{equation}
For gauge equivariance, i.e.\ a gauge transformation $\sigma:U\subseteq \mathcal{M}\to K$ “commuting” with $\phi$, we need $\phi$ to satisfy~\eqref{eq:phi_equivariance}:
\begin{equation}\label{eq:phi_gauge_equivariance_condition}
    \sigma\triangleright\phi f=\phi(\sigma\triangleright f)\,,
\end{equation}
where $(\sigma\triangleright f)(p)=\sigma(\pi_{{P}}(p))\triangleright f(p)$.

Equation \eqref{eq:phi_feature_map_condition} is a condition on how the feature map $\phi f$ needs to transform when moving in the fiber. To move across the manifold we need the concept of geodesics and for that a connection on the tangent bundle $T\mathcal{M}$ is necessary. The intuition for a connection is that it designates how a tangent vector changes when moved along the manifold. The most common choice, if we have a (pseudo) Riemannian metric $g_{\mathcal{M}}:T\mathcal{M}\times T\mathcal{M}\to \mathbb{R}$ on $\mathcal{M}$, is to induce a connection from the metric, called the \emph{Levi-Civita connection}. This construction is the one we will assume here. For details see~\cite{albinLinearAnalysisManifolds}.
\begin{remark}
    The metric acts on two tangent vectors from the same tangent space. Hence writing $g_\mathcal{M}(X,X)$ for $X\in T_x\mathcal{M}$ is unique. If we want to reference the metric at some specific point we will denote this in the subscript.
\end{remark}

To make clear how we move on the manifold we introduce the \emph{exponential map} from differential geometry, which lets us move around by following geodesics. The exponential map,
\begin{equation}
\exp:\mathcal{M}\times T\mathcal{M}\to \mathcal{M}\,,
\end{equation}
maps $(x,X)$, where $x\in\mathcal{M}$ is a point and $X\in T_x\mathcal{M}$ is a tangent vector, to the point $\gamma(1)$ where $\gamma$ is a geodesic such that $\gamma(0)=x$ and $\frac{\mathrm{d}}{\dd t}\gamma(0)=X$. The exponential map is well-defined since every tangent vector $X\in T_x\mathcal{M}$ corresponds to a unique such geodesic. 
\begin{remark}
    It is common to write the exponential map $\exp(x,X)$ as $\exp_xX$, which is then interpreted as a mapping $\exp_x : T_x\mathcal{M} \to \mathcal{M}$ for each fixed $x \in \mathcal{M}$.
\end{remark}

Note that the exponential map defines a diffeomorphism between the subset $B_{R}=\{X\in T_x\mathcal{M}:\sqrt{g_{\mathcal{M}}(X,X)}<R\}\subset T_x\mathcal{M} $ and the open set $ \{y\in\mathcal{M}:d_{g_{\mathcal{M}}}(x,y)<R\} $ where $ d_{g_{\mathcal{M}}} $ is the geodesic induced distance function on $ \mathcal{M} $. This will later be used as our integration domain.

In order to lift the convolution presented in~\cite{CohenChengGauge2019} we need a notion of parallel transport in the principal bundle ${P}$. This will be necessary for the layer $\phi$ to actually output feature maps. In order to do this, we introduce vertical and horizontal directions with respect to the base manifold $\mathcal{M}$ in the tangent spaces $T_pP$ for $p \in P$.
\begin{definition}
    A \emph{connection} on a principal bundle $P$ is a smooth decomposition $T_pP = V_pP \oplus H_pP$, into a vertical and a horizontal subspace, which is equivariant with respect to the right action $\triangleleft$ of $K$ on $P$.  
\end{definition}

In particular, the connection allows us to uniquely decompose each vector $X \in T_pP$ as $X = X^V + X^H$ where $X^V\in V_pP$ and $X^H\in H_pP$, and provides a notion of transporting a point $p \in P$ parallel to a curve in $\mathcal{M}$.

\begin{definition}
    Let $\gamma : [0,1] \to \mathcal{M}$ be a curve and let $P$ be a principal bundle equipped with a connection. The \emph{horizontal lift} of $\gamma$ through $p \in \pi_P^{-1}(\gamma(0))$ is the unique curve ${\gamma^{\uparrow}_p : [0,1] \to P}$ such that $\forall t \in [0,1]$
    \begin{enumerate}
        \item[\emph{i)}] $\pi_P\left(\gamma^{\uparrow}_p(t)\right) = \gamma(t)$
        \item[\emph{ii)}] $\frac{\mathrm{d}}{\dd t}\gamma^{\uparrow}_p(t) \in H_{\gamma^{\uparrow}_p(t)}P$
    \end{enumerate}
\end{definition}

\begin{remark}
    The property $\frac{\mathrm{d}}{\dd t}\gamma^{\uparrow}_p(t) \in H_{\gamma^{\uparrow}_p(t)}P$ implies that the tangent to the lift $\gamma^{\uparrow}_p$ has no vertical component at any point along the curve.
\end{remark}

Given a curve $ \gamma $ on $ \mathcal{M} $ connecting $ \gamma(0)=x $ and $ \gamma(1)=y $ we can now define the \emph{parallel transport map} $ \mathcal{T}_{\gamma} $ on $ P $  as
\begin{equation}
    \mathcal{T}_{\gamma}:{P}_x\to{P}_y\,,\quad p\mapsto\gamma^{\uparrow}_{p}(1)\,.
\end{equation}
Moreover, the map $ \mathcal{T}_{\gamma} $ is equivariant with respect to $ K $~{\cite[Theorem 11.6]{kolarNaturalOperationsDifferential1993}}; that is
\begin{equation}
    \mathcal{T}_{\gamma}(p\triangleleft k)=\mathcal{T}_{\gamma}(p)\triangleleft k\,, \quad \forall k\in K\,.
\end{equation}
Since we will only be working with geodesics and there is a bijection between geodesics through $ x\in \mathcal{M} $ and $ T_x\mathcal{M} $ we will instead denote the parallel transport map as $ \mathcal{T}_X $ where $ X=\frac{d}{dt}\gamma(0) $.

The final parts we need to discuss before defining the equivariant convolutions are the properties of the integration region and measure.
\begin{lemma}\label{lem:change_of_variables}
    Let $(u,U)$ be a coordinate patch around $x\in\mathcal{M}$. The set $B_{R}=\{X\in T_x\mathcal{M}:\sqrt{g_{\mathcal{M}}(X,X)}<R\}\subset T_x\mathcal{M} $ and the integration measure $ \sqrt{\det(g_{\mathcal{M},x})}\dd X $ are invariant under change of coordinates that preserve orientation.
\end{lemma}
\begin{remark}
    The above lemma abuses notation for $X$ since $X$ in $g_\mathcal{M}(X,X)$ is indeed a tangent vector in $T_x\mathcal{M}$ while $\dd X=\dd y^1\wedge \dd y^2\wedge\cdots\wedge \dd y^d$ where $y^i$ are the coordinates obtained from the coordinate chart. See the previous section for more details on coordinate charts.
\end{remark}
\begin{proof}
    If $ X=X^{n}e_n $ and $ Y=Y^{n}e_n $ are two tangent vectors in $ T_x\mathcal{M} $ then the metric evaluated on these is
    \begin{equation}\label{eq:metric_on_tangent_vectors}
        g_{\mathcal{M}}(X,Y)=X^{n}Y^{m}g_{\mathcal{M}}(e_n,e_m)\,.
    \end{equation}
    Expressing $X$ in other coordinates $ X'=\tensor{(k^{-1})}{^{n}_{m}}X^{m}\tensor{k}{^{i}_{n}}e_{i} $ and similar for $ Y $ we get that~\eqref{eq:metric_on_tangent_vectors} transforms as
    \begin{equation}
        g_{\mathcal{M}}(X',Y')= \tensor{(k^{-1})}{^{n}_{j}}X^{j}\tensor{(k^{-1})}{^{m}_{i}}Y^{i}g_{\mathcal{M}}(\tensor{k}{^{\ell}_{n}}e_\ell,\tensor{k}{^{o}_{m}}e_o)\,.
    \end{equation}
    Since $ g_{\mathcal{M},x} $ is bilinear at each $ x\in \mathcal{M} $
    \begin{equation}
        g_{\mathcal{M}}(X',Y')=\tensor{\delta}{^{l}_{j}}X^{j}\tensor{\delta}{^{m}_{i}}Y^{i}g_{\mathcal{M}}(e_l,e_m)=g_{\mathcal{M}}(X,Y)\,,
    \end{equation}
    and we are done. The invariance of $ \sqrt{\det(g_{\mathcal{M},x})}\dd X $ comes from noting $ \dd X'=\det(k^{-1})\dd X $ and
    \begin{equation}
        \det(g'_{\mathcal{M},x})=\det(g_{\mathcal{M}}(e'_m, e'_n))=\det(\tensor{k}{^{i}_{m}}\tensor{k}{^{j}_{n}}g_{\mathcal{M}}(e_i, e_j))=\det(k)^{2}\det(g_{\mathcal{M},x})\,.
    \end{equation}
    Hence,
    \begin{equation}
        \sqrt{\det(g'_{\mathcal{M},x})}\dd X'= \sqrt{\det(k)^{2}\det(g_{\mathcal{M},x})}\det(k^{-1})\dd X=\sqrt{\det(g_{\mathcal{M},x})}\dd X\,,
    \end{equation}
    since we have restricted $\GL(d,\mathbb{R})$ to those elements with positive determinant, i.e.\ preserve orientation.
\end{proof}
\begin{remark}
    The integration measure $\sqrt{\det(g_\mathcal{M})}\dd X$ is expressed in local coordinates but is an intrinsic property of the (pseudo) Riemannian manifold. It is often written as $\vol_{T_x\mathcal{M}}$ to be explicitly coordinate independent.
\end{remark}
We can now state the convolution defined in~\cite{CohenChengGauge2019} as follows. Choose a local coordinate chart around $x\in \mathcal{M}$ and let $s\in\Gamma(E_\rho)$ be a section of an associated bundle (a data point in our terminology in Section~\ref{sec:general_gauge}). The gauge equivariant convolution is
    \begin{equation}\label{eq:cheng_conv}
        (\Phi s)(x) = \int_{B_R} \kappa(x,X)s|_{\exp_xX_G}(x)\sqrt{\det(g_{\mathcal{M}})}\dd X\,,
    \end{equation}
    given a convolution kernel function
    \begin{equation}
        \kappa:\mathcal{M}\times T\mathcal{M}\to \Hom(E_\rho, E_\eta)\,,
    \end{equation}
    and $X_G$ is the geometric tangent vector that corresponds to the coordinate representation $X$. Here, $\det(g_{\mathcal{M}})$ is the determinant of the Riemannian metric at $x$, and $s|_{\exp_xX}(x)$ represents the parallel transport, with respect to a connection on the associated bundle, of the element $s(\exp_xX)$ back to $x$ along a geodesic. The convolution \eqref{eq:cheng_conv} is gauge equivariant if $\kappa$ satisfies the constraint \eqref{eq:gauge_kernel_condition}.
    
Using the principle bundle construction from above, we now present the lifted version of~\eqref{eq:cheng_conv}.
\begin{definition}[Gauge equivariant convolution]
    Let $U\subseteq \mathcal{M}$ be such that $x = \pi_{{P}}(p)\in B_R\subseteq U$ and choose a gauge $\omega:U\to P_U=\pi_{{P}}^{-1}(U)$. Let $ f\in C(P;\rho) $ be a feature map, then the gauge equivariant convolution is defined as
    \begin{equation}\label{eq:phi_gauge_case}
    (\phi f)(p)=[\kappa\star f](p)=\int_{B_R}\kappa(x,X)f(\mathcal{T}_{X}p)\vol_{T_x\mathcal{M}}\,,
    \end{equation}
    where
    \begin{equation}
        \kappa:\mathcal{M}\times T\mathcal{M}\to\Hom(V_{\rho}, V_{\eta})\,,
    \end{equation}
    is the convolution kernel and $\vol_{T_x\mathcal{M}}$ is the volume form for $T_x\mathcal{M}$.
\end{definition} 
As we will see in Theorem~\ref{thm:gaugeeq}, the convolution \eqref{eq:phi_gauge_case} is gauge equivariant if $\kappa$ satisfies the constraint \eqref{eq:gauge_kernel_condition}.
\begin{remark}
    In an ordinary CNN a convolutional kernel has compact support on the image plane, $\kappa:\mathbb{Z}^2\to\Hom(V_{\text{in}}, V_{\text{out}})$ and hence depends on its position but in this case the kernel depends on its position on $\mathcal{M}$ and a direction. 
\end{remark}

For $\phi$ to be gauge equivariant, the kernel $\kappa$ must have the following properties.
\begin{theorem}\label{thm:gaugeeq}
    Let $(u,U)$ be a coordinate chart such that $U\subseteq \mathcal{M}$ is a neighborhood to $x=\pi_{{P}}(p)$ and let $ \phi:C(P;\rho)\to C(P;\eta) $ be defined as in \eqref{eq:phi_gauge_case}. Then $\phi$ satisfies the feature map condition \eqref{eq:phi_feature_map_condition}, along with the gauge equivariance condition \eqref{eq:phi_gauge_equivariance_condition} for all rigid gauge transformations $\sigma:U\to K$, if
    \begin{equation}\label{eq:gauge_kernel_condition}
        \kappa(x,X')=\eta(k^{-1})\kappa(x,X)\rho(k)\,,
    \end{equation}
    where $X'=k\triangleright X$ is the transformation of tangent vectors under the gauge group.
\end{theorem}
\begin{proof}
    Choose a local coordinate chart $(u,U)$ and let $U\subseteq \mathcal{M}$ be a neighborhood to $x=\pi_{{P}}(p)$. Let further $ \phi:C(P;\rho)\to C(P;\eta) $ be defined in \eqref{eq:phi_gauge_case}. We can then write~\eqref{eq:phi_gauge_case} in the local chart as
    \begin{equation}
        (\phi f)(p)=\int_{B_R}\kappa(x,X')f(\mathcal{T}_{X'_G}p)\sqrt{\det(g_\mathcal{M})}\dd X'\,.
    \end{equation}
    Let $k\in K$ be such that  $X'=k\triangleright X$ and let $\sigma:U\to K,\ \sigma(y)=k $ for all $y\in U$ be a rigid gauge transformation. Note that $X'_G=X_G$ since the geometric vector is independent of choice of coordinate representation. Then the left hand side of \eqref{eq:phi_gauge_equivariance_condition} can be written as
    \begin{align}
        (\sigma\triangleright\phi f)(p)&=\sigma(x)\triangleright (\phi f)(p)\nonumber
        =k \triangleright(\phi f)(p)\\
        &=\eta(k)\int_{B_R}\kappa(x,X')f(\mathcal{T}_{X_G'}p)\sqrt{\det(g_\mathcal{M})}\dd X'.
      \end{align}
    Using \eqref{eq:gauge_kernel_condition} and a change of variables we get
    \begin{equation}
        \int_{B_R}\kappa(x,X)\rho(k)f(\mathcal{T}_{X_G}p)\sqrt{\det(g_\mathcal{M})}\dd X\,.
    \end{equation}
    From here we first prove the feature map condition and follow up with the proof of the gauge equivariance.
    
    Since $ f $ is a feature map we get $ \rho(k)f(\mathcal{T}_{X_G}p)=f\big((\mathcal{T}_{X_G}p)\triangleleft k^{-1}\big) $ and using the equivariance of $ T $ we arrive at
    \begin{equation}
        f\big((\mathcal{T}_{X_G}p)\triangleleft k^{-1}\big)=f\big(\mathcal{T}_{X_G}(p\triangleleft k^{-1})\big)\,.
    \end{equation}
    Thus,
    \begin{equation}
        k\triangleright(\phi f)(p)=(\phi f)(p\triangleleft k^{-1})\,,
    \end{equation}
    which gives the feature map condition.
    
    For the gauge equivariance condition note that 
    \begin{equation}
        \rho(k)f(\mathcal{T}_{X_G}p)=\sigma(x)\triangleright f(\mathcal{T}_{X_G}p)=\sigma\big(\pi_{{P}}(\mathcal{T}_{X_G}p)\big)\triangleright f(\mathcal{T}_{X_G}p)=(\sigma\triangleright f)(\mathcal{T}_{X_G}p)\,,
    \end{equation}
    using that $\sigma$ is a rigid gauge transformation. Hence we arrive at 
    \begin{equation}
        (\sigma\triangleright\phi f)(p)=\phi(\sigma\triangleright f)(p)\,,
    \end{equation}
    proving the gauge equivariance of \eqref{eq:phi_gauge_case} for kernels which satisfy \eqref{eq:gauge_kernel_condition}.
    
    Note that we use the equivariance of the parallel transport in ${P}$ to get the feature map condition and the rigidness of $\sigma$ to arrive at the gauge equivariance.
\end{proof}
\begin{remark}
    Along the same lines, one can also prove that the convolution~\eqref{eq:cheng_conv} is gauge equivariant if the kernel satisfies \eqref{eq:gauge_kernel_condition}. The main difference is that in the general case the point $p\in P$, at which the feature map is evaluated, holds information about the gauge used.  For~\eqref{eq:cheng_conv} the choice of gauge is more implicit. For more intuition on this, see the discussion below.
\end{remark}

To get an intuition of the difference between the gauge equivariant convolution presented in~\cite{CohenChengGauge2019} and the lifted convolution in~\eqref{eq:phi_gauge_case} we first note that both require an initial choice: for~\eqref{eq:cheng_conv} we need to choose a coordinate chart $u$ around the point $x\in U\subseteq\mathcal{M}$ and for the lifted version we need to choose a local gauge $\omega:U\to {P}$ around $x=\pi_{{P}}(p)$. When dealing with gauge transformations that are changes of local basis choosing a gauge and a local coordinate system is the same.

Continuing, we note that applying a gauge transformation to a feature map on the principal bundle is the same as moving the evaluation point of the feature map along a fiber in the principal bundle:
\begin{equation}
    (\sigma\triangleright f)(p)=\sigma\big(\pi_{{P}}(p)\big)\triangleright f(p)=f\Big(p\triangleleft\sigma\big(\pi_{{P}}(p)\big)\Big)\,.
\end{equation}
Since the action of the gauge group is free and transitive on the fibers of ${P}$ we can interpret this as the feature using its evaluation point as information about what gauge it is in. Hence this states that the gauge group action on $f$ amounts to evaluating $f$ in a different gauge over the same point in $\mathcal{M}$. Using this interpretation we get that to parallel transport a feature map $f(\mathcal{T}_{X}p)$ designates which gauge should be used at every point and hence how $f$ transforms when moved along $\gamma$. Choosing a connection on an associated bundle does the same thing: it prescribes how a vector changes when moved between fibers. In this sense the transport of a feature map on ${P}$ is the same as parallel transport a vector in an associated bundle. Since the integration measure used in~\eqref{eq:cheng_conv} is just a local coordinate representation of the volume form $\vol_{T_x\mathcal{M}}$ we have now related all components of~\eqref{eq:cheng_conv} to our lifted version.

\subsection{Examples: Gauge equivariant networks on graphs and meshes}

In this section we will illustrate how gauge-equivariant CNNs can be implemented in practice. This will be done in the context of graph networks where gauge equivariance corresponds to changes of coordinate system in the tangent space at each point in the graph. The main references for this section are \cite{worrall2017,CohenMesh,2020arXiv200601570W}. We begin by introducing some relevant background on harmonic networks.

\subsubsection{Harmonic networks on surfaces}

Let $\mathcal{M}$ be a smooth manifold of real dimension 2. For any point $x\in \mathcal{M}$ we have the tangent space $T_x\mathcal{M} \cong \mathbb{R}^2$. The main idea is to consider a convolution in $\mathbb{R}^2$ at the point $x$ and lift it to $\mathcal{M}$ using the Riemann exponential mapping. As we identify the tangent space $T_x\mathcal{M}$ with $\mathbb{R}^2$ we have a rotational ambiguity depending on the choice of coordinate system. This is morally the same as a ``local Lorentz frame'' in general relativity. This implies that the information content of a feature map in the neighborhood of a point $x$ can be arbitrarily rotated with respect to the local coordinate system of the tangent space at $x$. We would like to have a network that is independent of this choice of coordinate system. Moreover, we also have a path-dependence when transporting filters across the manifold $\mathcal{M}$. To this end we want to construct a network which is equivariant in the sense that the convolution at each point is equivariant with respect to an arbitrary choice of coordinate system in $\mathbb{R}^2$. Neural networks with these properties were constructed in \cite{worrall2017,2020arXiv200601570W} and called \emph{harmonic networks}.

We begin by introducing standard polar coordinates in $\mathbb{R}^2$:
\begin{equation}
(r, \theta)\in \mathbb{R}_+\times [0, 2\pi)\,.
\end{equation}
In order to ensure equivariance we will assume that the kernels in our network are given by the \emph{circular harmonics}
\begin{equation}
\kappa_m(r, \theta; \beta)= R(r)e^{i(m\theta +\beta)}\,.
\label{circularharmonics}
\end{equation}
Here the function $R: \mathbb{R}_+\to \mathbb{R}$ is the radial profile of the kernel and $\beta$ is a free parameter. The degree of rotation is encoded in $m\in \mathbb{Z}$. The circular harmonic transforms by an overall phase with respect to rotations in $\SO(2)$:
\begin{equation}
\kappa_m(r, \theta-\phi; \beta)=e^{im\phi}\kappa_m(r, \theta; \beta)\,.
\end{equation}

Let $f:\mathcal{M}\to \mathbb{C}$ be a feature map. The convolution of $f$ with $\kappa_m$ at a point $x\in \mathcal{M}$ is defined by
\begin{equation}
(\kappa_m\star f)(x)=\iint_{D_x(\epsilon)} \kappa_m(r, \theta;\beta)f(r, \theta)\, r\,\dd r\,\dd\theta\,.
\end{equation}
The dependence on the point $x$ on the right hand side is implicit in the choice of integration domain
\begin{equation}
    D_x(\epsilon)=\Big\{ (r,\theta)\in T_x\mathcal{M}\, \Big|\, r\in [0,\epsilon], \theta\in [0,2\pi)\Big\}\,,
\end{equation}
which is a disc inside the tangent space $T_x\mathcal{M}$.

The group $\SO(2)$ acts on feature maps by rotations $\varphi\in [0,2\pi)$ according to the regular representation $\rho$:
\begin{equation}
(\rho(\varphi)f)(r, \theta)=f(r, \theta-\varphi)\,.
\end{equation}
Under such a rotation the convolution by $\kappa_m$ is equivariant
\begin{equation}
(\kappa_m\star \rho(\varphi)f)(x)=e^{im\varphi}(\kappa_m\star f)(x)\,,
\end{equation}
as desired. Note that when $m=0$ the convolution is \emph{invariant} under rotations.

Now assume that we may approximate $\mathcal{M}$ by a triangular mesh. This is done in order to facilitate the computer implementation of the network. The feature map is represented by a vector $f_i$ at the vertex $i$ in the triangular mesh. This plays the role of choosing a point $x\in \mathcal{M}$. Suppose we have a deep network and that we are studying the feature vector at layer $\ell$ in the network. When needed we then decorate the feature vector with the layer $\ell$ as well as the degree $m$ of rotation: $f_{i, m}^{\ell}$.

A feature vector is parallel transported from a vertex $j$ to vertex $i$ along a geodesic connecting them. Any such parallel transport can be implemented by a rotation in the tangent space:
\begin{equation}
P_{j\to i}(f_{j,m}^\ell)=e^{im\varphi_{ji}} f_{j,m}^\ell\,,
\end{equation}
where $\varphi_{ji}$ is the angle of rotation. We are now equipped to define the convolutional layers of the network. For any vertex $i$ we denote by $\mathcal{N}_i$ the set of vertices that contribute to the convolution. In the continuous setting this corresponds to the \emph{support} of the feature map $f$ on $\mathcal{M}$. The convolution mapping features at layer $\ell$ to features at layer $\ell+1$ at vertex $i$ is now given by
\begin{equation}
f^{\ell+1}_{i, m+m'} =\sum_{j\in \mathcal{N}_i} w_j\,  \kappa_m (r_{ij}, \theta_{ij}; \beta)P_{j\to i}(f^{\ell}_{j, m'})\,.
\label{discreteGaugeConv}
\end{equation} 
The coefficient $w_j$ represents the approximation of the integral measure and is given by
\begin{equation}
w_j=\frac{1}{3} \sum_{jkl} A_{jkl}\,,
\end{equation}
where $A_{jkl}$ is the area of the triangle with vertices $jkl$. The radial function $R$ and the phase $e^{i\beta}$ are learned parameters of the network. The coordinates $(r_{ij}, \theta_{ij})$ represents the polar coordinates of the tangent space of every vertex $j$ in $\mathcal{N}_i$.

Let us now verify that this is indeed the building blocks of a deep equivariant network. To this end we should prove that the following diagram commutes:
\begin{equation}
 \begin{tikzcd}
     f^{\ell}_{i, m}\arrow[r,maps to]\arrow[d,maps to] & {f^\prime}^{\ell}_{i, m}\arrow[d,maps to]\\
     f^{\ell+1}_{i, m+m^{\prime}}\arrow[r,maps to] & {f^\prime}^{\ell+1}_{i, m+m^\prime}
 \end{tikzcd}
\end{equation}
Here, $f'=\rho(-\varphi)f$. The commutativity of this diagram ensures that a coordinate change in the tangent space commutes with the convolution. Let us now check this.

If we rotate the coordinate system at vertex $i$ by an angle $-\varphi$ the feature vector transforms according to
\begin{equation}
 f^{\ell}_{i, m} \rightarrow {f^{\prime}}^{\ell}_{i, m} = e^{im\varphi} f^{\ell}_{i, m}\,.
 \end{equation}
This coordinate transformation will affect the coordinates $(r_{ij}, \theta_{ij})$ according to
\begin{equation}
(r^{\prime}_{ij}, \theta^{\prime}_{ij})=(r_{ij}, \theta_{ij}+\varphi)\,.
\end{equation}
The parallel transport $P_{j\to i }$ of features from $j$ to $i$ further transforms as
\begin{equation}
P^{\prime}_{j\to i }=e^{im\varphi}P_{j\to i }\,.
\end{equation}
Using the above observations, let us now write the convolution with respect to the rotated coordinate system
\begin{eqnarray}
{f^{\prime}}^{\ell+1}_{i, m+m'}&=&\sum_{j\in \mathcal{N}_i} w_j\,  \kappa_m (r^{\prime}_{ij}, \theta^{\prime}_{ij}; \beta)P^{\prime}_{j\to i}({f^{\prime}}^{\ell}_{j, m'}) = e^{i(m+m^{\prime})}{f}^{\ell+1}_{i, m+m'}\,.
\end{eqnarray}
Thus we conclude that the diagram commutes. Note that nonlinearities can also be implemented in an equivariant fashion by only acting on the radial profile $R(r_{ij})$ but leaving the angular phase $e^{i\theta}$ untouched.

The formula (\ref{discreteGaugeConv}) can be viewed as a discretization of the general gauge equivariant convolution on an arbitrary $n$-dimensional manifold $\mathcal{M}$ given in equation~\eqref{eq:cheng_conv}.
 The disc $D_x(\epsilon)$ plays the role of the ball $B_R$ in the general formula. The combination $\sqrt{\det(g_{\mathcal{M},x})}\dd X$ is approximated by the weight coefficients $w_j$ in (\ref{discreteGaugeConv}), while the coordinates $X$ in the ball $B_R$   corresponds to $(r_{ij}, \theta_{ij})$ here. The parallel transport of the input feature $f$ is implemented by the exponential map $\exp_x X$, whose discretization is $P_{j\to i}$. Thus we conclude that the harmonic network discussed above may be viewed as a discretized two-dimensional version of gauge equivariant convolutional neural networks, a fact that appears to have gone previously unnoticed in the literature.

\subsubsection{Gauge-equivariant mesh CNNs}

A different approach to gauge equivariant networks on meshes was given in \cite{CohenMesh}. A mesh can be viewed as a discretization of a two-dimensional manifold. A mesh consists of a set of vertices, edges, and faces. One can represent the mesh by a graph, albeit at the expense of loosing information concerning the angle and ordering between incident edges. Gauge equivariant CNNs on meshes can then be modeled by graph CNNs with gauge equivariant kernels.

Let $M$ be a mesh, considered as a discretization of a two-dimensional manifold $\mathcal{M}$. We can describe this by considering a set of vertices $\mathcal{V}$ in $\mathbb{R}^3$, together with a set of tuples $\mathcal{F}$ consisting of vertices at the corners of each face. The mesh $M$ induces a graph $\mathcal{G}$ by ignoring the information about the coordinates of the vertices.

We first consider  graph convolutional networks. At a vertex $x\in\mathcal{V}$ the convolution between a feature $f$ and a kernel $K$ is given by
\begin{equation}
(\kappa\star f)(x)=\kappa_{\text{self}}f(x)+\sum_{y\in \mathcal{N}_x}\kappa_{\text{nb}} f(x)\,.
\end{equation}
The sum runs over the set $\mathcal{N}_x$ of neighboring vertices to $x$. The maps $\kappa_{\text{self}}, \kappa_{\text{nb}}\in \mathbb{R}^{N_\mathrm{in}\times N_\mathrm{out}}$ are modeling the self-interactions and nearest neighbor interactions, respectively. Notice that $\kappa_{\text{nb}}$ is independent of $y \in \mathcal{N}_x$ and so does not distinguish between neighboring vertices. Such kernels are said to be \emph{isotropic} (see also Remark~\ref{rk:isotropic}).

One can generalize this straightforwardly by allowing the kernel $\kappa_\mathrm{nb}$ to depend on the neighboring vertices. In order to obtain a general gauge equivariant network on $M$ we must further allow for features to be parallel transported between neighboring vertices. To this end we introduce the matrix $\rho(k_{y\to x})\in \mathbb{R}^{N_{in}\times N_{out}}$ which transports the feature vector at $y$ to the vertex $x$. The notation here is chosen to indicate that this is a representation $\rho$ of the group element $k_{y\to x}$ that rotates between neighboring vertices in $M$. 

Putting the pieces together we arrive at the gauge equivariant convolution on the mesh $M$:
\begin{equation}
(\kappa\star f)(x)=\kappa_\mathrm{self}f(x)+\sum_{y\in \mathcal{N}_x}\kappa_\mathrm{nb}(\theta_{xy})\big(\rho(g_{y\to x})f\big)(x)\,,
\end{equation}
where $\theta_{xy}$ is the polar coordinate of $y$ in the tangent space $T_x M$. Under a general rotation $k_\varphi$ by an angle $\varphi$ the kernels must transform according to
\begin{equation}
\label{meshgaugecondition}
\kappa(\theta-\varphi)=\rho_\mathrm{out}(k_{-\varphi})\kappa(\theta)\rho_\mathrm{in}(k_\varphi)\,,
\end{equation}
in order for the network to be equivariant:
\begin{equation}
(\kappa\star \rho_\mathrm{in}(k_{-\varphi})f)(x)=\rho_\mathrm{out}(k_{-\varphi})(\kappa\star f)(x)\,.
\end{equation}

Let us compare this with the harmonic networks discussed in the previous section. If we omit the self-interaction term, the structure of the convolution is very similar. The sum is taken over neighboring vertices $\mathcal{N}_x$, which is analogous to~\eqref{discreteGaugeConv}. The kernel $\kappa_\mathrm{nb}$ is a function of the polar coordinate $\theta_{xy}$ in the tangent space at $x$. The corresponding radial coordinate $r_{xy}$ is suppressed here, but obviously we can generalize to kernels of the form $\kappa_\mathrm{nb}(r_{xy}, \theta_{xy})$. Note, however, that here we have not made any assumption on the angular dependence of $\kappa_\mathrm{nb}(r_{xy}, \theta_{xy})$, in contrast to the circular harmonics in~(\ref{circularharmonics}). Note also that the condition~\eqref{meshgaugecondition} is indeed a special case of the general kernel condition for gauge equivariance as given in  Theorem~\ref{thm:gaugeeq}.

\section{Group equivariant layers for homogeneous spaces}
\label{sec:equiv-conv-layers}

\noindent In the previous section, we considered neural networks which process data and feature maps defined on general manifolds $\mathcal{M}$, and studied the equivariance of those networks with respect to local gauge transformations. A prominent example was the local transformations of frames induced by a change of coordinates on $\mathcal{M}$. The freedom to choose frames is a local (gauge) symmetry which exists in any vector bundle over the manifold $\mathcal{M}$, but if $\mathcal{M}$ has an additional global  (translation) symmetry, this can be exploited with great performance gains. In this section, we discuss \emph{group equivariant convolutional neural networks (GCNNs)} \cite{cohen2018,Cohen2016}, which are equivariant with respect to global transformations of feature maps induced from a global symmetry on $\mathcal{M}$. This section is largely based on \cite{aronsson2021}, with a few additions: We relate convolutional layers in GCNNs to the gauge equivariant convolutions defined in Section \ref{subsec:gauge_equivariant_convoultion}, and we also discuss equivariance with respect to \emph{intensity} and analyze its compatibility with group equivariance.

\subsection{Homogeneous spaces and bundles}

\subsubsection{Homogeneous spaces}
A manifold $\mathcal{M}$ with a sufficiently high degree of symmetry gives rise to symmetry transformations which translate any point $x \in \mathcal{M}$ to any other point $y \in \mathcal{M}$. For instance on a sphere $\mathcal{M} = S^2$, any point $x$ can be rotated into any other point $y$; similarly, any point $y$ in Euclidean space $\mathcal{M} = \mathbb{R}^n$ can be reached by translation from any other point $x$. Intuitively, this means that all points on the manifold are equivalent. This idea is formalized in the notion of a homogeneous space.

\begin{definition}
Let $G$ be a topological group. A topological space $\mathcal{M}$ is called a \emph{homogeneous $G$-space}, or just a \emph{homogeneous space}, if there is a continuous, transitive group action
\begin{equation}\label{eq:G_action}
G \times \mathcal{M} \to \mathcal{M}\,, \qquad (g,x) \mapsto g \cdot x\,.
\end{equation}
\end{definition}

In the special case of Lie groups $G$ and smooth manifolds $\mathcal{M}$, all homogeneous $G$-spaces are diffeomorphic to a quotient space $G/K$, with $K \leq G$ a closed subgroup \cite[Theorem 21.18]{lee2013smooth}. We therefore restrict attention to such quotient spaces. The elements of a homogeneous space $\mathcal{M} = G/K$ are denoted sometimes as $x$ and sometimes as $gK$, depending on the context.

\begin{remark}
For technical reasons, the Lie group $G$ is assumed to be unimodular. This is not a strong assumption, since examples of such groups include all finite or (countable) discrete groups, all abelian or compact Lie groups, the Euclidean groups, and many others \cite{folland2016course,fuhr2005abstract}. We also assume the subgroup $K \leq G$ to be compact - a common assumption that includes most homogeneous spaces of practical interest.
\end{remark}

\begin{example}\hfill
\begin{enumerate}
\item Any group $G$ is a homogeneous space over itself with respect to group multiplication. In this case, $K$ is the trivial subgroup so that $\mathcal{M} = G/K = G$.

\item In particular, the pixel lattice $\mathcal{M} = \mathbb{Z}^2$ used in ordinary CNNs is a homogeneous space with respect to translations: $G=\mathbb{Z}^2$.

\item The $n$-sphere is a homogeneous space $S^n = \SO(n+1)/\SO(n)$ for all $n \geq 1$. The special case $n = 2$, $S^{2} = \SO(3)/\SO(2)$, has been extensively studied in the context of geometric deep learning, cf.\ Section~\ref{sec:spherical}.

\item Euclidean space $\mathbb{R}^n = \E(n)/\mO(n)$ is homogeneous under rigid motions; combinations of translations and rotations which form the Euclidean group $G = \E(n) = \mathbb{R}^n \rtimes \mO(n)$.
\end{enumerate}
\end{example}

Homogeneous $G$-spaces are naturally equipped with a bundle structure \cite[\S 7.5]{steenrod1999topology} since $G$ is decomposed into orbits of the subgroup $K$, which form fibers over the base manifold $G/K$. The bundle projection is given by the quotient map
\begin{align}
q : G \to G/K \,, \qquad g \mapsto gK\,,
\end{align}
which maps a group element to its equivalence class. The free right action $G \times K \to G$ defined by right multiplication, $(g,k) \mapsto gk$, preserves each fiber $q^{-1}(gK)$ and turns $G$ into a principal bundle with structure group $K$. We can therefore view the homogeneous setting as a special case of the framework discussed in Section \ref{sec:gaugeeq} with $P = G$ and $\mathcal{M} = G/K$.

\subsubsection{Homogeneous vector bundles} Consider a rotation $g \in \SO(3)$ of the sphere, such that $x \in S^2$ is mapped to $gx \in S^2$. Intuitively, it seems reasonable that when the sphere rotates, its tangent spaces rotate with it and the corresponding transformation $T_xS^2 \to T_{gx}S^2$ should be linear, because all tangent vectors are rotated in the same way. Indeed, for any homogeneous space $\mathcal{M} = G/K$, the differential of the left-translation map $L_g$ on $\mathcal{M}$ is a linear isomorphism $\dd L_g : T_x\mathcal{M} \to T_{gx}\mathcal{M}$ for all $x \in \mathcal{M}$ and each $g \in G$. This means that the transitive $G$-action on the homogeneous space $\mathcal{M}$ induces a transitive $G$-action on the tangent bundle $T\mathcal{M}$ that is linear on each fiber. This idea is captured in the notion of a homogeneous vector bundle.

\begin{definition}[{\cite[\S 5.2.1]{wallach2018harmonic}}]\label{def:hombundle}
Let $\mathcal{M}$ be a homogeneous $G$-space and let $E \xrightarrow{\pi} \mathcal{M}$ be a smooth vector bundle with fibers $E_x$. We say that $E$ is \emph{homogeneous} if there is a smooth left action $G \times E \to E$ satisfying
\begin{equation}\label{eq:bundleaction}
g \cdot E_x = E_{gx}\,,
\end{equation}
and such that the induced map $L_{g,x} : E_x \to E_{gx}$ is linear, for all $g \in G, x \in \mathcal{M}$.
\end{definition}

Associated vector bundles $E_\rho =G \times_\rho V_\rho$ are homogeneous vector bundles with respect to the action $G \times E_\rho \to E_\rho$ defined by $g \cdot [g', v] = [gg', v]$. The induced linear maps
\begin{equation}\label{eq:bundlelinearmap}
L_{g,x} : E_x \to E_{gx}\,, \qquad [g', v] \mapsto [gg', v]\,,
\end{equation}
leave the vector inside the fiber invariant and are thus linear. Any homogeneous vector bundle $E$ is isomorphic to an associated vector bundle $G \times_\rho E_K$ where $K = eK \in G/K$ is the identity coset and where $\rho$ is the representation defined by $\rho(k) = L_{k,K} : E_K \to E_K$ \cite[\S 5.2.3]{wallach2018harmonic}.

\begin{remark}
In this section, $(\rho,V_\rho)$ and $(\eta,V_\eta)$ are finite-dimensional unitary representations of the compact subgroup $K \leq G$. Unitarity is important for defining induced representations and, by extension, $G$-equivariant layers below. This is not a restriction of the theory. Indeed, since $K$ is compact, unitarity of $\rho$~($\eta$) can be assumed without loss of generality by defining an appropriate inner product on $V_\rho$ ($V_\eta$) \cite[Lemma 7.1.1]{deitmar2014principles}.
\end{remark}

Let $\langle \cdot , \cdot \rangle_\rho$ be an inner product that turns $V_\rho \simeq E_K$ into a Hilbert space and makes $\rho$ unitary. This inner product then induces an inner product on each fiber $E_x$,
\begin{equation}\label{eq:unitary_structure}
\langle [g, v] , [g, w]\rangle_x = \langle v\,, w\rangle_\rho,
\end{equation}
which is well-defined precisely because $\rho$ is unitary. Further consider the unique $G$-invariant measure $\dd x$ on $G/K$ such that, for all integrable functions $f : G \to \mathbb{C}$ \cite[Theorem 1.5.3]{deitmar2014principles},
\begin{equation}\label{eq:quotient_integral_formula}
\int_G f(g) \ \dd g = \int_{G/K} \int_K f(xk) \ \dd k \dd x\,.
\end{equation}
We can then combine the measure $\dd x$ with \eqref{eq:unitary_structure} to integrate the point-wise norm of a section.

\begin{definition}\label{def:L2_homogeneous} Let $(\rho,V_\rho)$ be a finite-dimensional unitary $K$-representation and consider the homogeneous vector bundle $E_\rho = G \times_\rho V_\rho$.
\begin{enumerate}
\item
The \emph{induced representation} $\mathrm{ind}_K^G \rho$ of $G$ is the representation
\begin{equation}
\big( \mathrm{ind}_K^G \rho(g) s\big)(x) = g \cdot s(g^{-1}x)\,,
\end{equation}
on the complex Hilbert space of square-integrable sections
\begin{equation}
L^2(E_\rho) = \left\{ s : G/K \to E_\rho \ \middle| \  \int_{G/K} \| s(x)\|_x^2 \ \dd x < \infty\right\}\,.
\end{equation}

\item The \emph{induced representation} $\mathrm{Ind}_K^G \rho$ of $G$ is the representation
\begin{equation}
\big( \mathrm{Ind}_K^G \rho(g) f\big)(g') = f(g^{-1}g')\,,
\end{equation}
on the complex Hilbert space of square-integrable feature maps
\begin{equation}
L^2(G;\rho) = \left\{ f : G \to V_\rho \ \middle| \ \int_G \| f(g)\|_\rho^2 \ \dd g < \infty \right\}\,.
\end{equation}
\end{enumerate}
\end{definition}

The linear isomorphism $\varphi_\rho : C(G;\rho) \to \Gamma(E_\rho)$,  $f \mapsto s_f$ extends to a unitary isomorphism $L^2(G;\rho) \to L^2(E_\rho)$ that intertwines the induced representations $\mathrm{ind}_K^G \rho$ and $\mathrm{Ind}_K^G \rho$. That is, the induced representations are unitarily equivalent and we therefore choose to identify them.

\subsection{Group equivariant layers} To summarize the previous subsection, global symmetry of the homogeneous space $\mathcal{M} = G/K$ gives rise to homogeneous vector bundles and an induced representation that lets us translate sections and feature maps. GCNNs are motivated by the idea that layers between feature maps should preserve the global translation symmetry of $\mathcal{M}$. They do so by intertwining induced representations.

\begin{definition}\label{def:layers} A \emph{$G$-equivariant layer} is a bounded linear map $\Phi : L^2(E_\rho) \to L^2(E_\eta)$ that intertwines the induced representations:
\begin{equation}\label{eq:Glayer}
\Phi \circ \mathrm{ind}_K^G \rho = \mathrm{ind}_K^G \eta \circ \Phi\,.
\end{equation}
That is, $G$-equivariant layers are elements $\Phi \in \mathrm{Hom}_G(L^2(E_\rho),L^2(E_\eta))$.
\end{definition}

The unitary equivalence between $\mathrm{ind}_K^G \rho$ and $\mathrm{Ind}_K^G \rho$ implies that any $G$-equivariant layer $\Phi : L^2(E_\rho) \to L^2(E_\eta)$ induces a unique bounded linear operator $\phi : L^2(G;\rho) \to L^2(G;\eta)$ such that $\Phi s_f = s_{\phi f}$, as in Section \ref{sec:gaugeeq}. This operator also intertwines the induced representations,
\begin{equation}
\phi \circ \mathrm{ind}_K^G \rho = \mathrm{ind}_K^G \eta \circ \phi\,,
\end{equation}
hence $\phi \in \mathrm{Hom}_G(L^2(G;\rho),L^2(G;\eta))$. The operators $\phi$ are also called $G$-equivariant layers.

As the name suggests, GCNNs generalize convolutional neural networks ($G = \mathbb{Z}^2, K = \{0\}$) to other homogeneous spaces $G/K$. The next definition generalize convolutional layers \eqref{eq:2dCNN} in this direction.

\begin{definition}\label{def:convlayer_sec3}
A \emph{convolutional layer} $\phi : L^2(G;\rho) \to L^2(G;\eta)$ is a bounded operator
\begin{equation}\label{eq:convlayer}
(\phi f)(g) = [\kappa \star f](g) = \int_G \kappa(g^{-1}g') f(g') \ \dd g'\,, \qquad f \in L^2(G;\rho)\,,
\end{equation}
with an operator-valued kernel $\kappa : G \to \mathrm{Hom}(V_\rho,V_\eta)$.
\end{definition}

Given bases for $V_\rho$ and $V_\eta$, we can think of the kernel $\kappa$ as a matrix-valued function. Each row in this matrix is a function $\kappa_i : G \to V_\rho$, just like the feature maps, so we can interpret each row as a separate filter that we convolve with respect to. This is analogous to ordinary CNNs in which both data and filters have the same structure as images. Furthermore, $\dim V_\rho$ is the number of input channels and $\dim V_\eta$ the number of output channels, one for each filter $\kappa_i$. From this perspective, the full matrix $\kappa$ is a stack of $\dim V_\eta$ filters and the convolutional layer $\phi$ computes all output-channels simultaneously.

Convolutional layers form the backbone of GCNNs, and implementations are often based on these layers. Note that the kernel in a convolutional layer cannot be chosen arbitrarily but must satisfy certain transformation properties, to make sure that $\kappa \star f$ transforms correctly. First of all, the requirement that $\kappa \star f \in  L^2(G;\eta)$ implies that
\begin{equation}
 \int_G \eta(k)\kappa(g) f(g) \ \dd g = \eta(k)[\kappa \star f](e) = [\kappa \star f](k^{-1}) = \int_G \kappa(kg) f(g) \ \dd g\,, 
\end{equation}
which is satisfied if $\kappa(kg) = \eta(k)\kappa(g)$. Moreover, unimodularity of $G$ means that the left Haar measure on $G$ is also right-invariant, so we can perform a change of variables $g \mapsto gk$,
\begin{equation}
\int_G \kappa(g)f(g) \ \dd g =  \int_G \kappa(gk) f(gk) \ \dd g = \int_G \kappa(gk) \rho(k)^{-1}f(g) \ \dd g\,,
\end{equation}
which indicates that $\kappa(gk) = \kappa(g)\rho(k)$. These relations can be summarized in one equation,
\begin{equation}
\kappa(kgk') = \eta(k)\kappa(g)\rho(k')\,, \qquad g \in G, k,k' \in K\,,
\end{equation}
which was previously discussed in \cite{cohen2018}. A direct consequence is that $\kappa(gk)f(gk) = \kappa(g)f(g)$ for all $k \in K$ and each $g \in G$, so the product $\kappa f$ only depends on the base point $x = q(g) \in G/K$. The quotient integral formula \eqref{eq:quotient_integral_formula} therefore implies that
\begin{equation}\label{eq:conv_over_GK}
[\kappa \star f](g) = \int_{G/K} \kappa(g^{-1}y)f(y) \ \dd y\,.
\end{equation}

\begin{remark}
The integral \eqref{eq:conv_over_GK} is closely related to the gauge equivariant convolution \eqref{eq:phi_gauge_case}. First of all, homogeneous spaces $\mathcal{M} = G/K$ always admit Riemannian metrics $g_\mathcal{M}$ that are invariant under translations \cite[\S 2.3]{Howard1994}, see also \cite{sors2004integral}. The Riemannian volume form $\mathrm{vol}_\mathcal{M}$ is also invariant, and the corresponding Riemannian volume measure is thus an invariant measure on $\mathcal{M}$. By the quotient integral formula \cite[Theorem 1.5.3]{deitmar2014principles}, the Riemannian volume measure is related to $\dd y$ by a positive scaling factor $c > 0$, so that
\begin{equation}
\int_\mathcal{M} \kappa(g^{-1}y)f(y) \ \dd y = \int_\mathcal{M}  L_g \kappa f \ \mathrm{vol}_\mathcal{M}\,.
\end{equation}
The sliding kernel $\kappa(g^{-1}y)$ in \eqref{eq:conv_over_GK} can be viewed as a special case of the explicitly base point-dependent kernel $\kappa(x,X)$ in \eqref{eq:phi_gauge_case}, after taking into account the diffeomorphism between $B_R$ and $\{y \in \mathcal{M} : d_{g_\mathcal{M}}(x,y) < R\}$. We can also interpret the domain of integration as the kernel support. The parallel transport map $\mathcal{T}_X$ need not be invoked here as the integrand in \eqref{eq:conv_over_GK} is already defined on $x \in G/K$, without lifting to $G$. Finally, the relation $\mathrm{vol}_\mathcal{M}|_x = \mathrm{vol}_{T_x\mathcal{M}}$ lets us rewrite \eqref{eq:conv_over_GK}, with some abuse of notation and ignoring the constant $c$, as
\begin{equation}
[\kappa \star f](g) = \int_{B_R} L_g\kappa f \ \mathrm{vol}_{T_x\mathcal{M}}\,,
\end{equation}
which is similar to \eqref{eq:phi_gauge_case}.
\end{remark}

Boundedness of \eqref{eq:convlayer} is guaranteed if the kernel matrix elements $\kappa_{ij} : G \to \mathbb{C}$ are integrable functions, for some choices of bases in $V_\rho$, $V_\eta$ \cite{aronsson2021}.

\begin{theorem}[{\cite{aronsson2021}}]\label{thm_main}
Let $\phi : L^2(G;\rho) \to L^2(G;\eta)$ be a bounded linear map.
\begin{enumerate}
\item If $\phi$ is a convolutional layer, then $\phi$ is a $G$-equivariant layer.
\item If $\phi$ is a $G$-equivariant layer such that $\Im(\phi)$ is a space of bandlimited functions, then $\phi$ is a convolutional layer.
\end{enumerate}
\end{theorem}

The bandlimit criteria is automatically satisfied for all finite groups and for discrete abelian groups such as $G = \mathbb{Z}^2$ \cite[Corollaries 20-21]{aronsson2021}. It follows that all $G$-equivariant layers can be expressed as convolutional layers for these groups.

\subsection{Equivariance with respect to intensity}\label{subsec:equiv-intensity} In image processing tasks, a neural network may treat an image differently depending on the level of saturation. One way to avoid this is to design saturation-equivariant neural networks, or \emph{intensity}-equivariant neural networks if we generalize beyond image data. In this section, we define this notion of equivariance and investigate the question of when a $G$-equivariant layer is also intensity-equivariant. This part is based on the concept of induced systems of imprimitivity in \cite[\S 3.2]{kaniuth2013induced}.

Mathematically, one changes the intensity of a data point $s \in L^2(E_\rho)$ by scaling the vector at each point: $(\psi s)(x) = \psi(x) s(x)$ where $\psi : G/K \to \mathbb{C}$ is a continuous function. Equivalently, we can scale the feature maps instead, via the mapping
\begin{equation}\label{eq:intensity}
S(\psi) : L^2(G;\rho) \to L^2(G;\rho)\,, \qquad \big(S(\psi)f\big)(g) = \psi(gK)f(g)\,.
\end{equation}
For technical reasons we assume that $\psi$ vanishes at infinity: $\psi \in C_0(G/K)$.

\begin{definition}
A bounded linear map $\phi : L^2(G;\rho) \to L^2(G;\eta)$ is \emph{equivariant with respect to intensity}, or \emph{intensity-equivariant}, if
\begin{equation}\label{eq:intensity_equivariant}
S(\psi) \circ \phi = \phi \circ S(\psi)\,,
\end{equation}
for all $\psi \in C_0(G/K)$.
\end{definition}

\begin{remark}
The mapping \eqref{eq:intensity} is a $*$-representation of $C_0(G/K)$ on the space $L^2(G;\rho)$, and intensity-equivariant maps \eqref{eq:intensity_equivariant} are intertwiners of two such representations.
\end{remark}

\begin{example}
Let $T : V_\rho \to V_\eta$ be a bounded linear map that intertwines $\rho$ and $\eta$, that is, $T \in \mathrm{Hom}_K(V_\rho,V_\eta)$. By letting $T$ act point-wise on feature maps $f \in L^2(G;\rho)$, we obtain a bounded linear map
\begin{equation}\label{eq:intensity_equivariant_layer}
\phi_T : L^2(G;\rho) \to L^2(G;\eta)\,, \qquad (\phi_T f)(g) = T\big(f(g)\big)\,,
\end{equation}
and we observe that $\phi_T$ is equivariant with respect to intensity. This is because $\phi_T$ performs point-wise transformations of vectors and $S$ performs point-wise multiplication by scalar, so we need only employ the linearity of $T$:
\begin{equation}
\big( S(\psi) \phi_T f\big)(g) = \psi(gK) T\big(f(g)\big) = T\big( \psi(gK) f(g)\big) = \big( \phi_T S(\psi)f\big)(g)\,.
\end{equation}
Note that \eqref{eq:intensity_equivariant_layer} is also $G$-equivariant since its action on $f(g) \in V_\eta$ does not depend on $g \in G$. Indeed,
\begin{equation}
\big( \phi_T \circ \mathrm{Ind}_K^G\rho(g) f\big)(g') = T\big(f(g^{-1}g')\big) = (\phi_T f)(g^{-1}g') = \big( \mathrm{Ind}_K^G\eta(g) \circ \phi_T f\big)(g)\,.
\end{equation}
\end{example}

It turns out that \eqref{eq:intensity_equivariant_layer} are the only possible transformations $\phi : L^2(G;\rho) \to L^2(G;\eta)$ that are both $G$-equivariant and intensity-equivariant.

\begin{theorem}[{\cite[Theorem 3.16]{kaniuth2013induced}}]\label{thm:intensity_equivariance}
Let $\phi : L^2(G;\rho) \to L^2(G;\eta)$ be a $G$-equivariant layer. Then $\phi$ is intensity-equivariant if and only if $\phi = \phi_T$ for some $T \in \mathrm{Hom}_K(V_\rho,V_\eta)$.
\end{theorem}

For some groups, this theorem exclude convolutional layers from being intensity-equivariant, as the following example illustrates. This means that intensity-equivariant and convolutional layers are two separate classes of layers for these groups.

\begin{example}
Consider the special case $G = \mathbb{R}$ and let $K = \{0\}$ be the trivial subgroup. Then $\rho,\eta$ are trivial representations, and assume for simplicity that $\dim V_\rho = \dim V_\eta = 1$ so that $L^2(\mathbb{R};\rho) = L^2(\mathbb{R};\eta) = L^2(\mathbb{R})$. If we now let $\phi : L^2(\mathbb{R}) \to L^2(\mathbb{R})$ be a convolutional layer with some kernel $\kappa : \mathbb{R} \to \mathbb{C}$,
\begin{equation}\label{eq:ex_intensity_conv}
(\phi f)(x) = [\kappa \star f](x) = \int_{-\infty}^\infty \kappa(y-x) f(y) \ \dd y, \qquad f \in L^2(\mathbb{R})\,,
\end{equation}
and consider an arbitrary element $\psi \in C_0(\mathbb{R})$, then
\begin{equation}\label{eq:ex_intensity_eq1}
\phi( S(\psi) f)(x) = \int_{-\infty}^\infty \kappa(y-x) \psi(y) f(y) \ \dd y\,.
\end{equation}
Note that the function $\psi$ is part of the integrand, which is not the case for $S(\psi) \phi f$. This is essentially what prevents convolutional layers \eqref{eq:ex_intensity_conv} from being intensity-equivariant.

To see this, fix $\epsilon > 0$ and consider the bump function $\psi_\epsilon : \mathbb{R} \to [0,1]$ defined by
\begin{equation}
\psi_{\epsilon}(x) = \left\{
\begin{aligned}
&\exp\left(1 - \frac{1}{1 - (x/\epsilon)^2}\right) , \qquad &&x \in (-\epsilon,\epsilon)\\
&0, &&\text{otherwise}
\end{aligned}
\right. \,,
\end{equation}
which is supported on the compact interval $[-\epsilon,\epsilon]$ and satisfies $\psi_{\epsilon}(0) = 1$. Then $\psi_\epsilon \in C_0(\mathbb{R})$ and it is clear that
\begin{equation}\label{eq:ex_intensity_eq2}
\big(S(\psi_{\epsilon}) \phi f\big)(0) = \psi_{\epsilon}(0) \phi f(0) = \phi f(0)\,.
\end{equation}
Comparing \eqref{eq:ex_intensity_eq1} and \eqref{eq:ex_intensity_eq2}, we see that the convolutional layer $\phi$ is intensity-equivariant only if
\begin{equation}
\int_{-\infty}^\infty \kappa(y)f(y) \ \dd y = \phi f(0) = \int_{-\infty}^\infty \kappa(y) \psi_\epsilon(y) f(y) \ \dd y\,,
\end{equation}
for each $f \in L^2(\mathbb{R})$. Hölder's inequality yields the bound
\begin{equation}\label{eq:ex_intensity_bound}
|\phi f(0)| \leq \int_{-\infty}^\infty |\kappa(y)\psi_\epsilon(y) f(y)| \ \dd y = \|\kappa \psi_\epsilon f\|_1  \leq \| \kappa \psi_\epsilon\|_2 \|f\|_2\,.
\end{equation}
However, because $\psi_\epsilon \leq 1$ everywhere and vanishes outside $[-\epsilon,\epsilon]$, we have
\begin{equation}\label{eq:intensity_ex_bound}
\|\kappa \psi_\epsilon\|_2^2 \leq \int_{-\epsilon}^\epsilon |\kappa(x)|^2 \ \dd x\,,
\end{equation}
which vanishes as $\epsilon \to 0$. It follows that $\phi f(0) = 0$ for all $f \in L^2(\mathbb{R})$. This argument can be adapted to show that $\phi f(x) = 0$ for all $x \in \mathbb{R}$, so we conclude that $\phi$ must vanish identically. In other words, there does not exist a non-zero, intensity-equivariant convolutional layer in the case $G = \mathbb{R}$, $K = \{0\}$. This is consistent with Theorem \ref{thm:intensity_equivariance} because if a convolutional layer
\eqref{eq:ex_intensity_conv} had been intensity-equivariant, then Theorem \ref{thm:intensity_equivariance} would imply that $\phi = \phi_T$ acts point-wise, i.e. that $\phi f(x) \in \mathbb{C}$ only depends on $f(x) \in \mathbb{C}$. This would require the kernel $\kappa$ to behave like a Dirac delta function, which is not a mathematically well-defined function.
\end{example}

The conclusion would have been different, had \eqref{eq:intensity_ex_bound} not vanished in the limit $\epsilon \to 0$. This would have required the singleton $\{0\}$ to have non-zero Haar measure, which is only possible if the Haar measure is the counting measure, that is, if $G$ is discrete \cite[Proposition 1.4.4]{deitmar2014principles}. In that case, one could define a convolution kernel $\kappa_T : G \to \mathrm{Hom}(V_\rho,V_\eta)$ in terms of a Kronecker delta, $\kappa_T(g) = \delta_{g,e}T$ for some linear operator $T \in \mathrm{Hom}_K(V_\rho,V_\eta)$. The convolutional layer
\begin{equation}
(\phi f)(g) = [\kappa_T \star f](g) = \int_G \kappa_T(g^{-1}g') f(g') \ \dd g' = T\big(f(g)\big)\,,
\end{equation}
would then coincide with $\phi_T$ and intensity equivariance would be achieved.

\section{General group equivariant convolutions}
\label{sec:GC}
\noindent In the previous sections, we have developed a  mathematical framework for
the construction of equivariant convolutional neural networks on homogeneous
spaces, starting from an abstract notion of principal bundles, leading to the
convolutional integral \eqref{eq:convlayer}.  Here, we continue this discussion and consider
various generalizations of \eqref{eq:convlayer} which could be implemented in a
CNN. While doing so, we will reproduce various expressions for convolutions
found in the literature. A similar overview of convolutions over feature maps
and kernels defined on homogeneous spaces can be found in \cite{Kondor2018},
however our discussion also takes non-trivial input- and output representations
into account (see also \cite{cohen2018,
  cohenIntertwinersInducedRepresentations2018a}).

\subsection{Structure of symmetric feature maps}
Let $\mathcal{X},\mathcal{Y}$ be topological spaces and $V_1,V_2$ vector
spaces. The input features are maps $f_{1} : \mathcal{X} \to V_1$, the output
features are maps $f_2 : \mathcal{Y} \to V_2$. Assume that a group $G$ acts on
all four spaces by\footnote{Note that as opposed to previous sections, here, we
  are considering representations of the entire group $G$, not just of a
  subgroup.}
\begin{alignat}{2}
&\sigma_1(g) x\,, \ \forall x \in \mathcal{X}\,, \qquad &&\rho_1(g) v_1\,, \ \forall v_1 \in V_1\,, \\
&\sigma_2(g) y\,, \ \forall y \in \mathcal{Y}\,, \qquad &&\rho_2(g) v_2\,, \ \forall v_2 \in V_2\,,
\end{alignat}
for all $g \in G$. These actions can be combined to define a group action on the feature maps:
\begin{alignat}{1}
  &[\pi_1(g)f_{1}](x) = \rho_1(g) f_{1}(\sigma_1^{-1}(g)x)\,,\label{eq:2}\\
  &[\pi_2(g)f_2](y) = \rho_2(g) f_2(\sigma_2^{-1}(g)y)\,.\label{eq:3}
\end{alignat}

\begin{example}[GCNNs]\label{ex:GCNNs}
  In the simplest case of non-steerable GCNNs discussed above, we have $\rho_{1}=\rho_{2}=\id$ and the input to the first layer is defined on the homogeneous space $\mathcal{X} = G/K$ of $G$ with respect to some subgroup $K$. The output of the first layer then has $\mathcal{Y} = G$. Subsequent layers have $\mathcal{X} = \mathcal{Y} = G$. The group $G$ acts on itself and on $G/K$ by group multiplication: $\sigma_{1,2}(g')g=g'g$ for $g'\in G$ and $g\in G$ or $g\in G/K$.

  On top of this, \cite{Kondor2018} discusses also the cases of convolutions from a homogeneous space into itself, $\mathcal{X}=\mathcal{Y}=G/K$ and of a double coset space into a homogeneous space $\mathcal{X}=H\backslash G/K$, $\mathcal{Y}=G/K$. In all these cases, $\sigma_{1,2}$ are given by group multiplication.\footnote{Note, however, that there appear to be some mathematical inconsistencies in \cite{Kondor2018}. For example, in Case~I of Section~4.1, the symmetry group acts from the right on a left quotient space, which is ill-defined. Furthermore, the choice of coset representative denoted by $\bar{x}$ is not equivariant with respect to the group action, even from the left. Similar problems arise in Case~III.}
\end{example}

The representations $\sigma_{i},\rho_{i}$ arising in applications are dictated by the problem at hand. In the following example, we discuss some typical cases arising in computer vision.

\begin{example}[Typical computer vision problems]
\label{ex:compvision}
  Consider the case that the input data of the network consists of flat images. Then, for the first layer of the network, we have $\mathcal{X}=\RR^{2}$ and $V_{1}=\RR^{\Nc}$ if the image has $\Nc$ color channels and the input features are functions $f_{1}:\RR^{2}\rightarrow \RR^{\Nc}$. Typical symmetry groups $G$ of the data are rotations ($\SO(2)$), translations ($\RR^{2}$) or both ($\SE(2)$) or finite subgroups of these. For these groups, the representation $\sigma_{1}$ is the fundamental representation $\mathbf{2}_{\SO(2)}$ which acts by matrix multiplication, i.e.\ for $G=\SO(2)$, we have
  \begin{align}
    \sigma_{1}(\phi)=
    \begin{pmatrix}
      \cos\phi & \sin\phi\\
      -\sin\phi & \cos\phi
    \end{pmatrix}\,,\label{eq:so2fundrep}
  \end{align}
  and similarly for the other groups. Since the color channels of images do not contain directional information, the input representation $\rho_{1}$ is the trivial representation: $\rho_{1}=\um_{\Nc}$.

  The structure of the output layer depends on the problem considered. For image classification, the output should be a probability distribution $P(\Omega)$ over classes $\Omega$ which is invariant under actions of $G$. In the language of this section, this would correspond to the domain $\mathcal{Y}$ of the last layer and the representations $\sigma_{2}$ and $\rho_{2}$ being trivial and $V_{2}=P(\Omega)$.
  
  The cases of semantic segmentation and object detection are discussed in detail in Sections~\ref{sec:sem_seg_s2} and \ref{sec:obj_detection} below.
\end{example}

\subsection{The kernel constraint}
The integral mapping $f_{1}$ to $f_{2}$ is defined in terms of the kernel $\kappa$ which is a bounded, linear-operator valued, continuous mapping
\begin{align}
  \kappa : \mathcal{X} \times \mathcal{Y} \to \Hom(V_1,V_2)\,.
\end{align}
In neural networks, $\kappa$ additionally has local support, but we do not make this assumption here as the following discussion does not need it.

In order to have an integral which is compatible with the group actions, we require $\mathcal{X}$ to have a Borel measure which is invariant under $\sigma_1$, i.e.\
\begin{align}
  \int_\mathcal{X} f_{1}(\sigma_1(g)x) \dd x = \int_\mathcal{X} f_{1}(x) \dd x\label{eq:5},
\end{align}
for every integrable function $f_{1}: \mathcal{X}_1 \to V_{1}$ and all $g \in G$. Now, we can write the output features as an integral over the kernel $\kappa$ and the input features,
\begin{align}
  f_{2}(y)=[\kappa \cdot f_{1}](y) = \int_\mathcal{X}  \kappa(x,y) f_{1}(x) \dd x\,,\label{GC_transform}
\end{align}
where the matrix multiplication in the integrand is left implicit.

We now require the map from input to output features to be equivariant with respect to the group actions \eqref{eq:2} and \eqref{eq:3}, i.e.\ for any input function $f_{1}$, we require
\begin{equation}
[\kappa \cdot \pi_1(g) f_{1}] = \pi_2(g) [\kappa \cdot f_{1}]\qquad\forall g\in G\,. \label{GC_equivariance}
\end{equation}
This leads to the following
\begin{lemma}
  The transformation \eqref{GC_transform} is equivariant with respect to $\pi_1,\pi_2$ if the kernel satisfies the constraint
\begin{equation}\label{GC_constraint}
\kappa\big(\sigma_1^{-1}(g) x, \sigma_2^{-1}(g) y\big) = \rho_2^{-1}(g) \kappa(x,y) \rho_1(g)\,,\qquad\forall x\in \mathcal{X},\ \forall y\in\mathcal{Y},\ \forall g\in G\,.
\end{equation}
\end{lemma}
\begin{proof}
  The constraint \eqref{GC_constraint} is equivalent to
  \begin{align}
    \kappa\big(\sigma_{1}(g)x,y\big)\rho_{1}(g)=\rho_{2}(g)\kappa\big(x,\sigma_{2}^{-1}(g)y\big)\,.\label{eq:8}
  \end{align}
  Integrating against $f_{1}$ leads to
  \begin{align}
    \int_{\mathcal{X}}\kappa\big(\sigma_{1}(g)x,y\big)\rho_{1}(g)f_{1}(x)\dd x=\int_{\mathcal{X}}\rho_{2}(g)\kappa\big(x,\sigma_{2}^{-1}(g)y\big)f_{1}(x)\dd x\,.\label{eq:18}
  \end{align}
  Using \eqref{eq:5} on the left-hand side shows that this is equivalent to \eqref{GC_equivariance}.
\end{proof}

\subsection{Transitive group actions}
In the formulation above, the output feature map is computed as a scalar product of the input feature map with a kernel which satisfies the constraint \eqref{GC_constraint}. In this section, we discuss how the two can be combined into one expression, which is the familiar convolutional integral, if the group acts \emph{transitively} by $\sigma_1$ on the space $\mathcal{X}$. I.e.\ we assume that there exists a base point $x_0 \in \mathcal{X}$ such that for any $x \in \mathcal{X}$, there is a $g_x \in G$ with
\begin{equation}\label{GC_transitive}
x = \sigma_1(g_x) x_0\,.
\end{equation}
Defining
\begin{align}
  \kappa(y)=\kappa(x_{0},y),\label{eq:19}
\end{align}
we obtain from \eqref{GC_constraint}
\begin{align}
  \kappa(x,y) = \rho_{2}(g_x) \kappa\big(\sigma_{2}^{-1}(g_x) y\big)\rho_{1}^{-1}(g_x)\,.\label{eq:20}
\end{align}
Plugging this into \eqref{GC_transform} and using \eqref{GC_transitive} yields a convolution as summarized in the following proposition.
\begin{proposition}
  If $G$ acts transitively on $\mathcal{X}$, the map $\kappa\cdot f_{1}$ defined in \eqref{GC_transform} subject to the constraint \eqref{GC_constraint} can be realized as the convolution
  \begin{equation}\label{GC_conv}
    [\kappa \star f_{1}](y) = \int_{\mathcal{X}} \rho_{2}(g_x)\,\kappa\big(\sigma_{2}^{-1}(g_x) y\big)\,\rho_{1}^{-1}(g_x)\,f_{1}(\sigma_1(g_x) x_0) \,\dd x\,.
  \end{equation}
\end{proposition}
Since in \eqref{GC_conv}, the integrand only depends on $x$ through $g_{x}$, we can replace the integral over $\mathcal{X}$ by an integral over $G$ if the group element $g_{x}$ is unique for all $x\in\mathcal{X}$ (we can then identify $\mathcal{X}$ with $G$ by $x\mapsto g_{x}$). In this is the case, the group action of $G$ on $\mathcal{X}$ is called \emph{regular} (i.e.\ it is transitive and \emph{free}), leading to
\begin{proposition}\label{prop:regGroupAct}
  If $G$ acts regularly on $\mathcal{X}$, the map $\kappa\cdot f_{1}$ defined in \eqref{GC_transform} subject to the constraint \eqref{GC_constraint} can be realized as the convolution
  \begin{equation}\label{GC_Gconv}
  [\kappa \star f_{1}](y) = \int_{G} \rho_{2}(g)\, \kappa\big(\sigma_{2}^{-1}(g) y\big)\, \rho_{1}^{-1}(g)\,f_{1}(\sigma_{1}(g)x_{0})\,\dd g\,,
  \end{equation}
  where we use the Haar measure to integrate on $G$. Furthermore, for a regular group action, the group element $g_{x}$ in \eqref{eq:20} is unique and hence the kernel $\kappa(y)$ in \eqref{GC_Gconv}  is unconstrained.
\end{proposition}

\begin{remark}\label{rem:stabilizer}
  If there is a subgroup $K$ of $G$ which stabilizes $x_{0}$, i.e.\ $\sigma_{1}(h)x_{0}=x_{0}$ $\forall h \in K$, $\mathcal{X}$ can be identified with the homogeneous space $G/K$. Proposition \ref{prop:regGroupAct} corresponds to $K$ being trivial. As will be spelled out in detail in Section~\ref{sec:sem_dir_prod_grps}, the integral in \eqref{GC_Gconv} effectively averages the kernel $\kappa$ over $K$ before it is combined with $f_{1}$, leading to significantly less expressive effective kernels. In the case of spherical convolutions, this was pointed out in~\cite{makadia2007, cohen2018b}. Nevertheless, constructions of this form are used in the literature, cf.\ e.g.~\cite{esteves2018}.
\end{remark}

To illustrate \eqref{GC_Gconv}, we start by considering GCNNs as discussed in Example~\ref{ex:GCNNs}.

\begin{example}(GCNNs with non-scalar features)
  Consider a GCNN as discussed in Example \ref{ex:GCNNs} above. On $G$, a natural reference point is the unit element $e$, so we set $x_{0}=e$ and hence $g_{x}=x$. Since $\sigma_{1}$ is now a regular group action, we can use \eqref{GC_Gconv} which simplifies to
  \begin{equation}\label{GC_Gconv2}
    [\kappa \star f](y) = \int_{G} \rho_{2}(g) \kappa(\sigma_{2}^{-1}(g) y) \rho_{1}^{-1}(g)f(g)\dd g\,,
  \end{equation}
  where $\kappa(y)$ is unconstrained. This is the convolution used for GCNNs if the input- and output features are not scalars.
\end{example}

Another important special case for \eqref{GC_Gconv} is given by spherical convolutions which are widely studied in the literature, cf.\ Sections~\ref{sec:related-literature} and \ref{sec:spherical}.

\begin{example}[Spherical convolutions]\label{ex:sphericalCNN}
  Consider an equivariant network layer with $\mathcal{X} = \mathcal{Y} = G = \SO(3)$ as used in spherical CNNs~\cite{cohen2018b}. $G$ acts on itself transitively by left-multiplication, $\sigma_i(Q)R = QR$ for $Q,R\in\SO(3)$ with base point $x_{0}=e$. Therefore, according to Proposition~\eqref{GC_Gconv}, the transformation \eqref{GC_transform} can be written as\footnote{In \cite{cohen2018b}, the spherical convolution is defined as $[\kappa \star f_{1}](S) = \int_{\SO(3)} \kappa(S^{-1}R) f_{1}(R) \ \dd R$, which arises from our expression by $\rho_{1,2}\rightarrow\id$, $\kappa(R)\rightarrow\kappa(R^{-1})$. In the language developed below, the reference fixes a point in $\mathcal{Y}$ instead of $\mathcal{X}$. We use a different convention here to make the relation to the general case more transparent. For applications, the distinction is irrelevant.\label{fn:1}}
\begin{align}
  [\kappa \star f_{1}](S) = \int_{\SO(3)} \rho_2(R)\kappa(R^{-1}S)\rho_1^{-1}(R) f_{1}(R) \ \dd R\,,
  \label{SO3C_conv}
\end{align}
where the Haar measure on $\SO(3)$ is given in terms of the Euler angles $\alpha$, $\beta$, $\gamma$ by
\begin{align}
  \int_{\SO(3)}\dd R = \int_{0}^{2\pi}\dd\alpha\int_{0}^{\pi}\dd\beta\sin\beta\int_{0}^{2\pi}\dd\gamma \,.
\end{align}
\end{example}

Instead of assuming a transitive group action on $\mathcal{X}$ and using this to solve the kernel constraint, one can also assume a transitive group action on $\mathcal{Y}$ and use this to solve the kernel constraint. Specifically, if there is a $y_{0}\in\mathcal{Y}$ such that for any $y\in\mathcal{Y}$, we have a $g_{y}\in G$ satisfying
\begin{align}
  y=g_{y}y_{0}\,,
\end{align}
we can define
\begin{align}
  \kappa(x)=\kappa(x,y_{0})\,.
\end{align}
Then, according to \eqref{GC_constraint}, the two-argument kernel is given by
\begin{align}
  \kappa(x,y)=\rho_{2}(g_{y})\kappa(\sigma_{1}^{-1}(g_{y})x)\rho_{1}^{-1}(g_{y})\,,
\end{align}
yielding the following
\begin{proposition}
  If $G$ acts transitively on $\mathcal{Y}$, the map $\kappa\cdot f_{1}$ defined in \eqref{GC_transform} subject to the constraint \eqref{GC_constraint} can be realized as the convolution
  \begin{align}
    [\kappa\star f_{1}](y)=\int_{\mathcal{X}}\rho_{2}(g_{y})\kappa(\sigma_{1}^{-1}(g_{y})x)\rho_{1}^{-1}(g_{y}) f_{1}(x)\dd x\,.\label{eq:1}
  \end{align}
\end{proposition}
However, since the group element now depends on $y$ instead of on the integration variable $x$, we cannot replace the integral over $\mathcal{X}$ by an integral over $G$ as we did above in \eqref{GC_Gconv} and have to compute the group element $g_{y}$ for each $y$.\footnote{In the spherical case of fixing a point in $\mathcal{Y}$ to solve the kernel constraint, mentioned in footnote \ref{fn:1}, this problem does not arise: We have $\mathcal{X}=\mathcal{Y}=G$ and therefore integrate over $G$ from the beginning. Furthermore, we can choose $y_{0}=e$, so $y=g_{y}$.}

\subsection{Semi-direct product groups}
\label{sec:sem_dir_prod_grps}
In the previous section, we considered different convolutional integrals for the case that $G$ acts transitively on $\mathcal{X}$. In particular, we recovered the familiar integral over $G$ in the case that $G$ acts regularly on $\mathcal{X}$. In practice however, the action of $G$ on $\mathcal{X}$ is often transitive, but not regular. Instead, $G$ is a semi-direct product group $G=N\rtimes K$ of a normal subgroup $N$ and a subgroup $K$ such that $N$ acts regularly on $\mathcal{X}$ and one can choose a base point $x_{0}\in \mathcal{X}$ which is stabilized by $K$,
\begin{align}
  \forall x\in\mathcal{X}\quad\exists! n_{x}\in N\quad\text{s.t.}\quad x=\sigma_{1}(n_{x})x_{0}\quad\text{and}\quad\sigma_{1}(k)x_{0}=x_{0}\quad\forall k\in K\,.\label{eq:41}
\end{align}
Again, we define
\begin{align}
  \kappa(y)=\kappa(x_{0},y),
\end{align}
and use
\begin{align}
  \kappa(x,y) = \rho_{2}(n_x) \kappa(\sigma_{2}^{-1}(n_x) y)\rho_{1}^{-1}(n_x),\label{eq:21}
\end{align}
to obtain the convolutional integral
\begin{align}
  [\kappa\star f_{1}](y)=\int_{\mathcal{X}} \rho_{2}(n_{x})\kappa(\sigma_{2}^{-1}(n_{x})y)\rho_{1}^{-1}(n_{x})f_{1}(\sigma_{1}(n_{x})x_{0})\ \dd x\,.
\end{align}
Since $N$ acts regularly on $\mathcal{X}$, the integral over $\mathcal{X}$ can be replaced by an integral over $N$. However, since \eqref{eq:21} only fixes an element of $N$ but not of $K$, the kernel is not unconstrained and we have the following
\begin{proposition}
  If $G=N\rtimes K$ is a semi-direct product group satisfying \eqref{eq:41}, the map $\kappa\cdot f_{1}$ defined in \eqref{GC_transform} subject to the constraint \eqref{GC_constraint} can be realized as the convolution
\begin{align}
  [\kappa\star f_{1}](y)=\int_{N}\rho_{2}(n)\kappa(\sigma_{2}^{-1}(n)y)\rho_{1}^{-1}(n)f_{1}(\sigma_{1}(n)x_{0})\dd n\,,\label{eq:22}
\end{align}
where the kernel satisfies the constraint
\begin{align}
  \kappa(\sigma_{2}(h)y)=\rho_{2}(h)\kappa(y)\rho_{1}^{-1}(h)\,.\label{eq:23}
\end{align}
\end{proposition}

In practice, the constraint \eqref{eq:23} restricts the expressivity of the network, as mentioned in Remark~\ref{rem:stabilizer}. To construct a network layer using \eqref{eq:22} and \eqref{eq:23}, one identifies a basis of the space of solutions of \eqref{eq:23}, expands $\kappa$ in this basis in \eqref{eq:22} and trains only the coefficients of the expansion. A basis of solutions of \eqref{eq:23} for compact groups $K$ in terms of representation-theoretic quantities was given in \cite{lang2020}.

\begin{example}[$\SE(n)$ equivariant CNNs]
  In the literature, the special case $G=\SE(n)=\RR^{n}\rtimes\SO(n)$ and $\mathcal{X}=\mathcal{Y}=\RR^{n}$ has received considerable attention due to its relevance to applications. Our treatment follows~\cite{weiler2018}, for a brief overview of the relevant literature, see Section~\ref{sec:related-literature}.

  In this case, $\RR^{n}$ acts by vector addition on itself, for $t\in\RR^{n}$, $\sigma_{1}(t)x=\sigma_{2}(t)x=x+t$ and $\SO(n)$ acts by matrix multiplication, for $R\in\SO(n)$, $\sigma_{1}(R)x=\sigma_{2}(R)x=Rx$. Moreover, $\RR^{n}$ acts trivially on $V_{1}$ and $V_{2}$, i.e.\ $\rho_{1}(t)=\rho_{2}(t)=\id$. The base point $x_{0}$ is in this case the origin of $\RR^{n}$, which is left invariant by rotations. With this setup, \eqref{eq:22} simplifies to
  \begin{align}
    [\kappa\star f_{1}](y)=\int_{\RR^{n}}\kappa(y-t)f_{1}(t)\dd t\,.
  \end{align}
  The kernel constraint \eqref{eq:23} becomes
  \begin{align}
    \kappa(Ry)=\rho_{2}(R)\kappa(y)\rho_{1}^{-1}(R)\quad\forall R \in \SO(n)\,.\label{eq:10}
  \end{align}
\end{example}

\subsection{Non-transitive group actions}
If the group action of $G$ on $\mathcal{X}$ is not transitive, we cannot simply replace the integral over $\mathcal{X}$ by an integral over $G$ as in \eqref{GC_Gconv}. However, we can split the space $\mathcal{X}$ into equivalence classes under the group action by defining
\begin{align}
  \mathcal{X}/G = \{[x_{0}]:x_{0}\in\mathcal{X}\}\quad\text{where}\quad x_{0}\in\mathcal{X}\sim\tilde{x}_{0}\in\mathcal{X}\ \Leftrightarrow\ \exists\, g\in G\ \text{s.t.}\ \sigma_{1}(g)x_{0}=\tilde{x}_{0}\,.
\end{align}
Within each equivalence class $[x_{0}]$, $G$ acts transitively by definition. For each  class we select an arbitrary representative as base point and define a one-argument kernel by
\begin{align}
  \kappa_{x_{0}}(y)=\kappa(x_{0},y)\,.\label{eq:24}
\end{align}
Using this kernel, we can write the integral \eqref{GC_transform} as
\begin{align}
  [\kappa \star f_{1}](y) = \int_{\mathcal{X}/G}\int_{G} \rho_{2}(g) \kappa_{x_{0}}(\sigma_{2}^{-1}(g) y) \rho_{1}^{-1}(g)f_{1}(\sigma_{1}(g)x_{0}) \ \dd g\dd x_{0}\,.
\end{align}
\begin{example}[$\SO(3)$ acting on $\RR^{3}$]
  Consider $\mathcal{X}=\RR^{3}$ and $G=\SO(3)$. In this case, $G$ does not act transitively on $\mathcal{X}$, since $\SO(3)$ conserves the norm on $\RR^{3}$ and integrating over $G$ is not enough to cover all of $\mathcal{X}$. Hence, $\mathcal{X}/G$ can be identified with the space $\RR^{+}$ of norms in $\RR^{3}$. The split of $\mathcal{X}$ into $\mathcal{X}/G$ and $G$ therefore corresponds to the usual split in spherical coordinates, in which the integral over the radius is separate from the integral over the solid angle. In \cite{fox2021}, a similar split is used, where the integral over $G$ is realized as a graph convolution with isotropic kernels.
\end{example}

\section{Equivariant deep network architectures for machine learning}
\label{sec:Eqdeeparch}

\noindent After having defined various general convolution operators in the previous section, we now want to illustrate how to assemble these into an equivariant neural network architecture using discretized versions of the integral operators that appeared above. To this end, we will first discuss the crucial equivariant nonlinearities that enable the network to learn non-linear functions. Then, we will use two important tasks from computer vision, namely semantic segmentation on $S^2$ and object detection on $\ZZ^2$, to show in detail what the entire equivariant network architecture looks like.

\subsection{Nonlinearities and equivariance}
\label{sec:nonlinearities}
So far, we have only discussed the linear transformation in the equivariant network layers. However, in order to approximate nonlinear functions, it is crucial to include nonlinearities into the network architecture. Of course, for the entire network to be equivariant with respect to group transformations, the nonlinearities must also be equivariant. Various equivariant nonlinearities have been discussed in the literature which we will review here briefly. An overview and experimental comparison of different nonlinearities for $\E(2)$ equivariant networks was given in~\cite{weiler2019}.

A nonlinear activation function of a neural network is a nonlinear function $\eta$ which maps feature maps to feature maps. An equivariant nonlinearity additionally satisfies the constraint
\begin{align}
  \eta[\pi_{1}(g)f]=\pi_{2}(g)\eta(f)\,,\label{eq:39}
\end{align}
where we used the notation introduced in \eqref{eq:2} and \eqref{eq:3}.
\begin{remark}[Biases]
  Often, the feature map to which the nonlinearity is applied is not directly the output of a convolutional layer, but instead a learnable bias is added first, i.e.\ we compute $\eta(f+b)$, where $b:\mathcal{X}\rightarrow V_{1}$ is constant.
\end{remark}

The most important class of nonlinearities used in group equivariant networks act point-wise on the input, i.e.
\begin{align}
  (\eta(f))(x)=\bar{\eta}(f(x))\,,\label{eq:4}
\end{align}
for some nonlinear function $\bar{\eta}:V_{1}\rightarrow V_{2}$. If $f$ transforms as a scalar (i.e.\ $\rho_{1,2}=\id$ in \eqref{eq:2} and \eqref{eq:3}), then any function $\bar{\eta}$ can be used to construct an equivariant nonlinearity according to \eqref{eq:4}. In this case, a rectified linear unit is the most popular choice for $\bar{\eta}$ and was used e.g.\ in \cite{worrall2017, weiler2018, cohen2018b, esteves2020b}, but other popular activation functions appear as well~\cite{thomas2018}.

If however the input- and output feature maps transform in non-trivial representations of $G$, $\bar{\eta}$ needs to satisfy
\begin{align}
  \bar{\eta}\circ\rho_{1}(g)=\rho_{2}(g)\circ\bar{\eta}\,,\qquad\forall g\in G\,.\label{eq:42}
\end{align}
\begin{remark}\label{rem:bandlimits_vs_nonlinearities}
  If the convolution is computed in Fourier space with limited bandwidth, point-wise nonlinearities in position space as \eqref{eq:4} break strict equivariance since they violate the bandlimit. The resulting networks are then only approximately group equivariant.
\end{remark}

In the special case that the domain of $f$ is a finite group $G$ and $\rho_{1,2}$ are trivial, the point-wise nonlinearity \eqref{eq:4} is called \emph{regular}~\cite{Cohen2016}. Regular nonlinearities haven been used widely in the literature, e.g.~\cite{CohenGauge2019, cohen2018b, weiler2018a, dieleman2016, hoogeboom2018hexaconv, bekkers2018}. If the domain of the feature map is a quotient space $G/K$, \cite{weiler2019} calls \eqref{eq:4} a \emph{quotient nonlinearity}. Similarly, given a non-linear function which satisfies \eqref{eq:42} for representations $\rho_{1,2}$ of the subgroup $K$, a nonlinearity which is equivariant with respect to the induced representation $\mathrm{Ind}_{K}^{G}$ can be constructed by point-wise action on $G/K$~\cite{weiler2019}. If $f$ is defined on a semi-direct product group $N\rtimes G$, all these constructions can be extended by acting point-wise on $N$.

A nonlinearity which is equivariant with respect to a subgroup $K$ of a semi-direct product group $G=K\ltimes N$ is given by the \emph{vector field nonlinearity} defined in \cite{marcos2017}. In the reference, it is constructed for the cyclic rotation group $K=C_{N}<\SO(2)$ and the group of translations $N=\RR^{2}$ in two dimensions, but we generalize it here to arbitrary semi-direct product groups. The vector field nonlinearity maps a function on $G$ to a function on $N$ and is equivariant with respect to  representations $\pi_{\mathrm{reg}}$ and $\pi_{2}$ defined by
\begin{align}
    [\pi_{\mathrm{reg}}(\tilde{k})f_{1}](kn) &= f_{1}(\tilde{k}^{-1}kn)\,,\\
    [\pi_{2}(\tilde{k}) f_2](n) &= \rho_{2}(\tilde{k})f_2(n)\,,
\end{align}
for some representation $\rho_{2}$ of $K$. It reduces the domain of the feature map by taking the maximum over orbits of $K$, akin to a subgroup maxpooling operation. However, in order to retain some of the information contained in the $K$-dependence of the feature map, it multiplies the maximum with an argmax over orbits of $K$ which is used to construct a feature vector transforming in the representation $\rho_{2}$ of $K$. Its equivariance is shown in the proof of the following
\begin{proposition}[Vector field nonlinearity]\label{prop:vfieldnonlinearity}
  The nonlinearity $\eta:L^2(K\ltimes N,V_1) \rightarrow L^2(N,V_2)$ defined by
\begin{align}
  [\eta(f)](n)=\max_{k\in K}(f(kn))\rho_{2}(\argmax_{k\in K}(f(kn)))v_{0}\,,
\end{align}
where $v_{0}\in V_2$ is some reference point, satisfies the equivariance property
\begin{align}
    \eta(\pi_{\mathrm{reg}}(k)f)=\pi_2(k)[\eta(f)]\,,\qquad\forall k\in K\,.
\end{align}
\end{proposition}
\begin{proof}
To verify the equivariance, we act with the $\pi_{\mathrm{reg}}$ on $f$, yielding
\begin{align}
  [\eta(\pi_{\mathrm{reg}}(\tilde k)f)](n)&=\max_{k\in K}(f(\tilde{k}^{-1}kn))\rho_{2}(\argmax_{k\in K}(f(\tilde{k}^{-1}kn)))v_{0}\nonumber\\
  &=\max_{k\in K}(f(kn))\rho_{2}(\tilde{k}\argmax_{k\in K}(f(kn)))v_{0}\\
  &=\rho_{2}(\tilde{k})[\eta(f)](n)\,,\nonumber
\end{align}
where we used that the maximum is invariant and the argmax equivariant with respect to shifts.
\end{proof}
\begin{example}[Two-dimensional roto-translations]
    Consider the special case $G=C_N\ltimes \RR^{2}$, $v_{0}=(1,0)$ and $\rho_{2}$ the fundamental representation of $C_{N}$ in $\RR^{2}$ which was discussed in \cite{marcos2017}, where the vector field nonlinearity was first developed. In this case, we will denote a feature map on $C_{N}\ltimes\RR^{2}$ by a function $f_{\theta}(x)$ where $\theta\in C_N$ and $x\in\RR^2$. Then, the vector field nonlinearity is given by
    \begin{align}
        [\eta(f)](x)=\max_{\theta\in C_N}(f_{\theta}(x))\begin{pmatrix}\cos(\argmax_{\theta\in C_{N}}f_{\theta}(x))\\\sin(\argmax_{\theta\in C_{N}}f_{\theta}(x))\end{pmatrix}\,,
    \end{align}
    illustrating the origin of its name.
\end{example}

If the input- and output features transform in the same unitary transformation, i.e.\ $\rho_{1}=\rho_{2}=\rho$ and $\rho(g)\rho^{\dagger}(g)=\id$ for all $g\in G$, then \emph{norm nonlinearities} are a widely used special case of \eqref{eq:4}. These satisfy \eqref{eq:42} and are defined as
\begin{align}
  \bar{\eta}(f(x))=\alpha(||f(x)||)\,f(x)\,, \label{eq:43}
\end{align}
for any nonlinear function $\alpha:\RR\rightarrow\RR$. Examples for $\alpha$ used in the literature include sigmoid~\cite{weiler2018}, relu~\cite{worrall2017, esteves2020b}, shifted soft plus~\cite{thomas2018} and swish~\cite{muller2021}.  In~\cite{Favoni:2020reg}, norm nonlinearities are used for matrix valued feature maps with $||\cdot||=\Re(\tr(\cdot))$ and $\alpha=\mathrm{relu}$. A further variation are \emph{gated nonlinearities}~\cite{weiler2018,finzi2021} which are of the form $\bar{\eta}(f(x))=\sigma(s(x))\,f(x)\,$ with $\sigma$ the sigmoid function and $s(x)$ an additional scalar feature from the previous layer.

Instead of point-wise nonlinearities in position space, nonlinearities in Fourier space have also been used. These circumvent the problem mentioned in Remark~\ref{rem:bandlimits_vs_nonlinearities} and the resulting networks are therefore equivariant within numerical precision. E.g.\ \cite{thomas2018} uses norm nonlinearities of the form \eqref{eq:43} in Fourier space. However, the most important nonlinearity in this class are \emph{tensor product nonlinearities}~\cite{kondor2018b, kondor2018a} which compute the tensor product of two Fourier components of the input feature map. They yield a feature map transforming in the tensor product representation which is then decomposed into irreducible representations. To eschew the large tensors in this process, \cite{cobb2020, mcewen2021} introduce various refinements of this basic idea.

\begin{remark}[Universality]
  The Fourier-space analogue of a point-wise nonlinearity in position space is a nonlinearity which does not mix the different Fourier components, i.e.\ which is of the form
  \begin{align}
      [\eta(f)]^{\ell} = \bar{\eta}(f^{\ell})\,.
  \end{align}
  This is the way that norm- and gated nonlinearities have been used in \cite{thomas2018,weiler2018,finzi2021}. However, as pointed out in \cite{finzi2021}, these nonlinearities can dramatically reduce the approximation capacity of the resulting equivariant networks. This problem does not exist for tensor product nonlinearities.
\end{remark}

\emph{Subgroup pooling}~\cite{Cohen2016, weiler2018a, bekkers2018, worrall2018, winkels2018} with respect to a subgroup $K$ of $G$ can be seen as a nonlinearity which does not act point-wise on $f$, but on orbits of $K$ in the domain of $f$,
\begin{align}
  \eta(f)(gK)=\bar{\eta}(f(\{gk|k\in K\}))\,.
\end{align}
This breaks the symmetry group of the network from $G$ to $G/K$ and yields a feature map defined on $G/K$.
The function $\bar{\eta}$ is typically an average\footnote{Of course, the average is a linear operation.} or maximum of the arguments.

\subsection{Semantic segmentation on $S^2$}
\label{sec:sem_seg_s2}
After having reviewed equivariant nonlinearities, we can now proceed to discuss concrete equivariant network architectures. For computer vision tasks such as semantic segmentation and object detection, the standard flat image space of $\mathbb{Z}^{2}$ is a homogeneous space under translation and hence falls in the class of convolutions discussed in Section~\ref{sec:equiv-conv-layers}. The standard convolution \eqref{eq:2dCNN} is, by \eqref{thm_main}, the natural $\mathbb{Z}^{2}$ equivariant layer in this context. Let's now take the first concrete steps into a more complex example of how the equivariant structure on homogeneous spaces can be applied in a more interesting setup.

Moving to the sphere $S^2$ provides a simple example of a non-trivial homogeneous space as the input manifold. As detailed in \eqref{eq:35}, a semantic segmentation model can now be viewed as a map
\begin{equation}
\mathcal{N}: L^{2}(S^{2}, \mathbb{R}^3) \rightarrow L^{2}(S^{2}, P(\Omega))\,,
\end{equation}
where $P(\Omega)$ is the space of probability distributions over the $\Nc$ classes $\Omega$. The output features transform as scalars under the group action on the input manifold. In the notation of \eqref{eq:2} and \eqref{eq:3}, this means that $\mathcal{X}=\mathcal{Y}=S^{2}$, $V_{1}=\RR^{3}$ and $V_{2}=\RR^{\Nc}$. The symmetry group should act on the output space in the same way as on the input space, i.e.\ $\sigma_{2}=\sigma_{1}$. Since the output class labels do not carry any directional information, we have $\rho_{2}=\id_{\Nc}$.

Viewed as a quotient space $S^2 = \SO(3) / \SO(2)$, the sphere is a homogeneous space where each point can be associated with a particular rotation. One possible parametrization would be in terms of latitude ($\theta$) and longitude ($\phi$), i.e. the latitude specifies the angular distance from the north pole whereas the longitude specifies the angular distance from a meridian.

The corresponding convolution following from \eqref{eq:convlayer} can be formulated as an integral over $S^2$ using this parametrization. In practice there are much more efficient formulations using the spectral structure on the sphere, see Section \ref{sec:spherical} for a treatment of how convolutions can be performed using a Fourier decomposition.

\subsection{Object detection on $\ZZ^2$}
\label{sec:obj_detection}
\begin{figure}
\begin{tikzcd}
  f_{\mathrm{in}}:\arrow{d}{\Phi_1 \;\textbf{first convolution}}& \mathbb{Z}^{2} \arrow{r} & \mathbb{R}^{3} & (f_{\mathrm{in}})_{ij}^{c} \\
  f_{2}:\arrow{d}{\textrm{relu} \;\textbf{activation}}& \mathbb{Z}^{2} \arrow{r} & \mathbb{R}^{\Nc[2]} & (f_{2})^{c}_{ij} \\
  f_{2}':\arrow{d}{\Phi_2 \;\textbf{second convolution}}& \mathbb{Z}^{2} \arrow{r} & \mathbb{R}^{\Nc[2]} & (f_{2}')^{c}_{ij} \\
  f_\mathrm{out}: & \mathbb{Z}^{2} \arrow{r} & \mathbb{R}^2\oplus \mathbb{R}^{3}  & (f_{\mathrm{out}})^{c}_{i j}
\end{tikzcd}
\caption{Simple object detection model. In the first row the name of the first feature map (i.e. input image) is given by $f_{\mathrm{in}}$, it maps the domain $\mathbb{Z}^{2}$ to RGB values. It can be represented by a rank 3 tensor $(f_{\mathrm{in}})^{c}_{ij}$. The first convolution maps $f_{\mathrm{in}}$ to a new feature map $f_{2}$ with $\Nc[2]$ filters giving rise to another rank 3 tensor. The output is a feature map $\mathbb{Z}^{2} \rightarrow \mathbb{R}^2\oplus \mathbb{R}^{3}$, where an element in the co-domain takes the form $(x, y, w, h, c)$. The first part in the co-domain $\mathbb{R}^2$, represents anchor point coordinates and transforms under translations. The second part $\mathbb{R}^3$ represents the dimensions of the bounding box together with a confidence score, both invariant under translations.}
\label{fig:object_detection}
\end{figure}
\begin{figure}
    \centering
    \includegraphics[width=0.8\textwidth]{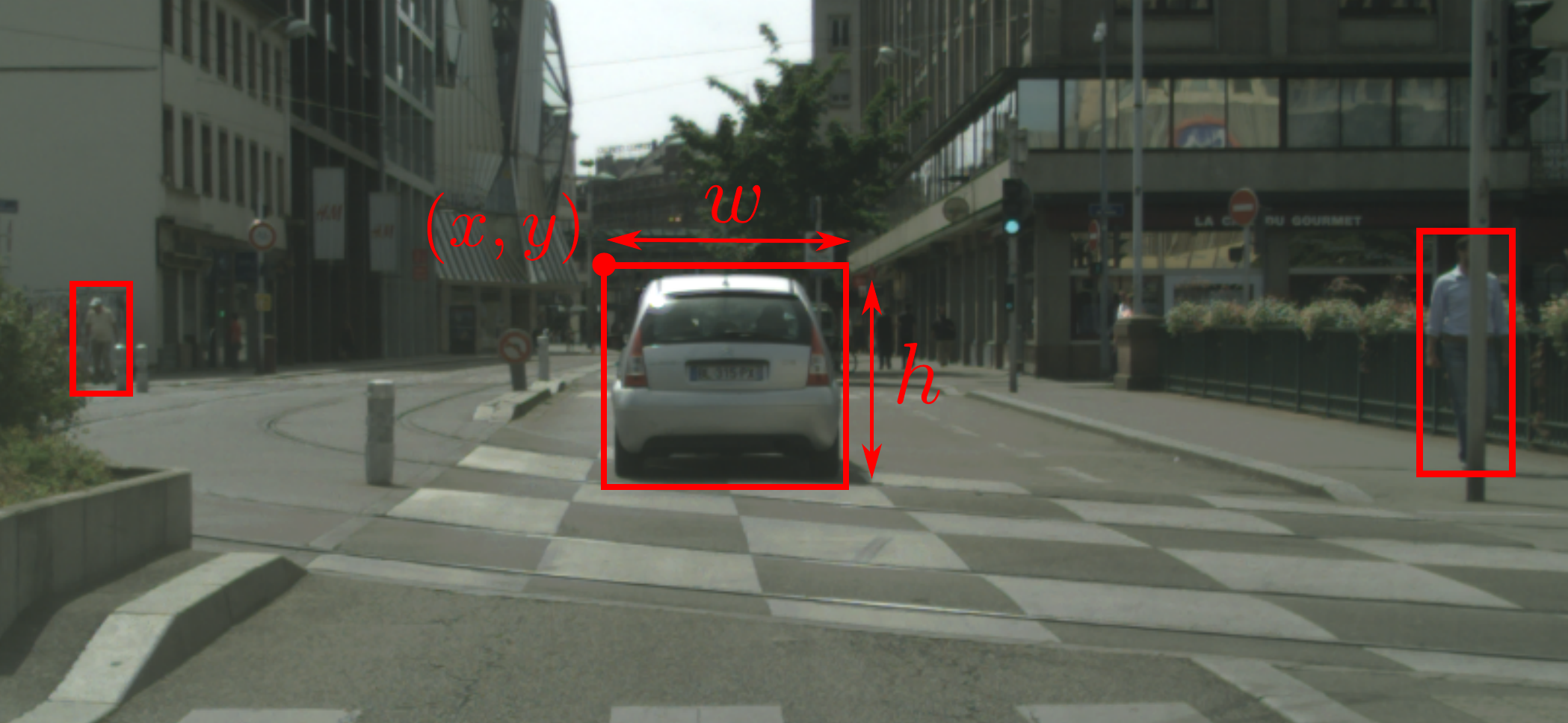}
    \caption{Objects detected by bounding boxes. A bounding box is given by an anchor point $(x,y)$ together with dimensions $(w, h)$. Image from \cite{Cordts2016Cityscapes}.}
    \label{fig:bounding_boxes}
\end{figure}

If we instead stay on $\ZZ^2$ as the input space but let the model output object detections we have a non-trivial example of where the output transforms under the group action and where equivariance of the full model becomes relevant.

Let us consider a
single-stage object detection model that outputs a dense map $\mathbb{Z}^{2} \rightarrow \mathbb{R}^2\oplus\mathbb{R}^{3}$ of candidate detection in the form $\left(x, y, w, h, c\right)$, where $(x, y) \in \mathbb{R}^2$ corresponds to the anchor point of an axis aligned bounding box with dimensions $(w, h)$. The confidence score $c$ (or binary classifier) corresponds to the probability of the bounding box containing an object.\footnote{The dense map of candidate detections is typically filtered based on the confidence score using a method such as non-maximum suppression.} Figure \ref{fig:bounding_boxes} illustrates how the anchor point together with the dimensions form a bounding box used to indicate a detection. The first subspace $\mathbb{R}^2$ in the co-domain, corresponding to anchor point coordinates, is identified as continuous coordinates on the input space $\mathbb{Z}^2$ and transforms under translations.

Starting from the same example architecture as the semantic segmentation model in Figure~\ref{fig:image_classification}, the object detection model in Figure~\ref{fig:object_detection} ends with a feature map transforming in the fundamental representation of the translation group $\mathbb{Z}^2$. Note that since the detection problem is formulated as a regression task, which is usually implemented using models with floating point arithmetic, the output naturally takes values in $\mathbb{R}$ rather than $\mathbb{Z}$.

As introduced in \eqref{eq:2dconvWithTrans}, on $\ZZ^2$, the convolutions are discretized to
\begin{align}
  f_2(x, y) &= \Phi_1(f_\mathrm{in})(x, y) = [\kappa_1 \star f_\mathrm{in}](x, y) =
    \sum_{(x',y')\in \mathbb{Z}^2} L_{(x, y)}\kappa_1 (x', y')  f_\mathrm{in}(x', y')\,,
  \label{eq:conv_semseg}
\end{align}
where we have specified the first layer as an illustration (cf.\ Figure~\ref{fig:object_detection}) and $\kappa_{1}$ is the kernel.
Concretely on $[0, W]\times[0, H]\in\mathbb{Z}^{2}$ with the feature map and kernel represented as rank 3 tensors\footnote{In the sense of multidimensional arrays.} this takes the form
\begin{align}
  (f_2)^{c}_{xy} &= \sum_{(x',y')\in \mathbb{Z}^2}\sum_{c=0}^{2} \kappa^{c}_{(x'-x)(y'-y)} (f_\mathrm{in})^c_{x'y'}\,.
\end{align}

The convolution in \eqref{eq:conv_semseg} is equivariant with respect to translations under $\mathbb{Z}^2$ as shown in \eqref{eq:2dconvequiv}. For the  nonlinearity, we can choose any of the ones discussed in Section~\ref{sec:nonlinearities}, e.g.\ a point-wise relu,
\begin{equation}
  (f_2^\prime)_{i j}^{c} = \textrm{relu}\left((f_2)_{i j}^{c}\right)\,.
\end{equation}

Thus the model is a $G$-equivariant network that respects the $\mathbb{Z}^{2}$ structure of the image plane.
Note that in contrast to the case of semantic segmentation in Section~\ref{sec:sem_seg_s2} the output features here transforms under the group action. If the image is translated the corresponding anchor points for the detections should also be translated. This equivariance is built into the model rather than being learned from data, as would be the case in a standard object detection model.

In the notation of \eqref{eq:2} and \eqref{eq:3} the input and output spaces are $\mathcal{X} = \mathcal{Y} = \mathbb{Z}^2$ with output feature maps taking values in $V_2 = \mathbb{R}^2\oplus\mathbb{R}^3$. The symmetry group $G=\mathbb{Z}^2$ acts by $\sigma_{(x', y')} (x, y) = (x + x', y + y')$ on the input and output space and by $\rho_{(x', y')}(x, y, w, h, c) = (x + x', y + y', w, h, c)$ on $V_2$.

The output of this model in terms of anchor points and corresponding bounding box dimensions is one of many possible ways to formulate the object detection task. For equivariance under translations this representation makes it clear that it is the position of the bounding box that transforms with the translation group.

If we instead are interested in a model that is equivariant with respect to rotations of the image plane it is instructive to regard a model that predicts bounding boxes that are not axis aligned.
Let the output of the new model, as in Section~\ref{sec:intro}, take values in $V_2 = \mathbb{R}^2\oplus\mathbb{R}^2\oplus\mathbb{R}^2$ where an element $(a, v_1, v_2)$ corresponds to a bounding box with one corner at $a$ and spanned by the parallelogram of $v_1$ and $v_2$. All three output vector spaces transform in the fundamental representation of $\SO(2)$ so that $\rho_2 = \mathbf{2}_{\SO(2)}\oplus\mathbf{2}_{\SO(2)}\oplus\mathbf{2}_{\SO(2)}$, cf.~\eqref{eq:so2fundrep}.

\section{Spherical convolutions}
\label{sec:spherical}
\noindent In this section, we will investigate the spherical convolutions introduced in Example~\ref{ex:sphericalCNN} in more detail. Mathematically, this case is particularly interesting, because on the sphere and on $\SO(3)$, we can leverage the power of the rich spectral theory on the sphere in terms of spherical harmonics and Wigner matrices to find explicit and compact expressions for the convolutions. Also for practical considerations, this case is of particular importance since data given on a sphere arises naturally in many applications, e.g.\ for fisheye cameras~\cite{coors2018}, cosmological data~\cite{perraudin2019}, weather data, molecular modeling~\cite{boomsma2017} or diffusion MRI~\cite{elaldi2021}. To faithfully represent this data, equivariant convolutions are essential.

There is a sizable literature on equivariant spherical convolutions. The approach presented here follows \cite{cohen2018b} and extends the results in the reference at some points. An overview of the existing literature in the field can be found in Section~\ref{sec:related-literature}.

\subsection{Preliminaries}
For spherical CNNs, the input data has the form $f:S^{2}\rightarrow \RR^{\Nc}$, i.e. the first layer of the network has $\mathcal{X}=S^{2}$. The networks discussed here are special cases of the GCNNs constructed in Section~\ref{sec:equiv-conv-layers} for which the symmetry group $G$ is the rotation group in three dimensions, $\SO(3)$ and the subgroup $K$ is either $\SO(2)$ or trivial. In this framework, we identify $S^{2}$ with the homogeneous space $G/K=\SO(3)/\SO(2)$. The first layer of the network has trivial $K$ in the output, i.e. \ $\mathcal{Y}=\SO(3)$ and subsequent layers have trivial $K$ also in the input, i.e.\ $\mathcal{X}=\mathcal{Y}=\SO(3)$. The latter case was already discussed in Example~\ref{ex:sphericalCNN}, leading to \eqref{SO3C_conv}.

For the first layer, $\SO(3)$ acts by the matrix-vector product on the input space, $\sigma_{1}(R)x=Rx$ and as usual by group multiplication on the output, $\sigma_{2}(R)Q=RQ$. The construction in \cite{cohen2018b} uses in this case the identity element in $\mathcal{Y}=\SO(3)$ to solve the kernel constraint, leading to
\begin{align}
  (\kappa \star f)(R) = \int_{S^{2}}\rho_{2}(R)\kappa(R^{-1}x)\rho_{1}^{-1}(R)f(x)\,\dd{x}\,,\label{eq:28}
\end{align}
cf.\ \eqref{eq:1}. The integration measure on the sphere is given by
\begin{align}
  \int_{S^{2}}\dd{x(\theta,\varphi)}=\int_{0}^{2\pi}\dd{\varphi}\int_{0}^{\pi}\dd{\theta}\sin\theta\,.
\end{align}
Note that in~\eqref{eq:28} we perform the integral here over the input domain $\mathcal{X}$ and not the symmetry group $G$ since we cannot replace the integral over $\mathcal{X}$ by an integral over $G$ if we use a reference point in $\mathcal{Y}$, as mentioned above.

\begin{remark}[Convolutions with $\mathcal{X}=\mathcal{Y}=S^{2}$]
\label{rk:isotropic}
  At first, it might seem counter-intuitive that the feature maps in a spherical CNN have domain $\SO(3)$ instead of $S^{2}$ after the first layer. Hence, it is instructive to study a convolutional integral with $\mathcal{X}=\mathcal{Y}=S^{2}$ constructed using the techniques of Section~\ref{sec:GC}. The action of $\SO(3)$ on $S^{2}$ is as above by matrix-vector product. We next choose an arbitrary reference point $x_{0}$ on the sphere and denote by $R_{x}$ the rotation matrix which rotates $x_{0}$ into $x$: $R_{x}x_{0}=x$. Following \eqref{GC_conv}, we can then write the convolution as
  \begin{align}
    [\kappa\star f](y)=\int_{\mathcal{X}}\rho_{2}(R_{x})\kappa(R_{x}^{-1}y)\rho_{1}^{-1}(R_{x})f(x)\,\dd{x}\,.\label{eq:7}
  \end{align}
  However, the element $R_{x}$ is not unique: If $Q$ rotates around $x_{0}$, $Qx_{0}=x_{0}$, then $R_{x}Q$ also rotates $x_{0}$ into $x$. In fact, the symmetry group splits into a semi-direct product $\SO(3)=N\rtimes H$ with $H=\SO(2)$ the stabilizer of $x_{0}$ and $N=\SO(3)/H$. This special case was considered in Section~\ref{sec:sem_dir_prod_grps} and we can write \eqref{eq:7} as
  \begin{align}
    [\kappa\star f](y)=\int_{\SO(3)/H}\rho_{2}(R)\kappa(R^{-1}y)\rho_{1}^{-1}(R)f(R x_{0})\,\dd{R}\,.
  \end{align}
  According to \eqref{eq:23}, the kernel $\kappa$ is now not unconstrained anymore but satisfies
  \begin{align}
    \kappa(Qy)=\rho_{2}(Q)\kappa(y)\rho_{1}^{-1}(Q)\,,\label{eq:25}
  \end{align}
  for $Q\in H$. In particular, if the input and output features transform like scalars, $\rho_{1}=\rho_{2}=\id$, \eqref{eq:25} implies that the kernel is invariant under rotations around $x_{0}$, i.e.\ isotropic, as was noticed also in \cite{makadia2007}. In practice, this isotropy decreases the expressibility of the layer considerably.
\end{remark}

\subsection{Spherical convolutions and Fourier transforms}
\label{sec:spherFourier}
The Fourier decomposition on the sphere and on $\SO(3)$ is well studied and can be used to find compact explicit expressions for the integrals defined in the previous section. To simplify some expressions, we will assume in this and the following section that $V_{1,2}$ are vector spaces over $\RR$, so in particular $\rho_{1,2}$ are real representations.

A square-integrable function $f:S^{2}\rightarrow \RR^{c}$ can be decomposed into the spherical harmonics $Y^{\ell}_{m}$ via
\begin{align}
  f(x)=\sum_{\ell=0}^{\infty}\sum_{m=-\ell}^{\ell} \hat{f}_{m}^{\ell}Y^{\ell}_{m}(x)\,,
\end{align}
with Fourier coefficients
\begin{align}
  \hat{f}^{\ell}_{m}=\int_{S^{2}}f(x) \overline{Y^{\ell}_{m}(x)}\ \dd{x}\,,
\end{align}
since the spherical harmonics form a complete orthogonal set,
\begin{align}
  \sum_{\ell=0}^{\infty}\sum_{m=-\ell}^{\ell}\overline{Y^{\ell}_{m}(x)}Y^{\ell}_{m}(y)&=\delta(x-y)\,,\\
  \int_{S^{2}}\overline{Y^{\ell_{1}}_{m_{1}}(x)}Y^{\ell_{2}}_{m_{2}}(x)\dd{x}&=\delta_{\ell_{1}\ell_{2}}\delta_{m_{1}m_{2}}\,.
\end{align}
For later convenience, we also note the following property of spherical harmonics:
\begin{align}
  \overline{Y^{\ell}_{m}(x)}=(-1)^{\ell}Y^{\ell}_{-m}(x)\,.\label{eq:16}
\end{align}
In practice, one truncates the sum over $\ell$ at some finite $L$, the bandwidth, obtaining an approximation of $f$.

Similarly, a square-integrable function $f:\SO(3)\rightarrow \RR^{c}$ can be decomposed into Wigner D-matrices $\mathcal{D}^{\ell}_{mn}(R)$ (for a comprehensive review, see \cite{varshalovich1988}) via
\begin{align}
  f(R)=\sum_{\ell=0}^{\infty}\sum_{m,n=-\ell}^{\ell}\hat{f}^{\ell}_{mn}\mathcal{D}^{\ell}_{mn}(R)\,,\label{eq:26}
\end{align}
with Fourier coefficients
\begin{align}
  \hat{f}^{\ell}_{mn}=\frac{2\ell+1}{8\pi^{2}}\int_{\SO(3)}f(R) \overline{\mathcal{D}^{\ell}_{mn}(R)}\ \dd{R}\,,\label{eq:27}
\end{align}
since the Wigner matrices satisfy the orthogonality and completeness relations
\begin{align}
  \int_{\SO(3)}\overline{\mathcal{D}^{\ell_{1}}_{m_{1}n_{1}}(R)}\mathcal{D}^{\ell_{2}}_{m_{2}n_{2}}(R)\dd{R}&=\frac{8\pi^{2}}{2\ell_{1}+1}\delta_{\ell_{1}\ell_{2}}\delta_{m_{1}m_{2}}\delta_{n_{1}n_{2}}\,\\
  \sum_{\ell=0}^{\infty}\sum_{m,n=-\ell}^{\ell}\overline{\mathcal{D}^{\ell}_{mn}(Q)}\mathcal{D}^{\ell}_{mn}(R)&=\frac{8\pi^{2}}{2\ell+1}\delta(R-Q)\,.
\end{align}
Furthermore, the Wigner D-matrices form a unitary representation of $\SO(3)$ since they satisfy
\begin{align}  
  \mathcal{D}^{\ell}_{mn}(QR) &= \sum_{p=-\ell}^{\ell} \mathcal{D}^{\ell}_{mp}(Q) \mathcal{D}^{\ell}_{pn}(R)\,,\\
  \mathcal{D}^{\ell}_{mn}(R^{-1}) &= (\mathcal{D}^{\ell}_{mn}(R))^{-1}=(\mathcal{D}^{\ell}_{mn}(R))^{\dagger} = \overline{\mathcal{D}^{\ell}_{nm}(R)} \,.
\end{align}
Note furthermore that
\begin{align}
  \overline{\mathcal{D}^{\ell}_{mn}(R)}=(-1)^{n-m}\mathcal{D}^{\ell}_{-m,-n}(R)\,.\label{eq:13}
\end{align}
The regular representation of $\SO(3)$ on spherical harmonics of order $\ell$ is given by the corresponding Wigner matrices:
\begin{align}
  Y^{\ell}_{m}(Rx)=\sum_{n=-\ell}^{\ell}\overline{\mathcal{D}^{\ell}_{mn}(R)}Y^{\ell}_{n}(x)\,.\label{eq:15}
\end{align}
A product of two Wigner matrices is given in terms of the Clebsch--Gordan coefficients $C^{JM}_{\ell_{1}m_{1};\ell_{2}m_{2}}$ by
\begin{align}
 \mathcal{D}^{\ell_{1}}_{m_{1}n_{1}}(R)\mathcal{D}^{\ell_{2}}_{m_{2}n_{2}}(R)=\sum_{J=|\ell_{1}-\ell_{2}|}^{\ell_{1}+\ell_{2}}\sum_{M,N=-J}^{J}C^{JM}_{\ell_{1}m_{1};\ell_{2}m_{2}}C^{JN}_{\ell_{1}n_{1};\ell_{2}n_{2}}\mathcal{D}^{J}_{MN}(R)\,.\label{eq:14}
\end{align}

We will now use these decompositions to write the convolutions \eqref{eq:28} and \eqref{SO3C_conv} in the Fourier domain. To this end, we use Greek letters to index vectors in the spaces $V_{1}$ and $V_{2}$ in which the feature maps take values.
\begin{proposition}
  The Fourier transform of the spherical convolution \eqref{eq:28} with $\mathcal{X}=S^{2}$ and $\mathcal{Y}=\SO(3)$ is given by
\begin{align}
  [\widehat{(\kappa \star f)_{\mu}}]^{\ell}_{mn} &= \sum_{\nu=1}^{\dim V_{2}}\sum_{\sigma,\tau=1}^{\dim V_{1}}\sum_{\substack{\ell_{i}=0\\i=1,2,3}}^{\infty}\sum_{\substack{m_{i},n_{i}=-\ell_{i}\\i=1,2,3}}^{\ell_{i}}\sum_{J=|\ell_{2}-\ell_{1}|}^{\ell_{2}+\ell_{1}}\sum_{M,N=-J}^{J}C^{JM}_{\ell_{2}m_{2};\ell_{1}n_{1}}C^{JN}_{\ell_{2}n_{2};\ell_{1}m_{1}}\nonumber\\
  &\qquad C^{\ell m}_{JM;\ell_{3}m_{3}}C^{\ell n}_{JN;\ell_{3}n_{3}}\widehat{(\rho_{2,\mu\nu})}^{\ell_{2}}_{m_{2}n_{2}}\widehat{(\kappa_{\nu\sigma})}^{\ell_{1}}_{m_{1}}\overline{\widehat{\rho_{1,\sigma\tau}}^{\ell_{3}}_{n_{3}m_{3}}}\overline{\widehat{(f_{\tau})}^{\ell_{1}}_{n_{1}}}\,.\label{eq:9}
\end{align}
\end{proposition}
For $\rho_{1}=\rho_{2}=\mathrm{id}$, \eqref{eq:9} simplifies to
\begin{align}
  [\widehat{(\kappa \star f)}_{\mu}]^{\ell}_{mn} = \sum_{\nu=1}^{\dim V_{1}} (\widehat\kappa_{\mu\nu})^{\ell}_{n}\overline{(\widehat{f_{\nu}})^{\ell}_{m}} \,.\label{eq:12}
\end{align}
A similar calculation can be performed for input features in $\SO(3)$ as detailed by the following proposition.
\begin{proposition}
The Fourier transform of the spherical convolution \eqref{SO3C_conv} with $\mathcal{X}=\mathcal{Y}=\SO(3)$ can be written in the Fourier domain as
\begin{align}
  [\widehat{(\kappa \star f)}]^{\ell}_{mn} &=\frac{8\pi^{2}}{2\ell+1}\sum_{p=-\ell}^{\ell}\sum_{\substack{\ell_{i}=0\\i=1,2,3}}^{\infty}\sum_{\substack{m_{i},n_{i}=-\ell_{i}\\i=1,2,3}}^{\ell_{i}}\sum_{J=|\ell_{1}-\ell_{2}|}^{\ell_{1}+\ell_{2}}\sum_{M,N=-J}^{J}C^{JM}_{\ell_{1}m_{1};\ell_{2}m_{2}}C^{JN}_{\ell_{1}n_{1};\ell_{2}n_{2}}\nonumber\\
  &\qquad C^{\ell m}_{JM;\ell_{3}m_{3}}C^{\ell p}_{JN;\ell_{3}n_{3}}\widehat{\rho_{2}}^{\ell_{1}}_{m_{1}n_{1}}\cdot\widehat{\kappa}^{\ell}_{pn}\cdot\overline{\widehat{\rho_{1}}^{\ell_{2}}_{n_{2}m_{2}}}\cdot\widehat{f}^{\ell_{3}}_{m_{3}n_{3}}\,.\label{eq:6}
\end{align}
Here, the dot $\cdot$ denotes matrix multiplication in $V_{1,2}$, as spelled out in \eqref{eq:9}.
\end{proposition}
For $\rho_{1}=\rho_{2}=\id$ \eqref{eq:6} becomes
\begin{align}
  [\widehat{(\kappa \star f)}]^{\ell}_{mn} &=\frac{8\pi^{2}}{2\ell+1}\sum_{p=-\ell}^{\ell}\widehat{\kappa}^{\ell}_{pn}\cdot\widehat{f}^{\ell}_{mp}\,.
\end{align}
Note that in all these expressions, the Fourier transform is done component wise with respect to the indices in $V_{1,2}$.

\subsection{Decomposition into irreducible representations}
\label{sec:decomp-into-irreps}
An immediate simplification of \eqref{eq:9} and \eqref{eq:6} can be achieved by decomposing $\rho_{1,2}$ into irreps of $\SO(3)$ which are given by the Wigner matrices $\mathcal{D}^{L}$,
\begin{align}
  \rho(R) = \bigoplus_{\lambda}\bigoplus_{\mu}\mathcal{D}^{\lambda}(R)\,,\label{eq:34}
\end{align}
where $\mu$ counts the multiplicity of $\mathcal{D}^{\lambda}$ in $\rho$. Correspondingly, the spaces $V_{1,2}$ are decomposed according to $V_{1,2}=\bigoplus_{\lambda}\bigoplus_{\mu} V_{1,2}^{\lambda\mu}$, where $V_{1,2}^{\lambda\mu}=\RR^{2\lambda+1}$. The feature maps then carry indices $f^{\lambda\mu}_{\nu}$ with $\lambda=0,\dots,\infty$, $\mu=1,\dots,\infty$, $\nu=-\lambda,\dots,\lambda$ with only finitely many non-zero components. Using this decomposition of $V_{1,2}$ the convolution \eqref{eq:28} is given by
\begin{align}
  (\kappa\star f)^{\lambda\mu}_{\nu}(R) = \sum_{\rho=-\lambda}^{\lambda}\sum_{\theta=0}^{\infty}\sum_{\sigma=1}^{\infty}\sum_{\tau,\pi=-\theta}^{\theta}\int_{S^{2}}\mathcal{D}^{\lambda}_{\nu\rho}(R)\kappa^{\lambda\mu;\theta\sigma}_{\rho;\tau}(R^{-1}x)\mathcal{D}^{\theta}_{\tau\pi}(R^{-1})f^{\theta\sigma}_{\pi}(x)\dd{x}\,.
\end{align}
By plugging these expressions for $\rho_{1,2}$, $\kappa$ and $f$ into \eqref{eq:9}, we obtain the following proposition.
\begin{proposition}\label{prop:decomp_s2_conv}
  The decomposition of \eqref{eq:9} into representation spaces of irreducible representations of $\SO(3)$ is given by
\begin{align}
  [\widehat{(\kappa \star f)^{\lambda\mu}_{\nu}}]^{\ell}_{mn} &=\sum_{\rho=-\lambda}^{\lambda}\sum_{\theta=0}^{\infty}\sum_{\sigma=1}^{\infty}\sum_{\tau,\pi=-\theta}^{\theta} \sum_{j=0}^{\infty}\sum_{q,r=-j}^{\ell_{i}}\sum_{J=|\ell_{2}-j|}^{\ell_{2}+j}\sum_{M,N=-J}^{J}C^{JM}_{\lambda\nu;jr}C^{JN}_{\lambda\rho;jq}\nonumber\\
  &\qquad C^{\ell m}_{JM;\theta\tau}C^{\ell n}_{JN;\theta\pi}(\widehat{\kappa^{\lambda\mu;\theta\sigma}_{\rho;\tau}})^{j}_{q} \overline{(\widehat{f^{\theta\sigma}_{\pi}})^{j}_{r}}\,.
\end{align}
Similarly, \eqref{eq:6} decomposes according to
\begin{align}
  (\widehat{(\kappa\star f)^{\lambda\mu}_{\nu}})^{\ell}_{mn}&=\frac{8\pi^{2}}{2\ell+1}\sum_{\rho=-\lambda}^{\lambda}\sum_{\theta=0}^{\infty}\sum_{\sigma=1}^{\infty}\sum_{\tau,\pi=-\theta}^{\theta}\sum_{p=-\ell}^{\ell}\sum_{j=0}^{\infty}\sum_{q,r=-j}^{j}\sum_{J=|\lambda-\sigma|}^{\lambda+\sigma}\sum_{M,N=-J}^{J}C^{JM}_{\lambda\nu;\theta\tau}C^{JN}_{\lambda\rho;\theta\pi}\nonumber\\
  &\qquad C^{\ell m}_{JM;jq}C^{\ell p}_{JN;jr}(\widehat{\kappa^{\lambda\mu;\theta\sigma}_{\rho;\tau}})^{\ell}_{pn}(\widehat{f^{\theta\sigma}_{\pi}})^{j}_{qr}\,.\label{eq:29}
\end{align}
\end{proposition}
In these expressions, the Fourier transform of the convolution is given entirely in terms of the Fourier transforms of the kernel and the input feature map, as well as Clebsch--Gordan coefficients. In particular, Fourier transforms of the representation matrices $\rho_{1,2}$ are trivial in this decomposition of the spaces $V_{1}$ and $V_{2}$.

\subsection{Output features in $\SE(3)$}\label{ex:SE3Outputs}
  As an example of a possible application of the techniques presented in this section, consider the problem of 3D object detection in pictures taken by fisheye cameras. These cameras have a spherical image plane and therefore, the components of the input feature map $f$ transform as scalars under the regular representation (cf.~\eqref{eq:2}):
\begin{align}
  [\pi_{1}(R)f](x)=f(R^{-1}x)\,.
\end{align}
Since this is the transformation property considered in the context of spherical CNNs, the entire network can be built using the layers discussed in this section.
  
  As detailed in Section~\ref{sec:obj_detection}, for object detection, we want to identify the class, physical size, position and orientation for each object in the image. This can be realized by associating to each pixel in the output picture (which is often of lower resolution than the input) a class probability vector $p\in P(\Omega)$, a size vector $s\in\RR^{3}$ containing height, width and depth of the object, a position vector $x\in\RR^{3}$ and an orientation given by a matrix $Q\in\SO(3)$. Here, $P(\Omega)$ is the space of probability distributions in $\Nc$ classes as in \eqref{eq:35}. In the framework outlined above, this means that the output feature map takes values in $P(\Omega)\oplus\RR^{3}\oplus\RR^{3}\oplus\SO(3)$.

  If the fisheye camera rotates by $R\in\SO(3)$, the output features have to transform accordingly. In particular, since the classification and the size of the object do not depend on the rotation, $p$ transforms as $\Nc$ scalars and $s$ transforms as three scalars:
\begin{align}
  \rho_{2}(R) p = p, \qquad \rho_{2}(R) s = s\,,\label{eq:30}
\end{align}
where we used the notation introduced in \eqref{eq:3}. The position vector $x$ on the other hand transforms in the fundamental representation of $\SO(3)$
\begin{align}
  \rho_{2}(R)x = R\cdot x\,.\label{eq:31}
\end{align}
Finally, the rotation matrix $Q$ transforms by a similarity transformation
\begin{align}
  \rho_{2}(R)Q = R\cdot Q \cdot R^{T}\,.\label{eq:32}
\end{align}

As described above for the general case, the transformation property \eqref{eq:30}--\eqref{eq:32} of the output feature map can be decomposed into a direct sum of irreducible representations, labeled by an integer $\ell$. The scalar transformations of \eqref{eq:30} are $\Nc+3$ copies of the $\ell=0$ representation and the fundamental representation in \eqref{eq:31} is the $\ell=1$ representation. To decompose the similarity transformation in \eqref{eq:32}, we use the following
\begin{proposition}
Let $A$, $B$, $C$ be arbitrary matrices and $\vrize(M)$ denote the concatenation of the rows of the matrix $M$. Then,
  \begin{align}
  \vrize(A\cdot B\cdot C)=(A\otimes C^{T})\cdot\vrize(B)\,.\label{eq:33}
\end{align}
\end{proposition}
With this, we obtain for \eqref{eq:32}
\begin{align}
  \rho_{2}(R)\vrize(Q) = (R\otimes R)\cdot \vrize(Q)\,,
\end{align}
i.e.\ $Q$ transforms with the tensor product of two fundamental representations. According to \eqref{eq:14}, this tensor product decomposes into a direct sum of one $\ell=0$, one $\ell=1$ and one $\ell=2$ representation. In total, the final layer of the network will therefore have (in the notation introduced below \eqref{eq:34}) $\lambda=0,1,2$, $\nu=-\lambda,\dots,\lambda$ and $\mu=1,\dots, \Nc+4$ for $\lambda=0$, $\mu=1,2$ for $\lambda=1$ and $\mu=1$ for $\lambda=2$.

Note that the transformation properties \eqref{eq:30}--\eqref{eq:32} of the output feature map are independent of the transformation properties of the input feature map. We have restricted the discussion here to fisheye cameras since, as stated above, these can be realized by the spherical convolutions considered in this section. For pinhole cameras on the other hand, the representation $\pi_{1}$ acting on the input features will be a complicated non-linear transformation arising from the projection of the transformation in 3D space onto the flat image plane. Working out the details of this representation is an interesting direction for further research.

\subsection{$\SE(3)$ equivariant networks}
The same decomposition of $V_{1,2}$ into representation spaces of irreducible representations of $\SO(3)$, which was used in Section~\ref{sec:decomp-into-irreps} can also be used to solve the kernel constraint \eqref{eq:10} for equivariant networks of $\SE(3)$, as discussed in \cite{weiler2018}. In this section, we review this construction.

In the decomposition into irreps, the kernel constraint \eqref{eq:10} reads
\begin{align}
  \kappa^{\lambda\mu;\theta\sigma}_{\rho;\tau}(Ry)=\sum_{\nu=-\lambda}^{\lambda}\sum_{\pi=-\theta}^{\theta}\mathcal{D}^{\lambda}_{\rho\nu}(R)\,\kappa^{\lambda\mu;\theta\sigma}_{\nu;\pi}(y)\,\mathcal{D}^{\theta}_{\pi\tau}(R^{-1})\,.\label{eq:11}
\end{align}
In the following, we will use $\cdot$ to denote matrix multiplication in $\rho,\nu,\pi,\tau$ and drop the multiplicity indices $\mu,\sigma$ since \eqref{eq:11} is a component-wise equation with respect to the multiplicity.

On the right-hand side of \eqref{eq:11}, $\SO(3)$ acts in a tensor product representation on the kernel. To make this explicit, we use the vectorization \eqref{eq:33} and  the unitarity of the Wigner matrices to rewrite \eqref{eq:11} into\footnote{Note that there is a typo in the corresponding equation (13) in \cite{weiler2018}, where the complex conjugation is missing.}
\begin{align}
  \vrize(\kappa^{\lambda\theta}(Ry)) = (\mathcal{D}^{\lambda}(R)\otimes \overline{\mathcal{D}^{\theta}}(R))\cdot\vrize(\kappa^{\lambda\theta}(y))\,.
\end{align}
Using \eqref{eq:13} and \eqref{eq:14}, we can decompose the tensor product of the two Wigner matrices into a direct sum of Wigner matrices $\mathcal{D}^{J}$ with $J=|\lambda-\theta|,\dots,\lambda+\theta$. Performing the corresponding change of basis for $\vrize(\kappa^{\lambda\theta})$ leads to components $\vrize(\kappa^{\lambda\theta;J})$ on which the constraint takes the form
\begin{align}
  \vrize(\kappa^{\lambda\theta;J}(Ry))=\mathcal{D}^{J}(R)\cdot\vrize(\kappa^{\lambda\theta;J}(y))\,.\label{eq:17}
\end{align}
According to \eqref{eq:16} and \eqref{eq:15}, the spherical harmonics $Y^{\ell}_{m}$ solve this constraint and they in fact span the space of solutions with respect to the angular dependence of $\kappa^{\lambda\theta,J}$. Therefore, a general solution of \eqref{eq:17} has the form
\begin{align}
  \vrize(\kappa^{\lambda\theta;J}(y))=\sum_{k}\sum_{m=-J}^{J}w^{\lambda\theta;J}_{k,m}\varphi^{k}(||y||)Y^{J}_{m}(y)\,,
\end{align}
with radial basis functions $\varphi^{k}$ and (trainable) coefficients $w$.

As an example of an application of group equivariant network architectures, we considered spherical CNNs in this section, which are of great practical importance. Spherical convolutions serve as a good example of the Fourier perspective on group equivariant convolutions since the spectral theory on the sphere and rotation group $\SO(3)$ is well-understood. Consequently, in Proposition~\ref{prop:decomp_s2_conv}, we could give explicit, yet completely general expressions for spherical convolutions for feature maps transforming in arbitrary representations of $\SO(3)$, given just in terms of Clebsch--Gordan coefficients.

In principle, such expressions could be derived for a large class of symmetry groups. The foundation of this generalization was laid in~\cite{lang2020}, where it was shown how the kernel constraint for any compact symmetry group can be solved in terms of well-known representation theoretic quantities. In position space, algorithms already exist which can solve the symmetry constraint and generate equivariant architectures automatically~\cite{finzi2021}. 

\section{Conclusions}
\label{sec:conclusions}

\noindent In this paper we have reviewed the recent developments in geometric deep learning and presented a coherent  mathematical framework for equivariant neural networks. In the process we have also developed the theory in various directions, in an attempt to make it more coherent and emphasize the geometric perspective on equivariance.

Throughout the paper we have used the examples of equivariant semantic segmentation and object detection networks to illustrate equivariant CNNs. In particular, in Section~\ref{ex:SE3Outputs} we showed that in rotation-equivariant object detection using fisheye cameras, the input features transform as scalars with respect to the regular representation of $R\in \SO(3)$, while the output features take values in $\SE(3)$. It would be very interesting to generalize this to object detection that is equivariant with respect to $\SE(3)$ instead of $\SO(3)$. In this case, we would add translations of the two-dimensional image plane $\mathbb{R}^2\subset \mathbb{R}^3$ and so the regular representation of $\SE(3)$ needs to be projected onto the image plane. For instance, translations which change the distance of the object to the image plane will be projected to scalings. If this is extended to pinhole cameras, the resulting group action will be highly non-linear. 

Yet another interesting open problem is the development of an unsupervised theory of deep learning on manifolds. This would require to develop a formalism of {\it group equivariant generative networks}. For example, one would like to construct group equivariant versions of variational autoencoders, deep Boltzmann machines and GANs (see e.g. \cite{2020arXiv200501683D}).

As we have emphasized in this work, the feature maps in gauge equivariant CNNs can be viewed as sections of principal (frame) bundles, which are generally called {\it fields} in theoretical physics. The basic building blocks of these theories are special sections corresponding to irreducible representations of the gauge group; these are the {\it elementary particles} of Nature. It is tantalizing to speculate that this notion could also play a key role in deep learning, in the sense that a neural network gradually learns more and more complex feature representations which are built from ``elementary feature types'' arising from irreducible representations of the equivariance group.

The concept of equivariance to symmetries has been a guiding design principle for theoretical physics throughout the past century. The standard model of particle physics and the general theory of relativity provide prime examples of this. In physics, the fundamental role of symmetries is related to the fact that to every (continuous) symmetry is associated through Noether's theorem a conserved physical quantity. For example, equivariance with respect to time translations corresponds to conservation of energy during the evolution of a physical system. It is interesting to speculate that equivariance to global and local symmetries may play similar key roles in building neural network architectures, and that the associated conserved quantities can be used to understand the dynamics of the evolution of the network during training. Steps in this direction have been taken to examine and interpret the symmetries and conserved quantities associated to different gradient methods and data augmentation during the training of neural networks \cite{Gluch2021}. In the application of neural networks to model physical systems, several authors have also constructed equivariant (or invariant) models by incorporating equations of motion - in either the Hamiltonian or Lagrangian formulation of classical mechanics - to accommodate the learning of system dynamics and conservation laws \cite{Greydanus2019,Toth2019,Cranmer2020}. Along these lines, it would be very interesting to look for the general analogue of Noether's theorem in equivariant neural networks, and understand the importance of the corresponding conserved quantities for the dynamics of machine learning.

We hope we have convinced the reader that geometric deep learning is an exciting research field with interesting connections to both mathematics and physics, as well as a host of promising applications in artificial intelligence, ranging from autonomous driving to biomedicine.  Although a huge amount of progress has been made, it is fair to say that the field is still in its infancy. In particular, there is a need for a more foundational understanding of the underlying mathematical structures of neural networks in general, and equivariant neural networks in particular. It is our hope that this paper may serve as a bridge connecting mathematics with deep learning, and will provide seeds for fruitful interactions across the fields of machine learning, mathematics and theoretical physics.

\nontocsec{Acknowledgments}
We are grateful to Oleksandr Balabanov, Robert Berman, Mats Granath, Carl-Joar Karlsson, Klas Modin, David M{\"u}ller and Christopher Zach for helpful discussions. The  work of D.P., J.A. and O.C. is supported by the Wallenberg AI, Autonomous Systems and Software Program (WASP) funded by the Knut and Alice Wallenberg Foundation. D.P. and J.G. are supported by the Swedish Research Council, J.G. is also supported by the Knut and Alice Wallenberg Foundation.

{\small
\providecommand{\href}[2]{#2}\begingroup\raggedright\endgroup

}

\end{document}